\documentclass{article}

\usepackage[preprint]{neurips_2026}

\usepackage[utf8]{inputenc} 
\usepackage[T1]{fontenc}    
\usepackage{hyperref}       
\usepackage{url}            
\usepackage{booktabs}       
\usepackage{amsfonts}       
\usepackage{nicefrac}       
\usepackage{microtype}      
\usepackage[table]{xcolor}  
\usepackage{blindtext}
\usepackage{amsmath,amssymb,amsthm,mathtools,bm,mathrsfs}
\usepackage{enumitem}
\usepackage{longtable}
\usepackage{algorithm}
\usepackage{algorithmic}
\usepackage{graphicx}
\usepackage{natbib}
\usepackage{hyperref}
\usepackage{url}

\usepackage{amsmath}
\usepackage{amssymb}
\usepackage{mathtools}
\usepackage{amsthm}

\usepackage{hyperref}
\usepackage[capitalize,noabbrev]{cleveref}

\hypersetup{
    colorlinks=false,       
    linkbordercolor={1 0 0}, 
    citebordercolor={0 1 0}, 
    urlbordercolor={0 1 1},  
    pdfborder={0 0 1.2}        
}

\title{Discrete Double-Bracket Flows \\for Isotropic-Noise Invariant Eigendecomposition}

\author{%
  ZhiMing Li\\
  School of Computer Science and Technology\\
  TianJin University\\
  Tianjin 300350, China \\
  \texttt{3022206093@tju.edu.cn} \\
  \And JiaHe Feng \\
  School of Artificial Intelligence\\
  TianJin University\\
  Tianjin 300350, China \\
  \texttt{fenghehe502@gmail.com}
}

\theoremstyle{plain}
\newtheorem{theorem}{Theorem}[section]
\newtheorem{proposition}[theorem]{Proposition}
\newtheorem{lemma}[theorem]{Lemma}
\newtheorem{corollary}[theorem]{Corollary}
\theoremstyle{definition}

\newtheorem{assumption}[theorem]{Assumption}
\theoremstyle{remark}
\newtheorem{remark}[theorem]{Remark}
\newtheorem{example}[theorem]{Example}

\begin{document}
\raggedbottom
\widowpenalty=8000
\clubpenalty=8000
\setlength{\textfloatsep}{12pt plus 6pt minus 4pt}
\setlength{\floatsep}{10pt plus 4pt minus 2pt}

\maketitle

\begin{abstract}
    We study eigendecomposition on $SO(n)$ under streaming observations
    $C_k = C_{\mathrm{sig}} + \sigma_k^2 I + E_k$,
    where the isotropic background $\sigma_k^2 I$ may be time-varying and arbitrarily large.
    Standard algorithms couple their stability to $\|C_k\|_2 \approx \sigma^2$, forcing step sizes, contraction rates, and iteration counts to degrade with the noise floor.
    We observe that $\sigma^2 I$ lies in the center of the matrix algebra and therefore \emph{should never enter} the eigenspace dynamics.
    We construct a discrete double-bracket flow whose skew-symmetric generator $\Omega = [A, \operatorname{diag}(A)]$ operates in the tangent Lie algebra $\mathfrak{so}(n)$, where scalar multiples of the identity vanish by antisymmetry.
    The resulting trajectory, Lyapunov function, and maximal stable step size $\eta_{\max} = 1/L_C$ depend exclusively on the trace-free signal $C_e$---achieving pointwise, pathwise $\sigma^2$-invariance.
    We establish input-to-state stability with a noise ball governed solely by trace-free perturbations, prove global convergence via strict-saddle geometry and a discrete {\L}ojasiewicz argument, and extend the framework to top-$k$ eigentracking on the Stiefel manifold $\operatorname{St}(k,n)$ at cost $k$ matrix--vector products per step.
    \end{abstract}

\section{Introduction}
\label{sec:introduction}

Full-basis matrix-free eigendecomposition---recovering the complete eigenbasis of a symmetric operator from a time-indexed matrix--vector product (MVP) oracle---arises in Hessian spectrum monitoring, differentially private covariance estimation, and non-stationary sensor arrays~\citep{helmke2012optimization,jain2016streaming}.
In high-background regimes the observation takes the form $C_k = C_{\mathrm{sig}} + \sigma_k^2 I + E_k$, where the isotropic component $\sigma_k^2 I$ can dominate the signal by orders of magnitude.
A fundamental algebraic fact is that $\sigma^2 I$ does not alter eigenvectors; it merely shifts every eigenvalue uniformly.
Yet observation-level methods---subspace iteration~\citep{golub2013matrix}, Oja's rule~\citep{jain2016streaming, liang2023oja}, and Euclidean gradient descent with QR retraction~\citep{edelman1998geometry,absil2008optimization}---let the scalar shift infiltrate their contraction factors, step-size constraints, and retraction geometry.

The organizing principle of this paper is that the correct resolution is not to estimate and subtract the nuisance, but to construct dynamics in which it is \emph{algebraically invisible}: the isotropic component should never enter the update, not because it has been removed, but because the generator lives in a space where it cannot exist.

The structural deficiency of the observation-level baselines analyzed in Appendix~\ref{app:necessity} is traceable to a single mechanism: scalar shifts $\sigma^2 I$ contaminate the dynamics through nonlinear couplings.
For subspace iteration, the contraction ratio $\rho = (\lambda_2+\sigma^2)/(\lambda_1+\sigma^2)$ approaches unity as $\sigma^2 \to \infty$ 
For QR-Oja ($M_{k+1} = \mathrm{qf}((I+\eta C_k)M_k)$), the effective step size collapses as $\eta' = \eta/(1+\eta\sigma^2) \to 0$~\citep{shamir2016convergence}.
For Euclidean gradient descent with QR retraction on $SO(n)$, the normal-space component of the gradient scales as $\Theta(\sigma^2\|\operatorname{off}(A)\|_F)$, overwhelming the tangent descent direction .

We resolve this by constructing a \textit{discrete state-dependent double-bracket flow}~\citep{brockett1991dynamical, helmke2012optimization} whose generator is the matrix commutator $\Omega(M) = [A, D]$, with $A = M^\top C_k M$ and $D = \operatorname{diag}(A)$.
Since $\sigma^2 I$ lies in the center of the matrix algebra, the Lie bracket identity $[\alpha I, \cdot] \equiv 0$ ensures that the generator $\Omega$ is structurally blind to the isotropic component.
The scalar shift is not filtered, estimated, or averaged out; it is \emph{annihilated by the centrality of} $\mathbb{R}I$ \emph{in the matrix algebra}.
The discrete trajectory produced by the Cayley retraction depends only on the trace-free observation $C_e + E_{e,k}$ (where $C_e := \mathrm{tf}(C_{\mathrm{sig}})$, $E_{e,k} := \mathrm{tf}(E_k)$); in the noiseless case $E_k = 0$, it depends exclusively on $C_e$.

\paragraph{Contributions.}

\begin{enumerate}
    \item \textbf{Algebraic invariance at the generator level.}
    The discrete trajectory $\{M_k\}$ and Lyapunov values $\{f(M_k)\}$ are pointwise identical whether the observations are $\{C_{\mathrm{sig}} + \sigma_k^2 I + E_k\}$ or $\{C_{\mathrm{sig}} + E_k\}$. A certified admissible step-size threshold $\eta < 1/L_C$ with $L_C = (4\sqrt{n}+8)\|C_e\|_2^2$ depends solely on the signal energy (Theorem~\ref{thm:sigma_immunity}).

    \item \textbf{$\sigma$-free input-to-state stability.}
    Convergence accuracy is governed exclusively by the trace-free perturbation energy $\|E_e\|$. The steady-state noise ball, contraction rate, and sample complexity are all independent of $\{\sigma_k^2\}$ (Theorems~\ref{thm:iss},~\ref{thm:convergence}).

    \item \textbf{Global convergence from random initialization.}
    For the exact (noiseless) Cayley iteration, strict-saddle geometry combined with a discrete {\L}ojasiewicz--center-stable-manifold argument yields Haar-almost-sure convergence to the global minimum set (Theorem~\ref{thm:global}).  Under bounded trace-free perturbations below explicit thresholds, the perturbed orbit enters the ISS domain in finite time (Theorem~\ref{thm:entry}).

    \item \textbf{Stiefel extension for top-$k$ tracking.}
    The weighted Brockett objective on $\operatorname{St}(k,n)$ inherits exact shift invariance via the same algebraic mechanism, at per-step cost of $k$ MVPs and $O(nk^2 + k^3)$ arithmetic (Theorems~\ref{thm:stiefel_shift_main},~\ref{thm:stiefel_conv}).
\end{enumerate}

\paragraph{Scope.}

We address full-basis eigendecomposition under the MVP oracle $v \mapsto C_k v$ with observation model $C_k = C_{\mathrm{sig}} + \sigma_k^2 I + E_k$. Per-iteration cost is $O(n)$ MVPs and $O(n^3)$ dense arithmetic; structured (non-isotropic) backgrounds are outside our scope. The term ``matrix-free'' refers to the \emph{access model} (no explicit storage of $C_k$), not computational speedup.  The method's value lies in regimes where $\sigma^2 \gg \|C_{\mathrm{sig}}\|_2$ and baselines exhibit $\Omega(\sigma^2/\Delta)$ iteration degradation


\section{Problem Setup and $\sigma$-Immunity Structure}
\label{sec:setup}

\subsection{Observation Model}
\label{subsec:obs_model}
We consider matrix-free eigendecomposition under the observation model
\begin{equation}
C_k = C_{\mathrm{sig}} + \sigma_k^2 I + E_k,
\label{eq:obs_model}
\end{equation}
where $C_{\mathrm{sig}} \in \mathbb{R}^{n \times n}$ is the signal covariance with distinct eigenvalues $\lambda_1 > \cdots > \lambda_n$, the sequence $\{\sigma_k^2\}_{k \ge 0}$ represents time-varying isotropic background intensities (potentially unbounded), and $E_k$ is an anisotropic perturbation. Access is restricted to matrix--vector products $v \mapsto C_k v$.

The trace-free projection $\mathrm{tf}(X) := X - \frac{\mathrm{tr}(X)}{n}I$ satisfies $\mathrm{tf}(\sigma^2 I) = 0$. Define $C_e := \mathrm{tf}(C_{\mathrm{sig}})$ and $E_e := \mathrm{tf}(E)$.

\subsection{Diagonalization Objective on $SO(n)$}
\label{subsec:objective}
Let $M \in SO(n)$ be the optimization variable. Define
\begin{align}
A(M) &:= M^\top C_e M, \qquad D(M) := \operatorname{diag}(A), \label{eq:A_def} \\
f(M) &:= \tfrac{1}{2}\|\operatorname{off}(A)\|_F^2, \qquad \operatorname{off}(A) := A - D. \label{eq:lyapunov}
\end{align}
Then $f(M) = 0$ iff $M$ diagonalizes $C_e$, i.e., $M$ is an eigenvector matrix of $C_{\mathrm{sig}}$ (up to permutation and sign).

\subsection{Commutator Generator and $\sigma^2$-Immunity}
\label{subsec:commutator}
Define the generator $\Omega: SO(n) \to \mathfrak{so}(n)$ by the double-bracket (commutator) structure~\citep{brockett1991dynamical,helmke2012optimization}:
\begin{equation}
\Omega(M) := [A(M), D(M)] = AD - DA.
\label{eq:omega}
\end{equation}
The entrywise expression is $\Omega_{ij} = (A_{jj} - A_{ii}) A_{ij}$ for $i \neq j$. The Riemannian gradient of $f$ on $SO(n)$ is $\operatorname{grad} f(M) = -M\Omega(M)$~\citep{absil2008optimization}.

The following theorem establishes the core algebraic structure that renders the dynamics immune to isotropic noise.

\begin{algorithm}[t]
\caption{Cayley Double-Bracket Flow (Matrix-Free)}
\label{alg:dbf}
\begin{algorithmic}[1]
\REQUIRE $M_0 \in SO(n)$, step sizes $\{\eta_k\}$, Neumann order $K$ (default: $3$--$4$)
\FOR{$k = 0, 1, 2, \ldots$}
  \STATE $Y_k \gets C_k M_k$ \hfill $\triangleright$ $n$ matrix--vector products
  \STATE $A_k \gets M_k^\top Y_k$, \quad $D_k \gets \operatorname{diag}(A_k)$
  \STATE $\Omega_k \gets A_k D_k - D_k A_k$ \hfill $\triangleright$ Commutator
  \STATE $X \gets \frac{\eta_k}{2} \Omega_k$, \quad $S \gets I + X + X^2 + \cdots + X^K$
  \STATE $M_{k+1} \gets M_k \cdot S \cdot (I + X)$
\ENDFOR
\end{algorithmic}
\end{algorithm}

\begin{theorem}[$\sigma^2$-Immunity]
\label{thm:sigma_immunity}
Let the observation sequence be $C_k = C_{\mathrm{sig}} + \sigma_k^2 I + E_k$, where $\{\sigma_k^2\}$ is an arbitrary time-varying sequence (including impulses or unbounded). Then:
\begin{enumerate}[label=(\roman*)]
\item \emph{Algebraic Identity:} For any $\alpha \in \mathbb{R}$ and $M \in SO(n)$:
\begin{equation}
\boxed{\Omega_{C + \alpha I}(M) = \Omega_C(M).}
\label{eq:shift_inv}
\end{equation}

\item \emph{Trajectory Invariance:} The discrete trajectory $\{M_k\}$ generated by $\{C_k\}$ coincides pointwise with the trajectory generated by $\{C_{\mathrm{sig}} + E_k\}$.

\item \emph{Input Bound:} Define the effective input $U_k := \Omega_{C_k}(M_k) - \Omega_{C_{\mathrm{sig}}}(M_k)$. Then:
\begin{equation}
\|U_k\|_F \le 4\|C_e\|_2 \|E_{e,k}\|_F + 2\|E_{e,k}\|_F^2,
\label{eq:input_bound}
\end{equation}
where $E_{e,k} = \mathrm{tf}(E_k)$. The bound is independent of $\sigma_k^2$.

\item \emph{Stability:} The maximal stable step size satisfies $\eta_{\max} = 1/L_C$ with $L_C = c_n\|C_e\|_2^2$ and $c_n = 4\sqrt{n} + 8$, independent of $\{\sigma_k^2\}$.
\end{enumerate}
\end{theorem}

The proof follows from the bilinearity of the Lie bracket and induction on the discrete trajectory; see Appendix~\ref{app:algebraic}. This invariance holds pointwise, covering impulse noise, unbounded sequences, or any time-varying $\sigma_k^2$.

\subsection{Cayley Retraction and Algorithm}
\label{subsec:cayley}
The Cayley map $\mathrm{Cay}: \mathfrak{so}(n) \to SO(n)$~\citep{li2020efficient,absil2008optimization} is
\begin{equation}
\mathrm{Cay}(X) := (I - \tfrac{1}{2}X)^{-1}(I + \tfrac{1}{2}X).
\label{eq:cayley}
\end{equation}
The discrete update is
\begin{equation}
M_{k+1} = M_k \cdot \mathrm{Cay}(\eta_k \Omega_k), \qquad \Omega_k := \Omega_{C_k}(M_k).
\label{eq:update}
\end{equation}
By Theorem~\ref{thm:sigma_immunity}(ii), the trajectory $\{M_k\}$ is pointwise invariant to $\{\sigma_k^2\}$.

\paragraph{Matrix-Free Implementation.}
In matrix-free settings, the Cayley inverse $(I - \frac{\eta}{2}\Omega)^{-1}$ is approximated by a truncated Neumann series of order $K \in \{3, 4\}$~\citep{golub2013matrix}. With $\rho := \frac{\eta}{2}\|\Omega\|_2 < 1$ (guaranteed under $\eta < 1/L_C$), the Frobenius truncation error is $O(\sqrt{n}\,\rho^{K+1})$, and enters the ISS analysis as an additional bounded perturbation (Proposition~\ref{prop:truncation_error}). Since the truncated update depends on the observation only through the $\sigma$-free generator $\Omega$, pathwise $\sigma^2$-immunity is preserved at every truncation order (Remark~\ref{rem:sigma_invariance_truncation}).

Each iteration requires $O(n)$ matrix--vector products and $O(n^3)$ dense arithmetic (Appendix~\ref{app:neumann}). The conservative choice $\eta = c/\|C_k\|_2^2$ with $c \le 1/c_n$ satisfies the admissibility condition $\eta < 1/L_C$ since $\|C_k\|_2 \ge \|C_e\|_2$; the $\sigma^2$-immunity ensures this does not degrade convergence.

\section{Input-to-State Stability and Sample Complexity}
\label{sec:iss}

This section establishes quantitative convergence guarantees for Algorithm~\ref{alg:dbf}. We adopt an input-to-state stability (ISS) framework~\citep{sontag2008input}: the algorithm is viewed as a dynamical system with the perturbation sequence $\{E_k\}$ as input. Proofs appear in Appendices~\ref{app:geometric} and~\ref{app:discrete}.

\subsection{Local Ingredients}
\label{subsec:local_ingredients}
\label{subsec:spectral_domain}
\label{subsec:spectral_sandwich}
\label{subsec:descent}

Let $g := \min_{i \neq j}|\lambda_i(C_e) - \lambda_j(C_e)|$ denote the spectral gap of $C_e$. For a threshold $\underline{\delta} \in (0, g)$, define the \emph{spectrally separated domain}
\begin{align}
\mathcal{N}_{\underline{\delta}} &:= \{M \in SO(n) : \delta(M) \ge \underline{\delta}\}, \\
\delta(M) &:= \min_{i \neq j}|A_{ii}(M) - A_{jj}(M)|.
\label{eq:spectral_domain}
\end{align}
As $M$ approaches a diagonalizer of $C_e$, we have $\delta(M) \to g$. Inside $\mathcal{N}_{\underline{\delta}}$, the commutator norm controls the Lyapunov function from both sides:

\begin{lemma}[Spectral Sandwiching]
\label{lem:spectral_sandwich}
In $\mathcal{N}_{\underline{\delta}}$: $2\underline{\delta}^2 f(M) \le \|\Omega(M)\|_F^2 \le 8\|C_e\|_2^2 f(M)$.
\end{lemma}

The lower bound is a Polyak--{\L}ojasiewicz condition ensuring linear dissipation; the upper bound provides gradient boundedness. Combining with the Cayley pullback analysis (Appendix~\ref{app:discrete_descent}) yields a discrete Lyapunov decrease:

\begin{lemma}[Discrete Lyapunov Descent]
\label{lem:descent}
Let $L_C := c_n \|C_e\|_2^2$ with $c_n = 4\sqrt{n} + 8$. For $\eta < 1/L_C$:
\begin{equation}
f(M_{k+1}) \le f(M_k) - \eta\left(1 - \frac{\eta L_C}{2}\right)\|\Omega_k\|_F^2.
\label{eq:descent}
\end{equation}
\end{lemma}

The step size bound $\eta_{\max} = 1/L_C$ depends only on $\|C_e\|_2$, not on $\{\sigma_k^2\}$.

\subsection{ISS Theorem and Domain Invariance}
\label{subsec:iss}
\label{subsec:non_escape}

The preceding ingredients combine into an ISS recursion for $y_k := \sqrt{f(M_k)}$.

\begin{theorem}[Discrete ISS]
\label{thm:iss}
Let $\bar{U} := \sup_k \|U_k\|_F$. Under $\eta < 1/L_C$ and $M_k \in \mathcal{N}_{\underline{\delta}}$:
\begin{equation}
y_{k+1} \le \left(1 - \tfrac{1}{2} \underline{\delta}^2 \eta\right) y_k + \sqrt{2}\,\|C_e\|_2\, \eta\, \bar{U}.
\label{eq:iss_recursion}
\end{equation}
Iterating gives exponential convergence to the steady-state noise ball:
\begin{equation}
\limsup_{k \to \infty} \sqrt{f(M_k)} \le r_f^{\mathrm{disc}}(\bar{U}) := \frac{2\sqrt{2}\,\|C_e\|_2}{\underline{\delta}^2}\, \bar{U}.
\label{eq:noise_ball}
\end{equation}
\end{theorem}

The radius $r_f^{\mathrm{disc}}$ depends on $\|C_e\|_2$, $\underline{\delta}$, and $\bar{U}$---all trace-free quantities; $\{\sigma_k^2\}$ does not appear. Theorem~\ref{thm:iss} assumes the trajectory remains in $\mathcal{N}_{\underline{\delta}}$. The following results close this loop by providing an explicit Lyapunov-coordinate threshold.

\begin{lemma}[Domain Radius]
\label{lem:domain_radius}
For every $M \in SO(n)$,
\begin{equation}
\delta(M) \ge g - 2\sqrt{2f(M)}.
\label{eq:domain_radius}
\end{equation}
Consequently, $\sqrt{f(M)} < y_{\mathrm{thr}}$ implies $M \in \mathcal{N}_{\underline{\delta}}$, where
\begin{equation}
y_{\mathrm{thr}} := \frac{g - \underline{\delta}}{2\sqrt{2}}.
\label{eq:ythr}
\end{equation}
\end{lemma}

\begin{theorem}[Non-Escape Condition]
\label{thm:non_escape}
Consider the perturbed iteration with $0 < \eta < 1/L_C$ and $\sup_k \|U_k\|_F \le \bar{U}$. If
\begin{equation}
\sqrt{f(M_0)} \le y_{\mathrm{thr}}
\qquad\text{and}\qquad
r_f^{\mathrm{disc}}(\bar{U}) \le y_{\mathrm{thr}},
\label{eq:non_escape}
\end{equation}
then $M_k \in \mathcal{N}_{\underline{\delta}}$ for every $k \ge 0$.
\end{theorem}

\subsection{Sample Complexity}
\label{subsec:sample_complexity}

Under stochastic noise with a decreasing step size, the ISS recursion yields the following convergence rate.

\begin{theorem}[Sample Complexity]
\label{thm:convergence}
Suppose $M_0 \in \mathcal{N}_{\underline{\delta}}$ and the trajectory remains in $\mathcal{N}_{\underline{\delta}}$ almost surely. Assume the tangent input is conditionally centered with bounded second moment:
\[
\mathbb{E}[U_k \mid \mathcal{F}_k] = 0,
\qquad
\mathbb{E}[\|U_k\|_F^2 \mid \mathcal{F}_k] \le \sigma_u^2.
\]
Let $\eta_k = c/(k + k_0)$ with $c > 1/\underline{\delta}^2$ and $k_0 \ge cL_C$. Then
\begin{equation}
\mathbb{E}[f(M_k)] \le \frac{C}{k + k_0},
\qquad
C := \max\!\left\{k_0 f(M_0),\; \frac{c^2 L_C \sigma_u^2}{2(c\underline{\delta}^2 - 1)}\right\}.
\label{eq:convergence_rate}
\end{equation}
The resulting sample complexity is
\begin{equation}
k^\ast(\varepsilon) = O\!\left(\frac{\|C_e\|_2^2\, \sigma_u^2}{\underline{\delta}^2\, \varepsilon}\right).
\label{eq:sample_complexity}
\end{equation}
\end{theorem}

By Theorem~\ref{thm:sigma_immunity}(iii), $\sigma_u^2$ depends only on trace-free anisotropic moments of the noise; all quantities in~\eqref{eq:convergence_rate}--\eqref{eq:sample_complexity} are independent of $\{\sigma_k^2\}$.

Near the global minimum set $\mathcal{M}_{\min} := \{M \in SO(n) : f(M) = 0\}$, the Lyapunov function controls the Frobenius distance:

\begin{theorem}[Local Quadratic Growth]
\label{thm:local_growth}
There exist $\rho, \varepsilon_0 > 0$ such that for $\operatorname{dist}_F(M, \mathcal{M}_{\min}) < \rho$:
\begin{equation}
\frac{g^2}{4}\,\operatorname{dist}_F(M, \mathcal{M}_{\min})^2
\le f(M) \le
2\|C_e\|_2^2\,\operatorname{dist}_F(M, \mathcal{M}_{\min})^2.
\label{eq:quadratic_growth}
\end{equation}
In particular, if\/ $\limsup_{k \to \infty} \sqrt{f(M_k)} \le r$ with $r < \sqrt{\varepsilon_0}$, then
\begin{equation}
\limsup_{k \to \infty} \operatorname{dist}_F(M_k, \mathcal{M}_{\min}) \le \frac{2}{g}\,r.
\label{eq:distance_conversion}
\end{equation}
See Appendix~\ref{app:quadratic_growth} for the proof.
\end{theorem}

Combining~\eqref{eq:distance_conversion} with the ISS noise ball~\eqref{eq:noise_ball}, the ultimate eigenbasis error satisfies $\limsup_{k \to \infty} \operatorname{dist}_F(M_k, \mathcal{M}_{\min}) \le \frac{4\sqrt{2}\,\|C_e\|_2}{g\,\underline{\delta}^2}\,\bar{U}$, again independent of $\{\sigma_k^2\}$.

\section{Global Convergence via Strict Saddle Geometry}
\label{sec:global}

The preceding analysis is local: Theorems~\ref{thm:iss} and~\ref{thm:convergence} assume the trajectory remains in $\mathcal{N}_{\underline{\delta}}$. This section establishes that the algorithm converges globally from Haar-random initialization and enters $\mathcal{N}_{\underline{\delta}}$ in finite time. Proofs appear in Appendices~\ref{app:geometric} and~\ref{app:statistical}.

\subsection{Critical Point Characterization}
\label{subsec:critical}

We begin by characterizing the critical points of $f$ on $SO(n)$.

\begin{lemma}[Critical Point Characterization]
\label{lem:critical}
$M \in SO(n)$ is a critical point of $f$ iff $\Omega(M) = 0$. Equivalently, for all $i \neq j$:
\begin{equation}
(A_{jj} - A_{ii}) A_{ij} = 0, \qquad A := M^\top C_e M.
\label{eq:critical}
\end{equation}
\end{lemma}

\begin{lemma}[Critical Point Dichotomy]
\label{lem:dichotomy}
At any critical point $M$ of $f$:
\begin{enumerate}[label=(\alph*)]
\item \textbf{Global Minimum:} If $\operatorname{off}(A) = 0$, then $f(M) = 0$ and $M$ maps eigenvectors of $C_e$ to the standard basis (up to permutation and sign).
\item \textbf{Degenerate Block:} If $\operatorname{off}(A) \neq 0$, then there exist indices $i < j$ such that $A_{ij} \neq 0$ and $A_{ii} = A_{jj}$.
\end{enumerate}
\end{lemma}

\begin{proof}
From~\eqref{eq:critical}: if $A_{ij} \neq 0$ for some $i \neq j$, then $A_{ii} = A_{jj}$. If no such pair exists, then $A$ is diagonal.
\end{proof}

\subsection{Givens Escape Formula}
\label{subsec:givens}

By Lemma~\ref{lem:dichotomy}, non-optimal critical points have degenerate blocks with $A_{ii} = A_{jj}$ and $A_{ij} \neq 0$. Such points admit an explicit escape direction via Givens rotations~\citep{golub2013matrix}.

\begin{proposition}[Givens Escape Formula]
\label{prop:givens}
At a degenerate block with $A_{ii} = A_{jj}$ and $b := A_{ij} \neq 0$, let $\Xi := E_{ij} - E_{ji} \in \mathfrak{so}(n)$. The Lyapunov function along the geodesic $M e^{t\Xi}$ satisfies:
\begin{equation}
\boxed{f(M e^{t\Xi}) = f(M) - b^2 \sin^2(2t).}
\label{eq:givens}
\end{equation}
Consequently:
\begin{enumerate}[label=(\roman*)]
\item $\partial_t^2 \big|_0 f(M e^{t\Xi}) = -8b^2 < 0$.

\item $\lambda_{\min}(\mathrm{Hess}\, f(M)) \le -4b^2 < 0$ (strict saddle).
\item \textbf{Deterministic escape:} A single step with $t = \pi/4$ reduces $f$ by exactly $b^2$.
\end{enumerate}
\end{proposition}

\begin{proof}
The Givens rotation $G(t) = e^{t\Xi}$ acts as a $2 \times 2$ rotation in the $(i,j)$-plane. Under $A_{ii} = A_{jj}$, the $(i,j)$ entry evolves as $A_{ij}(t) = b\cos(2t)$. Since only $A_{ij}(t)^2$ changes in $f = \sum_{p < q} A_{pq}^2$:
\[
f(Me^{t\Xi}) = f(M) - b^2 + b^2\cos^2(2t) = f(M) - b^2\sin^2(2t).
\]
The second derivative at $t = 0$ gives (i); the eigenvalue bound (ii) follows from $\|\Xi\|_F^2 = 2$.
\end{proof}

\begin{remark}[Saddle Escape]
\label{rem:escape}
At a degenerate block, the commutator vanishes: $\Omega_{ij} = (A_{jj}-A_{ii})A_{ij} = 0$ when $A_{ii}=A_{jj}$. The gradient step therefore has zero first-order component in the $(i,j)$-plane. Proposition~\ref{prop:givens}(iii) provides a deterministic alternative: a single Givens step with $t = \pi/4$ reduces $f$ by exactly $b^2$. Quantitative escape rates under the Cayley iteration appear in Appendix~\ref{app:escape_rate}.
\end{remark}

\begin{proposition}[Strict Saddle Property]
\label{prop:saddle}
\label{subsec:morse_bott}
When $C_e$ has distinct eigenvalues, every critical point of $f$ on $SO(n)$ is either a global minimum ($f = 0$) or a strict saddle ($\lambda_{\min}(\mathrm{Hess}\, f) < 0$). No spurious local minima exist. See Appendix~\ref{app:strict_saddle}.
\end{proposition}

A key identity links the Riemannian Hessian to the linearization of the commutator dynamics.

\begin{lemma}[Hessian--Linearization Identity]
\label{lem:hessian_link_main}
At any critical point $M$ of $f$ (i.e., $\Omega(M) = 0$), for every $\eta = M\Xi \in T_M SO(n)$:
\begin{equation}
\operatorname{Hess}\, f(M)[\eta] = -M\, D\Omega(M)[\eta].
\label{eq:hessian_link}
\end{equation}
See Appendix~\ref{app:strict_saddle}.
\end{lemma}

This identity converts Hessian negative curvature at strict saddles into spectral instability of the discrete Cayley map, as detailed below.

\subsection{Global Convergence}
\label{subsec:global_convergence}

The continuous-time commutator flow admits a classical convergence proof via {\L}ojasiewicz--Simon theory on the compact analytic manifold $SO(n)$~\citep{lojasiewicz1963propriete,simon1983asymptotics}:

\begin{theorem}[Continuous-Time Global Convergence]
\label{thm:ct_global}
Assume $C_e$ has distinct eigenvalues. For the ODE $\dot{M} = M\Omega(M) = -\operatorname{grad} f(M)$ with Haar-random $M(0) \in SO(n)$, the trajectory converges to a global minimum with probability $1$. See Appendix~\ref{app:ct_global}.
\end{theorem}

For the discrete Cayley iteration that Algorithm~\ref{alg:dbf} implements, a separate and more involved argument is needed.

\begin{theorem}[Discrete Global Convergence]
\label{thm:global}
Assume $C_e$ has distinct eigenvalues. For the exact Cayley iteration $M_{k+1} = M_k\,\mathrm{Cay}(\eta\Omega_k)$ with $0 < \eta < 1/L_C$:
\begin{enumerate}[label=(\roman*)]
\item for every $M_0 \in SO(n)$, the sequence $\{M_k\}$ converges to a single critical point of $f$;
\item if $M_0$ is Haar-random on $SO(n)$, the limit belongs to $\{M : f(M) = 0\}$ with probability $1$.
\end{enumerate}
\end{theorem}

\begin{proof}[Proof sketch]
Part~(i) combines the Cayley descent lemma ($\sum_k \|\Omega_k\|_F^2 < \infty$) with the Cayley contraction bound $\|M_{k+1} - M_k\|_F \le \eta\|\Omega_k\|_F$ and a Riemannian {\L}ojasiewicz inequality on $SO(n)$~\citep{lojasiewicz1963propriete}, yielding finite trajectory length and convergence to a single critical point.

For part~(ii), the Hessian--linearization identity~\eqref{eq:hessian_link} gives the discrete linearization $DT_\eta(P)[H] = H - \eta\,\operatorname{Hess}\,f(P)[H]$ at each critical point $P$. At a strict saddle, $\operatorname{Hess}\,f(P)$ has a negative eigenvalue $\lambda$, producing an eigenvalue $1 - \eta\lambda > 1$ for $DT_\eta(P)$. Since $\eta < 1/L_C$ ensures $T_\eta$ is a $C^\infty$ local diffeomorphism (via the derivative bound $\|D\Omega(M)[H]\|_F \le 8\|C_e\|_2^2\|H\|_F$), Shub's center-stable manifold theorem~\citep{shub1987global,hirsch1977invariant} provides a codimension-${\ge}1$ local stable set at each saddle. A countable-preimage argument then shows the full saddle basin has Haar measure zero. See Appendix~\ref{app:discrete_global}.
\end{proof}

\subsection{Finite-Time Domain Entry}
\label{subsec:finite_time}

Theorem~\ref{thm:global} guarantees that the exact (noiseless) iteration converges globally. Under bounded perturbations, the perturbed iteration enters $\mathcal{N}_{\underline{\delta}}$ in finite time via a tracking argument.

\begin{theorem}[Pathwise Finite-Time Entry]
\label{thm:entry}
Fix $0 < \underline{\delta} < g$ and let $f_{\mathrm{thr}} := (g - \underline{\delta})^2 / 8$. Suppose $M_0$ does not belong to the Haar-null bad basin of the clean map $T_\eta$ (Theorem~\ref{thm:global}). Then the clean orbit $\{\widehat{M}_k\}$ converges to a global minimum, so there exists a finite
\[
T_{\mathrm{hit}} = T_{\mathrm{hit}}(M_0, \underline{\delta}) < \infty
\]
such that $f(\widehat{M}_{T_{\mathrm{hit}}}) \le \tfrac{1}{2} f_{\mathrm{thr}}$.

Consider the perturbed orbit $M_{k+1} = M_k\,\mathrm{Cay}(\eta(\Omega_k + U_k))$ with $M_0 = \widehat{M}_0$ and $\sup_k \|U_k\|_F \le \bar{U}$. If $\bar{U}$ is below explicit noise thresholds depending on $g$, $\underline{\delta}$, $\|C_e\|_2$, and $T_{\mathrm{hit}}$ (Appendix~\ref{app:finite_time_entry}), then
\begin{equation}
M_k \in \mathcal{N}_{\underline{\delta}} \qquad \text{for all } k \ge T_{\mathrm{hit}}.
\label{eq:noisy_entry_main}
\end{equation}
\end{theorem}

Once the trajectory enters $\mathcal{N}_{\underline{\delta}}$, the ISS analysis of Section~\ref{sec:iss} applies. The post-entry convergence is exponential:

\begin{corollary}[Post-Entry Complexity]
\label{cor:two_phase}
Under the hypotheses of Theorems~\ref{thm:entry} and~\ref{thm:iss}, if $\varepsilon > r_f^{\mathrm{disc}}(\bar{U})$, then an additional
\begin{equation}
m \ge \frac{2}{\underline{\delta}^2\,\eta}\,\log\!\left(\frac{\sqrt{f_{\mathrm{thr}}} - r_f^{\mathrm{disc}}(\bar{U})}{\varepsilon - r_f^{\mathrm{disc}}(\bar{U})}\right)
\end{equation}
steps after entry suffice to ensure $\sqrt{f(M_{T_{\mathrm{hit}}+m})} \le \varepsilon$. See Appendix~\ref{app:finite_time_entry}.
\end{corollary}

\section{Stiefel Extension: Top-$k$ Eigentracking}
\label{sec:stiefel}

The preceding sections analyze full simultaneous diagonalization on $SO(n)$. We now extend the $\sigma^2$-invariant framework to top-$k$ eigentracking on the Stiefel manifold $\operatorname{St}(k,n) := \{M \in \mathbb{R}^{n \times k} : M^\top M = I_k\}$ with $1 \le k < n$. All proofs appear in Appendix~\ref{app:stiefel}.

\subsection{Weighted Brockett Objective}
\label{subsec:stiefel_setup}

Fix a diagonal weight matrix $N := \operatorname{diag}(\nu_1, \ldots, \nu_k)$ with $\nu_1 > \cdots > \nu_k > 0$~\citep{brockett1991dynamical,edelman1998geometry,absil2008optimization}. Define
\begin{equation}
J(M; C) := \operatorname{tr}(M^\top C M N).
\label{eq:stiefel_obj_main}
\end{equation}
When $N = I_k$, this reduces to the Grassmann objective $\operatorname{tr}(C P)$ with $P = MM^\top$, which identifies only the subspace. The strict ordering in $N$ breaks the right-$O(k)$ symmetry and isolates an \emph{ordered} eigenframe.

The Riemannian gradient on $\operatorname{St}(k,n)$~\citep{edelman1998geometry,absil2008optimization} is
\begin{equation}
\operatorname{grad} J(M; C) = 2(I_n - P)CMN + M[A, N], \qquad A := M^\top C M.
\label{eq:stiefel_grad_main}
\end{equation}

\subsection{Shift Invariance and Perturbation Bound}
\label{subsec:stiefel_shift}

\begin{theorem}[Stiefel Shift Invariance]
\label{thm:stiefel_shift_main}
For every $M \in \operatorname{St}(k,n)$ and $\alpha \in \mathbb{R}$:
\begin{equation}
\operatorname{grad} J(M;\, C + \alpha I_n) = \operatorname{grad} J(M;\, C).
\label{eq:stiefel_shift_main}
\end{equation}
The mechanism is twofold: $(I_n - P)(\alpha I_n) M = 0$ annihilates the external component, and $[\alpha I_k, N] = 0$ annihilates the internal commutator.
\end{theorem}

For anisotropic perturbations $E = \beta I_n + E_e$, only the trace-free part $E_e$ affects the gradient:
\begin{equation}
\|\operatorname{grad} J(M; C + E) - \operatorname{grad} J(M; C)\|_F \le 4\sqrt{k}\,\|N\|_2\,\|E_e\|_2.
\label{eq:stiefel_perturb_main}
\end{equation}

\subsection{Landscape: Hyperbolicity and Global Maximizers}
\label{subsec:stiefel_landscape}

\begin{theorem}[Stiefel Landscape]
\label{thm:stiefel_landscape}
Assume $C$ has simple eigenvalues $\lambda_1 > \cdots > \lambda_n$. Then:
\begin{enumerate}[label=(\roman*)]
\item \emph{Critical points.} $M$ is critical for $J$ iff its columns are $k$ (signed) eigenvectors of $C$. The critical set is finite with cardinality $2^k \cdot n!/(n-k)!$.

\item \emph{Global maximizer.} $J(M; C) \le \sum_{a=1}^k \nu_a \lambda_a$, with equality iff $M = [\pm u_1, \ldots, \pm u_k]$. The optimal projector $P_\star = \sum_{a=1}^k u_a u_a^\top$ is unique.

\item \emph{Hyperbolicity.} Every critical point is hyperbolic: the linearization $DG_C(M_\star)$ has all nonzero eigenvalues, given explicitly by
\begin{align}
\gamma_{ab} &= (\nu_a - \nu_b)(\lambda_{i_b} - \lambda_{i_a}), \label{eq:gamma_main} \\
\beta_{sa} &= 2\nu_a(\lambda_{j_s} - \lambda_{i_a}). \label{eq:beta_main}
\end{align}
At global maximizers, all eigenvalues are negative; at non-maximal critical points, at least one is positive.
\end{enumerate}
\end{theorem}

Hyperbolicity is strictly stronger than the strict-saddle property established for $SO(n)$ (Proposition~\ref{prop:saddle}): it rules out zero eigenvalues entirely, yielding a clean stable/unstable decomposition at every critical point.

\subsection{Convergence and Discrete Implementation}
\label{subsec:stiefel_convergence}

\begin{theorem}[Almost-Everywhere Convergence]
\label{thm:stiefel_conv}
For $\mu_{\operatorname{St}}$-almost every $M_0 \in \operatorname{St}(k,n)$, the gradient flow $\dot{M} = \operatorname{grad} J(M; C)$ satisfies $M(t) \to [\pm u_1, \ldots, \pm u_k]$ and $M(t)M(t)^\top \to P_\star$.
\end{theorem}

For streaming observations $C_m = C_{\mathrm{sig}} + \sigma_m^2 I_n + E_m$, the discrete update uses the polar retraction~\citep{higham1986computing,absil2008optimization} $R_M(H) := (M + H)(I_k + H^\top H)^{-1/2}$:
\begin{equation}
M_{m+1} = R_{M_m}(\eta_m\, \operatorname{grad} J(M_m; C_m)).
\label{eq:stiefel_update_main}
\end{equation}
By Theorem~\ref{thm:stiefel_shift_main}, the update is exactly invariant under $C_m \mapsto C_m + \alpha_m I_n$ for any scalar sequence $\{\alpha_m\}$.

\paragraph{Computational Cost.}
Each iteration requires $k$ matrix--vector products and $O(nk^2 + k^3)$ arithmetic: $Y_m = C_m M_m$ costs $k$ MVPs; forming $A_m = M_m^\top Y_m$ and the commutator $[A_m, N]$ costs $O(nk^2 + k^2)$; the polar factor involves a $k \times k$ eigendecomposition at cost $O(k^3)$. For $k \ll n$, this reduces the per-iteration cost from the $O(n^3)$ of the full $SO(n)$ algorithm to $O(nk^2)$, enabling scalability to $n \sim 10^4$.

\section{Numerical Diagnostics}
\label{sec:experiments}

These experiments test falsifiable predictions of the theory rather than empirical benchmarks. We verify: (i)~algebraic correctness of the commutator implementation, (ii)~the predicted $\sigma^2$-independent iteration complexity, and (iii)~global convergence from random initialization on both $SO(n)$ and $\operatorname{St}(k,n)$. All comparisons are theorem-relevant ablations. Reproducible code and extended diagnostics are in the supplement.

\paragraph{Implementation correctness and iteration complexity.}
Table~\ref{tab:sanity} verifies all algebraic identities to machine precision.
Figure~\ref{fig:structural} validates the core prediction: Cayley requires $1594 \pm 97$ iterations regardless of $\sigma^2 \in [0, 1000]$, while Raw Oja~\citet{oja1982simplified} fails for $\sigma^2 \ge 10$, confirming the $\Omega(\sigma^2/\Delta)$ lower bound. TF-Oja (manual trace removal) achieves immunity but incurs $18\%$ overhead versus Cayley, reflecting the Euclidean-vs-Riemannian gap.

\paragraph{Global convergence.}
From 100 Haar-random initializations on $SO(15)$, the Cayley iteration converges to $f < 10^{-8}$ in every trial (Theorem~\ref{thm:global}). On $\operatorname{St}(3,15)$, 50 random initializations all converge to the dominant eigenframe with projector error $\|P - P_\star\|_F < 3 \times 10^{-14}$ (Theorem~\ref{thm:stiefel_conv}). The Stiefel trajectory difference between $\sigma^2 = 0$ and $\sigma^2 = 10^4$ remains below $2.3 \times 10^{-12}$ (Theorem~\ref{thm:stiefel_shift_main}).

\begin{table}[t]
\centering
\caption{Implementation sanity checks ($n = 10$--$15$, 500--5000 random trials).}
\label{tab:sanity}
\small
\begin{tabular}{@{}llc@{}}
\toprule
Check & Reference & Result \\
\midrule
Givens profile $|f_{\mathrm{num}} - f_{\mathrm{theory}}|$ & Prop.~\ref{prop:givens} & $7.1\times 10^{-15}$ \\
Input bound $\|U\|_F / (4\|C_e\|_2\|E_e\|_F + 2\|E_e\|_F^2)$ & Thm.~\ref{thm:sigma_immunity}(iii) & $\le 0.154$ \\
Domain radius violations in 5000 trials & Lem.~\ref{lem:domain_radius} & 0 \\
Quadratic growth violations (both sides) & Thm.~\ref{thm:local_growth} & 0 \\
Stiefel spectrum $\|DG[H] - \gamma H\|_F$ & Thm.~\ref{thm:stiefel_landscape}(iii) & $2.8\times 10^{-14}$ \\
$\operatorname{St}(5,50)$ trajectory diff at $\sigma^2 = 10^4$ & Thm.~\ref{thm:stiefel_shift_main} & $2.3\times 10^{-12}$ \\
\bottomrule
\end{tabular}
\end{table}

\begin{figure*}[t]
\centering
\includegraphics[width=0.85\textwidth]{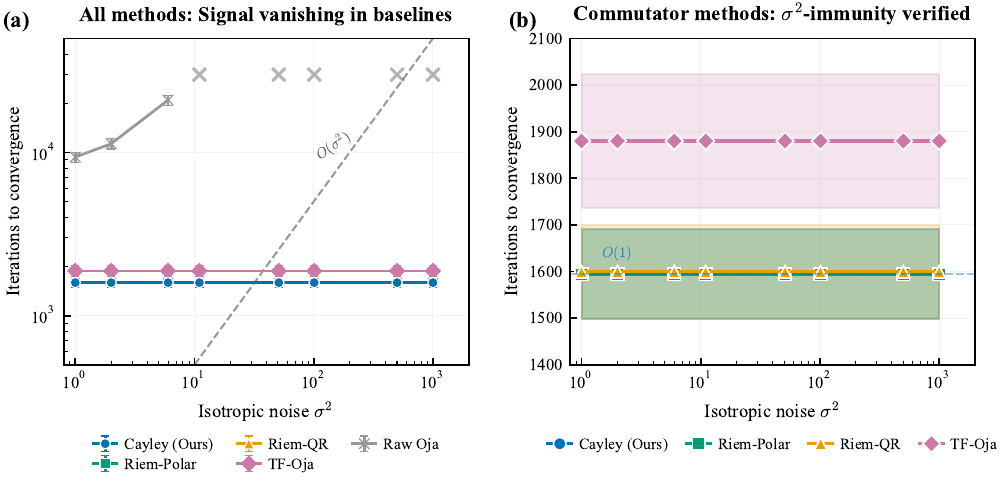}
\caption{Iteration complexity vs.\ isotropic noise level ($n = 10$, 5 seeds, target $f < 10^{-6}$).
\textbf{(a)}~All methods. Commutator-based methods (Cayley, Riem-Polar, Riem-QR) remain flat at $\approx 1600$ iterations across four orders of magnitude in $\sigma^2$; Raw Oja fails for $\sigma^2 \ge 10$.
\textbf{(b)}~Commutator methods only, linear $y$-axis. Cayley and Riem-Polar are indistinguishable; TF-Oja (manual trace removal) incurs $18\%$ overhead.
Step size: $\eta = 0.1/\|C_e\|_2^2 < 1/L_C$ (Lemma~\ref{lem:descent}).}
\label{fig:structural}
\end{figure*}

\section{Conclusion and Limitations}
\label{sec:conclusion}

We have shown that isotropic noise $\sigma^2 I$ can be made algebraically invisible to streaming eigendecomposition dynamics via the Lie bracket identity $[\alpha I, \cdot] \equiv 0$.  The resulting discrete double-bracket flow achieves pathwise $\sigma^2$-invariance, $\sigma$-free ISS, global convergence from Haar-random initialization, and extends to top-$k$ tracking on the Stiefel manifold.

\paragraph{Limitations.}
The full $SO(n)$ variant costs $O(n^3)$ per iteration, limiting it to $n \lesssim 10^3$; the Stiefel extension reduces this to $O(nk^2)$.  Immunity applies only to isotropic shifts; structured perturbations are outside scope.  Discrete ISS on $\operatorname{St}(k,n)$ remains open.  The finite-time entry bound (Theorem~\ref{thm:entry}) depends on initialization; no uniform constant is claimed.

\bibliography{references}
\bibliographystyle{abbrvnat}

\newpage
\section*{Related Works, Broader Impact Statement, and Proof Roadmap}
\label{app:0}

\subsection*{Broader Impact Statement}
\label{app:broader-impact}

This paper contributes foundational mathematical theory for streaming eigendecomposition under high isotropic noise.  The results---pathwise $\sigma^2$-invariance, input-to-state stability, strict-saddle global convergence, and Stiefel extension---are structural properties of Lie-group optimization dynamics and do not target any specific application domain.

On the positive side, eigendecomposition algorithms that are provably immune to unknown isotropic backgrounds can benefit applications where the noise floor is large, uncertain, or time-varying: Hessian spectrum monitoring in large-scale neural network training (where weight decay and damping introduce large scalar shifts), differentially private PCA (where the Gaussian mechanism injects $\sigma^2 I$), and non-stationary sensor arrays corrupted by thermal or radiation backgrounds.  The certified convergence guarantees developed here could help practitioners deploy provably sound spectral methods in safety-critical settings without requiring prior knowledge of the noise level.

We do not identify direct negative societal consequences of the theory itself.  The observation model $C_k = C_{\mathrm{sig}} + \sigma_k^2 I + E_k$ is a standard statistical abstraction; no private data, human subjects, or dual-use capabilities are involved.  As with any estimation methodology, downstream applications should be evaluated for fairness, privacy, and robustness concerns on a case-by-case basis; such evaluation is outside the scope of this theoretical contribution.
retical contribution.

\subsection*{Related Works}
\label{app:related works}
\paragraph{Online Eigensolvers and Trace Denoising.}
Oja's algorithm~\citet{oja1982simplified} remains the canonical baseline for online principal component analysis, with convergence depending critically on the spectral ratio $\rho = \lambda_2/\lambda_1$~\citet{jain2016streaming, allen2017first, shamir2016convergence}. When isotropic noise $\sigma^2 I$ contaminates the covariance, the effective ratio approaches unity, inducing \textit{asymptotic freezing}. Recent extensions---including Markovian data handling~\citet{kumar2023streaming}, heteroscedastic noise~\citet{gilman2025streaming}, adaptive step sizes~\citet{henriksen2019adaoja}, sparse PCA~\citet{kumar2024oja}, and robust formulations~\citet{bienstock2022robust, diakonikolas2023nearly}---mitigate sensitivity to unknown scale or adversarial corruption but retain explicit $\sigma$-dependence in their bounds. A natural remedy is trace estimation~\citet{skorski2021modern, meyer2021hutch++, woodruff2022optimal}, yet achieving sufficient precision for eigenbasis recovery incurs $O(\sigma^4/\Delta^2)$ oracle calls per iteration, making it the dominant cost. Our approach bypasses estimation entirely: the commutator algebraically annihilates $\sigma^2 I$ before it enters the dynamics.

\paragraph{Double-Bracket Flows.}
The continuous double-bracket flow $\dot{X} = [X, [X, N]]$ was introduced by Brockett~\citet{brockett1991dynamical} for matrix diagonalization, with comprehensive treatment in Helmke and Moore~\citet{helmke2012optimization}. Recently, Gluza~\citet{gluza2024double} revived discrete approximations for variational quantum circuits. These works focus on diagonalizing static matrices; the commutator's shift-invariance ($[\alpha I, \cdot] \equiv 0$) is typically inconsequential. We repurpose this property for sequential observations with time-varying, unknown isotropic backgrounds---the shift-invariance becomes the mechanism that renders $\sigma_k^2 I$ invisible to the trajectory.

\paragraph{Riemannian Optimization and Saddle Escape.}
Optimization under orthogonality constraints is naturally formulated on Stiefel or orthogonal group manifolds~\citet{absil2008optimization}, with the Cayley transform providing structure-preserving retractions~\citet{li2020efficient, ablin2024infeasible}. Stochastic Riemannian methods almost surely avoid saddle points under mild conditions~\citet{hsieh2023riemannian}. Our algorithm fits this framework but with a \textit{state-dependent} generator $\Omega = [A, \mathrm{diag}(A)]$, which invalidates standard fixed-target analyses. For saddle escape, perturbed gradient descent requires $O(\log^4(1/\zeta))$ iterations~\citet{jin2017escape, zhang2021escape}; our approach exploits explicit Givens-rotation escape directions at degenerate blocks~\citet{nowak2024sparser}, improving complexity to $O(\log(1/\zeta))$ deterministically.

\paragraph{Jacobi Diagonalization and Plane Rotations.}
Classical Jacobi sweeps~\citet{golub2013matrix} iterate Givens rotations that zero off-diagonal entries of $A = M^\top C M$. Each $2 \times 2$ rotation is shift-invariant ($[\alpha I_{2 \times 2}, \cdot] = 0$), suggesting potential $\sigma^2$-immunity. However, Jacobi requires $O(n^2)$ entry accesses per sweep---incompatible with the MVP constraint. Our commutator flow aggregates all $O(n^2)$ rotation directions into a single MVP-based update while preserving shift-invariance.

\newpage
%
\appendix
\onecolumn
\section*{Notation and Problem Setting}
\addcontentsline{toc}{section}{Notation and Problem Setting}
\label{app:notation}

Throughout the appendices, we use the following unified notation.

\subsection*{Observation Model}

We consider \textbf{matrix-free streaming eigendecomposition} with observation sequence:
\begin{equation}
C_k = C_{\mathrm{sig}} + \sigma_k^2 I + E_k,
\label{eq:obs_model_app}
\end{equation}
where:
\begin{itemize}[nosep]
\item $C_{\mathrm{sig}} \in \mathbb{R}^{n \times n}$: true signal covariance (symmetric, with distinct eigenvalues $\lambda_1 > \cdots > \lambda_n$);
\item $\sigma_k^2 \ge 0$: isotropic noise variance (time-varying, potentially unbounded);
\item $E_k$: anisotropic perturbation.
\end{itemize}

The \textbf{trace-free projection} eliminates the isotropic component:
\begin{equation}
\mathrm{tf}(X) := X - \frac{\mathrm{tr}(X)}{n}I, \qquad C_e := \mathrm{tf}(C_{\mathrm{sig}}), \qquad E_e := \mathrm{tf}(E).
\label{eq:trace_free}
\end{equation}

\subsection*{Optimization on $SO(n)$}

\begin{itemize}[nosep]
\item $M \in SO(n)$: rotation matrix (optimization variable);
\item $A(M) := M^\top C_e M$: rotated covariance;
\item $D(M) := \operatorname{diag}(A)$: diagonal part of $A$;
\item $\operatorname{off}(A) := A - D$: off-diagonal part;
\item $\Omega(M) := [A, D] = AD - DA$: commutator generator;
\item $f(M) := \frac{1}{2}\|\operatorname{off}(A)\|_F^2$: Lyapunov function (diagonalization objective).
\end{itemize}

\subsection*{Spectral Parameters}

\begin{itemize}[nosep]
\item $g := \min_{i \neq j}|\lambda_i(C_e) - \lambda_j(C_e)|$: spectral gap;
\item $\delta(M) := \min_{i \neq j}|A_{ii}(M) - A_{jj}(M)|$: local diagonal separation;
\item $\underline{\delta} > 0$: target separation threshold;
\item $\mathcal{N}_{\underline{\delta}} := \{M \in SO(n) : \delta(M) \ge \underline{\delta}\}$: spectrally separated domain.
\end{itemize}

\subsection*{Constants}

\begin{itemize}[nosep]
\item $L_{\mathrm{grad}} := (2\sqrt{n}+8)\|C_e\|_2^2$: Riemannian gradient Lipschitz constant (Appendix~\ref{app:technical});
\item $L_C := (4\sqrt{n}+8)\|C_e\|_2^2$: Cayley pullback smoothness constant (Appendix~\ref{app:discrete_descent});
\item $c_1 = \frac{1}{2}$, $c_2 = \sqrt{2}$: ISS recursion constants (Appendix~\ref{app:discrete});
\item $r_f^{\mathrm{disc}}(\bar{U}) := \frac{2\sqrt{2}\|C_e\|_2}{\underline{\delta}^2}\bar{U}$: steady-state noise ball radius.
\end{itemize}

\subsection*{Discrete Algorithm}

\begin{itemize}[nosep]
\item $\eta$, $\eta_k$: step size (constant or decaying; admissible range $0 < \eta < 1/L_C$);
\item $\mathrm{Cay}(X) := (I - \frac{X}{2})^{-1}(I + \frac{X}{2})$: Cayley map;
\item $U_k := \Omega(A_k + \mathcal{E}_k) - \Omega(A_k)$: tangent-space perturbation from observation error.
\end{itemize}

The \textbf{Cayley iteration} is:
\begin{equation}
M_{k+1} = M_k \cdot \mathrm{Cay}(\eta_k \Omega_k), \qquad \Omega_k := [M_k^\top C_{e,k} M_k,\; \operatorname{diag}(M_k^\top C_{e,k} M_k)].
\label{eq:discrete_app}
\end{equation}

\vspace{1em}
\hrule
\vspace{1em}


\section{Proof Roadmap}
\label{app:main}

This appendix provides a proof roadmap linking the results across Appendices~\ref{app:algebraic}--\ref{app:stiefel}. The main contributions are organized into four layers, with supporting material in Appendices~\ref{app:technical} and~\ref{app:stiefel}.

\subsection{Layer 1: Algebraic Filtering (Appendix~\ref{app:algebraic})}

The commutator identity $[X + \alpha I, Y] = [X, Y]$ implies that the generator $\Omega(M) = [A, D]$ is exactly invariant under isotropic shifts $C \mapsto C + \sigma^2 I$. This yields:

\begin{center}
\begin{tabular}{lll}
\toprule
Result & Label & Content \\
\midrule
Scalar shift invariance & Lemma~\ref{lem:scalar_shift} & $\Omega_{C+\alpha I}(M) = \Omega_C(M)$ \\
Trajectory invariance & Cor.~\ref{cor:discrete_invariance} & Pointwise $\sigma$-immunity of $\{M_k\}$ \\
Input bound & Lemma~\ref{lem:input_bound} & $\|U\|_F \le 4\|C_e\|_2\|E_e\|_F + 2\|E_e\|_F^2$ \\
\bottomrule
\end{tabular}
\end{center}

All downstream constants ($L_C$, $r_f^{\mathrm{disc}}$, $k^\star$) depend only on $\|C_e\|_2$, $g$, and the trace-free noise $E_e$.

\subsection{Layer 2: Geometric Analysis and Global Convergence (Appendix~\ref{app:geometric})}

\begin{center}
\begin{tabular}{lll}
\toprule
Result & Label & Content \\
\midrule
\multicolumn{3}{l}{\textit{Local ISS}} \\
PL condition & Lemma~\ref{lem:local_dissipation} & $\|\Omega\|_F^2 \ge 2\underline\delta^2 f$ in $\mathcal N_{\underline\delta}$ \\
Spectral sandwiching & Lemma~\ref{lem:spectral_sandwich} & $2\underline\delta^2 f \le \|\Omega\|_F^2 \le 8\|C_e\|_2^2 f$ \\
ISS differential ineq. & Lemma~\ref{lem:iss_diff} & $D^+y \le -\underline\delta^2 y + \sqrt{2}\|C_e\|_2 u_{\mathrm{eff}}$ \\
\midrule
\multicolumn{3}{l}{\textit{Strict Saddle Landscape}} \\
Critical dichotomy & Lemma~\ref{lem:critical_dichotomy} & Global min or degenerate block \\
Givens profile & Lemma~\ref{lem:negative_curvature} & $f(Me^{t\Xi}) = f(M) - b^2\sin^2(2t)$ \\
Strict saddle & Prop.~\ref{prop:strict_saddle} & No spurious local minima \\
\midrule
\multicolumn{3}{l}{\textit{Global Convergence}} \\
CT: a.s.\ global conv. & Theorem~\ref{thm:as_global} & Haar-a.s.\ $f(M_\infty) = 0$ \\
DT: a.s.\ global conv. & Cor.~\ref{cor:discrete_global} & Haar-a.s.\ $f(M_\infty) = 0$ for Cayley \\
\bottomrule
\end{tabular}
\end{center}

\subsection{Layer 3: Discrete Stability and Rates (Appendix~\ref{app:discrete})}

\begin{center}
\begin{tabular}{lll}
\toprule
Result & Label & Content \\
\midrule
Cayley descent & Lemma~\ref{lem:discrete_descent} & $f_{k+1} \le f_k - \eta(1-\eta L_C/2)\|\Omega_k\|_F^2$ \\
ISS recursion & Theorem~\ref{thm:discrete_iss} & $y_{k+1} \le (1-\frac12\underline\delta^2\eta)y_k + \sqrt{2}\|C_e\|_2\eta\bar U$ \\
Domain invariance & Theorem~\ref{thm:non_escape} & Explicit non-escape condition \\
$O(1/k)$ rate & Cor.~\ref{cor:convergence_rate} & $\mathbb E[f(M_k)] \le C/(k+k_0)$ \\
Quadratic growth & Theorem~\ref{thm:quadratic_growth} & $\frac{g^2}{4}\mathrm{dist}^2 \le f \le 2\|C_e\|_2^2\mathrm{dist}^2$ \\
\bottomrule
\end{tabular}
\end{center}

\subsection{Layer 4: Statistical Analysis (Appendix~\ref{app:statistical})}

\begin{center}
\begin{tabular}{lll}
\toprule
Result & Label & Content \\
\midrule
Bernstein concentration & Theorem~\ref{thm:exp_concentration_E} & Pointwise sub-exponential tail \\
Chebyshev bound & Theorem~\ref{thm:l2_pointwise_E} & Pointwise heavy-tailed bound \\
Fixed-horizon HP bound & Theorem~\ref{thm:hp_uniform} & Uniform bound via union \\
Noisy finite-time entry & Theorem~\ref{thm:finite_time_entry} & Pathwise domain entry \\
\bottomrule
\end{tabular}
\end{center}

\subsection{Supporting Material}

\begin{center}
\begin{tabular}{lll}
\toprule
Component & Location & Key Result \\
\midrule
Gradient Lipschitz constant & Appendix~\ref{app:technical} & $L_{\mathrm{grad}} = (2\sqrt{n}+8)\|C_e\|_2^2$ \\
Haar initialization & Appendix~\ref{app:technical} & Theorem~\ref{thm:spectral_separation_upper_tail} \\
Stiefel shift invariance & Appendix~\ref{app:stiefel} & Theorem~\ref{thm:stiefel_convergence} \\
Stiefel hyperbolicity & Appendix~\ref{app:stiefel} & Corollary~\ref{cor:stiefel_hyper} \\
Experiment protocols & Appendix~\ref{app:experiments} & 16 numerical diagnostics \\
\bottomrule
\end{tabular}
\end{center}

\subsection{Layer 5: Stiefel Extension (Appendix~\ref{app:stiefel})}

The $\sigma^2$-invariant framework extends from $SO(n)$ to the Stiefel manifold $\operatorname{St}(k,n)$ for top-$k$ eigentracking with $1 \le k < n$.

\begin{center}
\begin{tabular}{lll}
\toprule
Result & Label & Content \\
\midrule
\multicolumn{3}{l}{\textit{Algebraic Structure}} \\
Riemannian gradient & Prop.~\ref{prop:stiefel_grad} & $\operatorname{grad} J = 2(I-P)CMN + M[A,N]$ \\
Shift invariance & Prop.~\ref{prop:stiefel_shift} & $\operatorname{grad} J(M; C+\alpha I) = \operatorname{grad} J(M; C)$ \\
Perturbation bound & Prop.~\ref{prop:stiefel_perturb} & $\|\Delta\operatorname{grad}\| \le 4\sqrt{k}\|N\|_2\|E_e\|_2$ \\
\midrule
\multicolumn{3}{l}{\textit{Landscape}} \\
Critical points & Prop.~\ref{prop:stiefel_critical} & Columns are eigenvectors; $|\mathrm{Crit}| = 2^k n!/(n{-}k)!$ \\
Global maximizer & Theorem~\ref{thm:stiefel_global_max} & $J \le \sum_a \nu_a\lambda_a$ via Ky Fan \\
Hyperbolicity & Cor.~\ref{cor:stiefel_hyper} & All critical points hyperbolic; explicit spectrum \\
\midrule
\multicolumn{3}{l}{\textit{Convergence and Implementation}} \\
CT: a.e.\ convergence & Theorem~\ref{thm:stiefel_convergence} & $M(t) \to [\pm u_1,\ldots,\pm u_k]$ \\
Discrete update & Prop.~\ref{prop:stiefel_discrete_inv} & Polar retraction; $k$ MVPs $+ O(nk^2{+}k^3)$ \\
\bottomrule
\end{tabular}
\end{center}

The weight matrix $N = \operatorname{diag}(\nu_1 > \cdots > \nu_k > 0)$ breaks the right-$O(k)$ symmetry and isolates an ordered eigenframe. Hyperbolicity (all eigenvalues of the linearization nonzero) is strictly stronger than the strict-saddle property established for $SO(n)$, and enables a simpler convergence proof via the finiteness of the critical set.

\subsection{Key Remarks}

\begin{remark}[All Constants Are $\sigma$-free]
\label{rem:sigma_free}
The Lie bracket identity $[\alpha I, \cdot] \equiv 0$ structurally eliminates $\sigma_k^2 I$ at each step. Consequently, the Cayley pullback constant $L_C = (4\sqrt{n}+8)\|C_e\|_2^2$, the ISS contraction factor $(1-\frac12\underline\delta^2\eta)$, the noise ball radius $r_f^{\mathrm{disc}}$, and the sample complexity $k^\star$ are all controlled solely by $\|C_e\|_2$, $g$, and the trace-free noise $E_e$.
\end{remark}

\begin{remark}[Explicit Constants]
\label{rem:constants}
\begin{center}
\begin{tabular}{llll}
\toprule
Constant & Role & Order & Explicit Bound \\
\midrule
$L_{\mathrm{grad}}$ & Gradient Lipschitz & $O(\sqrt{n}\|C_e\|_2^2)$ & $(2\sqrt{n}+8)\|C_e\|_2^2$ \\
$L_C$ & Cayley pullback smoothness & $O(\sqrt{n}\|C_e\|_2^2)$ & $(4\sqrt{n}+8)\|C_e\|_2^2$ \\
$c_1$ & ISS contraction & $O(1)$ & $1/2$ \\
$c_2$ & ISS input gain & $O(1)$ & $\sqrt{2}$ \\
\bottomrule
\end{tabular}
\end{center}
\end{remark}

\section{Proofs for Algebraic Filtering}
\label{app:algebraic}

This appendix provides proofs for the algebraic properties of the commutator dynamics, corresponding to Theorem~1 in the main text. Throughout, we use the unified notation from Appendix~\ref{app:notation}: $C_e := \mathrm{tf}(C_{\mathrm{sig}})$ is the trace-free signal covariance, $A(M) := M^\top C_e M$, $D(M) := \operatorname{diag}(A)$, and $\Omega(M) := [A, D]$.

\subsection{Proof of Lemma~1.1: Scalar Shift Invariance}
\label{app:scalar_shift}

\begin{lemma}[Scalar Shift Invariance]
\label{lem:scalar_shift}
For any $\alpha \in \mathbb{R}$ and $M \in SO(n)$,
\begin{equation}
\Omega_{C + \alpha I}(M) = \Omega_C(M).
\end{equation}
\end{lemma}

\begin{proof}
Let $A = M^\top C M$ and $D = \operatorname{diag}(A)$. Then
\begin{align}
\Omega_{C + \alpha I}(M) 
&= [M^\top(C + \alpha I)M, \operatorname{diag}(M^\top(C + \alpha I)M)] \\
&= [A + \alpha I, D + \alpha I].
\end{align}
By the bilinearity of the Lie bracket and $[I, X] = 0$ for any $X$:
\begin{align}
[A + \alpha I, D + \alpha I] 
&= [A, D] + \alpha[A, I] + \alpha[I, D] + \alpha^2[I, I] \\
&= [A, D] = \Omega_C(M). \qedhere
\end{align}
\end{proof}

\begin{corollary}[Discrete Pointwise Invariance]
\label{cor:discrete_invariance}
For any observation sequence $C_k = C_{\mathrm{sig}} + \sigma_k^2 I + E_k$ where $\{\sigma_k^2\}$ is an arbitrary time-varying sequence, the discrete trajectory $\{M_k\}$ generated by the Cayley iteration~\eqref{eq:update} coincides pointwise with the trajectory generated by $\{C_{\mathrm{sig}} + E_k\}$.
\end{corollary}

\begin{proof}
By induction. The base case $M_0$ is given. Suppose $M_k$ is identical for both observation sequences. Then by Lemma~\ref{lem:scalar_shift}:
\[
\Omega_k^{(C_{\mathrm{sig}} + \sigma_k^2 I + E_k)} = \Omega_k^{(C_{\mathrm{sig}} + E_k)},
\]
hence $M_{k+1} = M_k \cdot \mathrm{Cay}(\eta_k \Omega_k)$ is identical. The induction is complete.
\end{proof}

\subsection{Proof of Lemma~1.2: Input Mapping Bound}
\label{app:input_bound}

\begin{lemma}[Input Mapping Bound]
\label{lem:input_bound}
Let $A = M^\top C_e M$, $D = \operatorname{diag}(A)$, and $\mathcal{E} = M^\top E_e M$ where $E_e = \mathrm{tf}(E)$. Define the effective input as:
\[
U := \Omega(A + \mathcal{E}) - \Omega(A).
\]
Then
\begin{equation}
\|U\|_F \le 4\|C_e\|_2 \|E_e\|_F + 2\|E_e\|_F^2.
\end{equation}
\end{lemma}

\begin{proof}
Expanding the difference using bilinearity of the Lie bracket:
\begin{align}
U &= [A + \mathcal{E}, D + \operatorname{diag}(\mathcal{E})] - [A, D] \\
&= \underbrace{[A, \operatorname{diag}(\mathcal{E})]}_{\text{(I)}} + \underbrace{[\mathcal{E}, D]}_{\text{(II)}} + \underbrace{[\mathcal{E}, \operatorname{diag}(\mathcal{E})]}_{\text{(III)}}.
\end{align}

We bound each term using the commutator norm inequality $\|[X, Y]\|_F \le 2\|X\|_2 \|Y\|_F$:

\textbf{Term (I):} $\|[A, \operatorname{diag}(\mathcal{E})]\|_F \le 2\|A\|_2 \|\operatorname{diag}(\mathcal{E})\|_F \le 2\|A\|_2 \|\mathcal{E}\|_F$.

\textbf{Term (II):} $\|[\mathcal{E}, D]\|_F \le 2\|\mathcal{E}\|_F \|D\|_2 \le 2\|A\|_2 \|\mathcal{E}\|_F$.

\textbf{Term (III):} $\|[\mathcal{E}, \operatorname{diag}(\mathcal{E})]\|_F \le 2\|\mathcal{E}\|_2 \|\mathcal{E}\|_F \le 2\|\mathcal{E}\|_F^2$.

By orthogonal invariance, $\|A\|_2 = \|C_e\|_2$ and $\|\mathcal{E}\|_F = \|E_e\|_F$. Summing:
\[
\|U\|_F \le 4\|C_e\|_2 \|E_e\|_F + 2\|E_e\|_F^2. \qedhere
\]
\end{proof}

\subsection{Proof of Theorem~1: Structural Filtering Theorem}
\label{app:thm1}

\begin{proof}[Proof of Theorem~1]
\textbf{(i) Algebraic identity.} This is Lemma~\ref{lem:scalar_shift}.

\textbf{(ii) Trajectory invariance.} This is Corollary~\ref{cor:discrete_invariance}. Note also that $\operatorname{off}(M^\top(C + \alpha I)M) = \operatorname{off}(M^\top C M)$ because $\operatorname{off}(\alpha I) = 0$, so the Lyapunov function $f(M) = \frac{1}{2}\|\operatorname{off}(A)\|_F^2$ is likewise invariant under scalar shifts.

\textbf{(iii) Input bound.} This is Lemma~\ref{lem:input_bound}.

\textbf{(iv) Stability constant.} The bound $\eta_{\max} = 1/L_C$ with $L_C = c_n\|C_e\|_2^2$ and $c_n = 4\sqrt{n} + 8$ is proved via the explicit Cayley pullback descent lemma in Appendix~\ref{app:discrete_descent}. Since $C_e = \mathrm{tf}(C)$ is unchanged by scalar shifts, the stability threshold is independent of $\sigma^2$.
\end{proof}

\begin{remark}[Impulse and Unbounded Noise Immunity]
\label{rem:impulse_immunity}
Because the discrete trajectory is pointwise $\sigma$-invariant by (i)--(ii), the algorithm automatically tolerates impulse noise, unbounded sequences, and any time-varying $\sigma_k^2$ without degradation. No explicit denoising or trace estimation is required.
\end{remark}

\section{Geometric Analysis and Global Convergence}
\label{app:geometric}

This appendix establishes input-to-state stability (ISS)~\citep{sontag2008input} for the commutator flow on $SO(n)$ and proves global convergence of both the continuous-time gradient flow and the discrete Cayley iteration. Throughout, we use the unified notation from Appendix~\ref{app:notation}.

\subsection{Local Dissipation (Polyak--\L{}ojasiewicz Condition)}
\label{app:local_dissipation}

\begin{lemma}[Local Dissipation]
\label{lem:local_dissipation}
In the spectrally separated domain $\mathcal{N}_{\underline{\delta}} = \{M \in SO(n) : \delta(M) \ge \underline{\delta}\}$:
\begin{equation}
\|\Omega(M)\|_F^2 \ge 2\underline{\delta}^2 f(M).
\end{equation}
\end{lemma}

\begin{proof}
Let $A = M^\top C_e M$ with diagonal entries $A_{ii}$. The commutator has entries
\[
\Omega_{ij} = (A_{jj} - A_{ii}) A_{ij}, \quad i \neq j,
\]
while $\Omega_{ii} = 0$. Therefore
\begin{align}
\|\Omega\|_F^2 &= 2\sum_{i < j} (A_{ii} - A_{jj})^2 A_{ij}^2, \\
f(M) &= \frac{1}{2}\|\operatorname{off}(A)\|_F^2 = \sum_{i < j} A_{ij}^2.
\end{align}
In $\mathcal{N}_{\underline{\delta}}$, we have $(A_{ii} - A_{jj})^2 \ge \underline{\delta}^2$ for all $i \neq j$, hence
\[
\|\Omega\|_F^2 \ge 2\underline{\delta}^2 \sum_{i < j} A_{ij}^2 = 2\underline{\delta}^2 f(M). \qedhere
\]
\end{proof}

\subsection{Spectral Sandwiching}
\label{app:spectral_sandwich}

\begin{lemma}[Operator Spectral Sandwiching]
\label{app1:spectral_sandwich}
For $M \in \mathcal{N}_{\underline{\delta}}$, let $A = M^\top C_e M$, $D = \operatorname{diag}(A)$, and $X = \operatorname{off}(A)$. Define the commutator operator $\mathcal{L}_A(X) := [X, D]$. Then
\begin{equation}
\underline{\delta} \|X\|_F \le \|\mathcal{L}_A(X)\|_F \le 2\|C_e\|_2 \|X\|_F.
\end{equation}
Equivalently, in terms of $f$ and $\Omega$:
\begin{equation}
2\underline{\delta}^2 f(M) \le \|\Omega(M)\|_F^2 \le 8\|C_e\|_2^2 f(M).
\end{equation}
\end{lemma}

\begin{proof}
\textbf{Lower bound.} By component-wise computation,
\[
[\operatorname{off}(A), D]_{ij} = (A_{jj} - A_{ii}) A_{ij},
\]
hence
\[
\|\mathcal{L}_A(X)\|_F^2 = 2\sum_{i < j} (A_{ii} - A_{jj})^2 A_{ij}^2 \ge \underline{\delta}^2 \cdot 2\sum_{i < j} A_{ij}^2 = \underline{\delta}^2 \|X\|_F^2.
\]

\textbf{Upper bound.} We have $|A_{ii} - A_{jj}| \le |A_{ii}| + |A_{jj}| \le 2\|A\|_2 = 2\|C_e\|_2$, thus
\[
\|\mathcal{L}_A(X)\|_F^2 \le (2\|C_e\|_2)^2 \cdot 2\sum_{i < j} A_{ij}^2 = 4\|C_e\|_2^2 \|X\|_F^2. \qedhere
\]
\end{proof}

\subsection{ISS Differential Inequality}
\label{app:iss_diff}

\begin{lemma}[ISS Differential Inequality]
\label{lem:iss_diff}
Let $y(t) := \sqrt{f(M(t))}$. In $\mathcal{N}_{\underline{\delta}}$, the Dini derivative satisfies:
\begin{equation}
D^+ y(t) \le -\underline{\delta}^2 y(t) + \sqrt{2}\|C_e\|_2 \cdot u_{\mathrm{eff}}(t),
\end{equation}
where $u_{\mathrm{eff}}(t) := \|U(t)\|_F + \frac{\rho}{2\|C_e\|_2}$ and $\rho := \|\dot{C}(t)\|_F$ is the drift rate.
\end{lemma}

\begin{proof}
Let
\[
A(t):=M(t)^\top C_e(t)M(t),
\qquad
X(t):=\operatorname{off}(A(t)),
\qquad
f(t)=\frac12\|X(t)\|_F^2.
\]
For the perturbed dynamics $\dot M = M(\Omega+U)$, with $\Omega=[A,D]$ and $D=\operatorname{diag}(A)$, differentiation gives
\begin{equation}
\dot A = [A,\Omega+U] + M^\top \dot C_e\, M.
\label{eq:iss_A_dot}
\end{equation}
Since $\operatorname{off}$ is an orthogonal projection,
\[
\dot f = \langle X,\operatorname{off}(\dot A)\rangle.
\]
By the directional-derivative identity from Appendix~\ref{app:technical},
\[
\mathrm{D}f(M)[M\Xi] = -\langle \Omega,\Xi\rangle_F
\qquad (\Xi\in\mathfrak{so}(n)),
\]
hence, for every skew-symmetric $\Xi$,
\[
\langle X,\operatorname{off}([A,\Xi])\rangle = -\langle \Omega,\Xi\rangle_F.
\]
Applying this to \eqref{eq:iss_A_dot} with $\Xi=\Omega+U$ yields
\begin{align}
\dot f
&=
\langle X,\operatorname{off}([A,\Omega+U])\rangle
+
\left\langle X,\operatorname{off}(M^\top \dot C_e M)\right\rangle \notag\\
&=
-\|\Omega\|_F^2
-
\langle \Omega,U\rangle_F
+
\left\langle X,\operatorname{off}(M^\top \dot C_e M)\right\rangle.
\label{eq:iss_f_dot_exact}
\end{align}
Therefore, by Cauchy--Schwarz and orthogonal invariance,
\begin{align}
\dot f
&\le
-\|\Omega\|_F^2
+
\|\Omega\|_F\|U\|_F
+
\sqrt{2}\rho\, y.
\label{eq:iss_f_dot_bound}
\end{align}
In $\mathcal N_{\underline\delta}$, Lemmas~\ref{lem:local_dissipation} and~\ref{lem:spectral_sandwich} give
\[
\|\Omega\|_F^2 \ge 2\underline\delta^2 y^2,
\qquad
\|\Omega\|_F \le 2\sqrt{2}\|C_e\|_2 y.
\]
Substituting into \eqref{eq:iss_f_dot_bound} gives
\begin{equation}
\dot f
\le
-2\underline\delta^2 y^2
+
2\sqrt{2}\|C_e\|_2 y\|U\|_F
+
\sqrt{2}\rho\, y.
\label{eq:iss_f_dot_y}
\end{equation}

If $y(t)>0$, then $\dot y = \dot f/(2y)$, so \eqref{eq:iss_f_dot_y} implies
\[
\dot y
\le
-\underline\delta^2 y
+
\sqrt{2}\|C_e\|_2\|U\|_F
+
\frac{\rho}{\sqrt{2}}.
\]
If $y(t)=0$, then $A(t)$ is diagonal and $\Omega(t)=0$. From \eqref{eq:iss_A_dot},
\[
D^+y(t)
=
\frac{1}{\sqrt{2}}\left\|
\operatorname{off}\!\left([A(t),U(t)] + M(t)^\top \dot C_e(t) M(t)\right)
\right\|_F
\le
\sqrt{2}\|C_e\|_2\|U(t)\|_F + \frac{\rho}{\sqrt{2}},
\]
where we used $\|[A,U]\|_F\le 2\|A\|_2\|U\|_F = 2\|C_e\|_2\|U\|_F$. Hence the same inequality holds for the upper Dini derivative at $y=0$.

Finally, since
\[
\sqrt{2}\|C_e\|_2\!\left(\|U\|_F + \frac{\rho}{2\|C_e\|_2}\right)
=
\sqrt{2}\|C_e\|_2\|U\|_F + \frac{\rho}{\sqrt{2}},
\]
we obtain the stated bound.
\end{proof}

\subsection{Geometric ISS Theorem}
\label{app:thm2}

\begin{proof}[Proof of Theorem~2]
Integrating Lemma~\ref{lem:iss_diff} via Gronwall's inequality with $\lambda := \underline{\delta}^2$:

\textbf{(i) Exponential convergence:}
\[
y(t) \le e^{-\lambda t} y(0) + \int_0^t e^{-\lambda(t-s)} \sqrt{2}\|C_e\|_2 \cdot u_{\mathrm{eff}}(s) \, ds.
\]
If $\sup_t u_{\mathrm{eff}}(t) \le \bar{u}$:
\[
y(t) \le e^{-\underline{\delta}^2 t} y(0) + \frac{\sqrt{2}\|C_e\|_2}{\underline{\delta}^2} \bar{u}.
\]

\textbf{(ii) Steady-state ball:}
Taking $t \to \infty$:
\[
\limsup_{t \to \infty} \sqrt{f(M(t))} \le r_f(\bar{u}) := \frac{\sqrt{2}\|C_e\|_2}{\underline{\delta}^2} \bar{u}. \qedhere
\]
\end{proof}

\subsection{Critical Point Structure and Strict Saddle Property}
\label{app:strict_saddle}

\subsubsection{Convergence of the Continuous-Time Flow}

\begin{theorem}[Asymptotic Convergence of Commutator Flow]
\label{thm:flow_converges}
Consider the ODE $\dot{M}(t) = M(t)\Omega(M(t))$ with $M(0) \in SO(n)$. Then:
\begin{enumerate}[label=(\roman*)]
\item The solution $M(t)$ exists globally and remains in $SO(n)$;
\item $t \mapsto f(M(t))$ is monotonically non-increasing and converges to some limit $f_\infty$;
\item Every $\omega$-limit point $M_\infty$ satisfies $\Omega(M_\infty) = 0$;
\item Since $f$ is real-analytic on the compact analytic manifold $SO(n)$, $M(t)$ converges to a \textbf{single} critical point $M_\infty$.
\end{enumerate}
\end{theorem}

\begin{proof}
\textbf{(i)} The vector field $F(M):=M\Omega(M)$ is smooth on the compact manifold $SO(n)$, hence globally Lipschitz in any smooth embedding. Therefore the ODE has a unique global solution. To prove invariance of $SO(n)$, note that $\Omega(M)^\top=-\Omega(M)$, so
\[
\frac{d}{dt}\big(M^\top M\big)
=
\Omega^\top + \Omega
=
0.
\]
Since $M(0)^\top M(0)=I$, we obtain $M(t)^\top M(t)=I$ for all $t$.

\textbf{(ii)} By Appendix~\ref{app:technical}, $\operatorname{grad} f(M)=-M\Omega(M)$. Hence the flow is exactly the negative gradient flow:
\[
\dot M = -\operatorname{grad} f(M).
\]
Therefore
\[
\frac{d}{dt} f(M(t))
=
-\|\Omega(M(t))\|_F^2
\le 0.
\]
Since $f\ge 0$, the monotone function $t\mapsto f(M(t))$ converges to some $f_\infty\ge0$.

\textbf{(iii)} Let $M_\infty$ be an $\omega$-limit point. Then there exists $t_k\to\infty$ with $M(t_k)\to M_\infty$. Fix $s\ge0$. By continuity of the flow map,
\[
M(t_k+s)=\Phi_s(M(t_k)) \longrightarrow \Phi_s(M_\infty).
\]
Because $f(M(t))\to f_\infty$, we also have
\[
f(\Phi_s(M_\infty))
=
f_\infty
=
f(M_\infty).
\]
Thus $s\mapsto f(\Phi_s(M_\infty))$ is constant on $[0,\infty)$. Differentiating at $s=0$,
\[
0
=
\frac{d}{ds}\bigg|_{s=0} f(\Phi_s(M_\infty))
=
-\|\Omega(M_\infty)\|_F^2,
\]
so $\Omega(M_\infty)=0$.

\textbf{(iv)} Since $f$ is real-analytic on the compact analytic manifold $SO(n)$, the {\L}ojasiewicz--Simon gradient inequality applies in a neighborhood of every critical point~\citep{lojasiewicz1963propriete,simon1983asymptotics}. Combined with~(ii), this implies finite trajectory length and convergence of the whole trajectory to a single critical point; see~\citep[Theorem~4.18]{absil2008optimization}.
\end{proof}

\subsubsection{Critical Points and Strict Saddles}

\begin{lemma}[Critical Point Characterization]
\label{lem:critical_characterization}
$M \in SO(n)$ is a critical point of $f$ if and only if $\Omega(M) = 0$.
Equivalently, for all $i \neq j$:
\begin{equation}
(A_{jj} - A_{ii}) A_{ij} = 0, \qquad A := M^\top C_e M.
\label{eq:critical_condition}
\end{equation}
\end{lemma}

\begin{proof}
By Appendix~\ref{app:technical}, $\operatorname{grad} f(M) = -M\Omega(M)$. Hence $\operatorname{grad} f(M)=0$ if and only if $\Omega(M)=0$, since $M$ is invertible. The entrywise condition follows from $\Omega_{ij}=(A_{jj}-A_{ii})A_{ij}$.
\end{proof}

\begin{lemma}[Critical Point Dichotomy]
\label{lem:critical_dichotomy}
At any critical point $M$ of $f$:
\begin{enumerate}[label=(\alph*)]
\item \textbf{Global Minimum:} If $\operatorname{off}(A) = 0$, then $f(M) = 0$ and $M$ maps eigenvectors of $C_e$ to the standard basis (up to permutation and sign).
\item \textbf{Degenerate Block:} If $\operatorname{off}(A) \neq 0$, then there exist indices $i < j$ such that $A_{ij} \neq 0$ and $A_{ii} = A_{jj}$.
\end{enumerate}
\end{lemma}

\begin{proof}
If $\operatorname{off}(A)=0$, then $A$ is diagonal. Because $A=M^\top C_e M$ is orthogonally similar to $C_e$, it has the same eigenvalues. Under the standing assumption that $C_e$ has distinct eigenvalues, the diagonal entries of $A$ are exactly those eigenvalues in some order. Hence the columns of $M$ are eigenvectors of $C_e$ up to permutation and sign, and $f(M)=0$.

If $\operatorname{off}(A)\neq0$, choose $i\neq j$ with $A_{ij}\neq0$. The critical-point condition \eqref{eq:critical_condition} forces $A_{ii}=A_{jj}$.
\end{proof}

\begin{lemma}[Negative Curvature at Degenerate Blocks]
\label{lem:negative_curvature}
Let $M \in SO(n)$ and $A := M^\top C_e M$. Suppose there exist $i \neq j$ with $A_{ii} = A_{jj}$ and $b := A_{ij} \neq 0$. Let $\Xi := E_{ij} - E_{ji} \in \mathfrak{so}(n)$. Then
\begin{equation}
f(M e^{t\Xi}) = f(M) - b^2 \sin^2(2t),
\label{eq:exact_profile}
\end{equation}
and consequently
\begin{equation}
\operatorname{Hess} f(M)[M\Xi, M\Xi] = -8b^2 < 0, \qquad \lambda_{\min}(\operatorname{Hess} f(M)) \le -4b^2.
\label{eq:negative_hessian}
\end{equation}
\end{lemma}

\begin{proof}
The matrix $G(t) := e^{t\Xi}$ acts as a $2 \times 2$ rotation in the $(i,j)$-plane. For $k \notin \{i,j\}$, the pair $(A_{ik}(t), A_{jk}(t))$ is rotated orthogonally, so $A_{ik}(t)^2 + A_{jk}(t)^2 = A_{ik}^2 + A_{jk}^2$. Entries $A_{pq}$ with $\{p,q\} \cap \{i,j\} = \emptyset$ are unchanged.

Standard Givens conjugation gives $A_{ij}(t) = (\cos^2 t - \sin^2 t) A_{ij} + (\cos t \sin t)(A_{jj} - A_{ii})$. Under $A_{ii} = A_{jj}$, this simplifies to $A_{ij}(t) = b \cos(2t)$. Since only $A_{ij}(t)^2$ changes in $f = \sum_{p<q} A_{pq}^2$:
\[
f(Me^{t\Xi}) = f(M) - b^2 + b^2 \cos^2(2t) = f(M) - b^2 \sin^2(2t).
\]
Since $\sin^2(2t) = 4t^2 + O(t^4)$, the second derivative at $t=0$ is $-8b^2 < 0$. The eigenvalue bound follows from $\|\Xi\|_F^2 = 2$.
\end{proof}

\begin{proposition}[Strict Saddle Property]
\label{prop:strict_saddle}
When $C_e$ has distinct eigenvalues, every critical point of $f$ on $SO(n)$ is either a global minimum ($f = 0$) or a strict saddle ($\operatorname{Hess} f$ has at least one negative eigenvalue). No spurious local minima exist.
\end{proposition}

\begin{proof}
If $\operatorname{off}(A) = 0$, then $M$ is a global minimum by Lemma~\ref{lem:critical_dichotomy}(a). If $\operatorname{off}(A) \neq 0$, there exists $(i,j)$ with $A_{ij} \neq 0$ and $A_{ii} = A_{jj}$, and Lemma~\ref{lem:negative_curvature} gives $\lambda_{\min}(\operatorname{Hess} f(M)) \le -4A_{ij}^2 < 0$.
\end{proof}

\subsubsection{Hessian--Linearization Identity}

\begin{lemma}[Hessian--Linearization Identity at Critical Points]
\label{lem:hessian_link}
Let $M$ be a critical point of $f$, and let $\eta=M\Xi\in T_MSO(n)$ with $\Xi\in\mathfrak{so}(n)$. Then
\begin{equation}
\operatorname{Hess}\,f(M)[\eta] = -M\,D\Omega(M)[\eta].
\label{eq:hess_lin_operator}
\end{equation}
Consequently,
\begin{equation}
\operatorname{Hess}\,f(M)[M\Xi,M\Xi]
=
-\langle D\Omega(M)[M\Xi],\Xi\rangle_F.
\label{eq:hess_lin_quadratic}
\end{equation}
\end{lemma}

\begin{proof}
Appendix~\ref{app:technical} gives $\operatorname{grad} f(M)=-M\Omega(M)$. Let $M(t)=M\exp(t\Xi)$ be the geodesic with $M(0)=M$ and $\dot M(0)=M\Xi$. Differentiating the gradient along this geodesic,
\[
\frac{D}{dt}\bigg|_{t=0}\operatorname{grad} f(M(t))
=
-\frac{D}{dt}\bigg|_{t=0}\big(M(t)\Omega(M(t))\big).
\]
Because $M$ is critical, $\Omega(M)=0$, so the derivative of the prefactor $M(t)$ drops out and
\[
\operatorname{Hess}\,f(M)[M\Xi]
=
-M\,D\Omega(M)[M\Xi].
\]
This proves \eqref{eq:hess_lin_operator}. Taking the Frobenius inner product with $M\Xi$ and using orthogonal invariance of left multiplication by $M$ yields \eqref{eq:hess_lin_quadratic}.
\end{proof}

\subsection{Continuous-Time: Almost Sure Global Convergence}
\label{app:ct_global}

\begin{lemma}[Linearization of the Flow at Critical Points]
\label{lem:flow_linearization}
Let $F(M):=M\Omega(M)=-\operatorname{grad} f(M)$. If $M_\star$ is a critical point of $f$ and $\eta\in T_{M_\star}SO(n)$, then
\[
DF(M_\star)[\eta] = -\operatorname{Hess} f(M_\star)[\eta].
\]
\end{lemma}

\begin{proof}
Write $\eta=M_\star\Xi$ with $\Xi\in\mathfrak{so}(n)$. Since $M_\star$ is critical, Lemma~\ref{lem:critical_characterization} gives $\Omega(M_\star)=0$. Therefore
\[
DF(M_\star)[\eta]
=
\eta\,\Omega(M_\star)+M_\star\,D\Omega(M_\star)[\eta]
=
M_\star\,D\Omega(M_\star)[\eta].
\]
By Lemma~\ref{lem:hessian_link},
$\operatorname{Hess} f(M_\star)[\eta]
=
-M_\star\,D\Omega(M_\star)[\eta]$.
Combining yields the claim.
\end{proof}

\begin{lemma}[Chart-Null Sets are Haar-Null on $SO(n)$]
\label{lem:haar_null}
If $N\subset SO(n)$ has measure zero in the manifold-chart sense, then $N$ has Haar measure zero.
\end{lemma}

\begin{proof}
Equip $SO(n)$ with the Frobenius bi-invariant Riemannian metric and let $\mathrm{vol}$ denote the associated Riemannian volume measure. Left multiplication $L_Q(M):=QM$ is an isometry for every $Q\in SO(n)$ because
\[
\langle QX,QY\rangle_F=\operatorname{tr}(Y^\top Q^\top QX)=\operatorname{tr}(Y^\top X)=\langle X,Y\rangle_F.
\]
Hence $\mathrm{vol}(L_QE)=\mathrm{vol}(E)$ for every Borel set $E\subset SO(n)$. Since $SO(n)$ is compact, $\mathrm{vol}(SO(n))<\infty$, and therefore
\[
\mu_{\mathrm{Haar}}(E):=\frac{\mathrm{vol}(E)}{\mathrm{vol}(SO(n))}
\]
is a left-invariant Borel probability measure on $SO(n)$, i.e.\ a normalized Haar measure.

Because $SO(n)$ is compact, it admits a finite smooth atlas $\{(U_a,\phi_a)\}_{a=1}^J$. In each chart, the Riemannian volume has the form
\[
\mathrm{vol}(E\cap U_a)
=
\int_{\phi_a(E\cap U_a)} \rho_a(x)\,dx
\]
for some smooth strictly positive density $\rho_a$. If $\phi_a(N\cap U_a)$ has Lebesgue measure zero, then $\mathrm{vol}(N\cap U_a)=0$. Summing over the finite atlas gives $\mathrm{vol}(N)=0$, hence $\mu_{\mathrm{Haar}}(N)=0$.
\end{proof}

\begin{theorem}[Almost Sure Global Convergence]
\label{thm:as_global}
Assume $C_e$ has distinct eigenvalues. Let $M(\cdot)$ solve
\[
\dot M = M\Omega(M) = -\operatorname{grad} f(M),
\qquad
M(0)\in SO(n).
\]
If $M(0)$ is sampled from Haar measure on $SO(n)$, then with probability $1$ the trajectory converges to a global minimum of $f$: there exists $M_\infty\in SO(n)$ with
\[
\lim_{t\to\infty} M(t)=M_\infty,
\qquad
f(M_\infty)=0.
\]
In particular, the basin of attraction of the non-optimal critical set
\[
\mathcal S_{\mathrm{bad}}
:=
\{M\in SO(n): \Omega(M)=0,\ f(M)>0\}
\]
has Haar measure zero.
\end{theorem}

\begin{proof}
Let $F(M):=M\Omega(M)$. By Theorem~\ref{thm:flow_converges}, every trajectory of $\dot M=F(M)$ converges to a single critical point of $f$, and the equilibria are exactly the critical points (Lemma~\ref{lem:critical_characterization}).

Fix $M_\star\in \mathcal S_{\mathrm{bad}}$. By Proposition~\ref{prop:strict_saddle}, $M_\star$ is a strict saddle, so $\operatorname{Hess} f(M_\star)$ has a negative eigenvalue. By Lemma~\ref{lem:flow_linearization},
\[
DF(M_\star) = -\operatorname{Hess} f(M_\star),
\]
so $DF(M_\star)$ has a positive eigenvalue. Hence $M_\star$ is an exponentially unstable equilibrium. Since $M_\star\in\mathcal S_{\mathrm{bad}}$ was arbitrary, every point of $\mathcal S_{\mathrm{bad}}$ is exponentially unstable.

For a $C^1$ flow on a compact manifold in which every trajectory converges to an equilibrium, the region of attraction of any set consisting entirely of exponentially unstable equilibria has manifold measure zero~\citep[Proposition~18]{markdahl2018almost}. Applying this with $X=SO(n)$, $S=\mathcal S_{\mathrm{bad}}$, and $\dot M=F(M)$ shows that the basin of $\mathcal S_{\mathrm{bad}}$ has chart measure zero. Lemma~\ref{lem:haar_null} upgrades this to Haar measure zero.

Since every trajectory converges to some equilibrium, every initial condition outside the exceptional basin converges to $M_\infty\notin \mathcal S_{\mathrm{bad}}$, hence $f(M_\infty)=0$.
\end{proof}

\begin{corollary}[Qualitative Finite-Time Entry into the Separated Domain]
\label{cor:ae_eventual_entry}
Fix any $\underline\delta\in(0,g)$. For Haar-almost every $M(0)\in SO(n)$ there exists a finite time $T_{\mathrm{enter}}<\infty$ such that $M(t)\in\mathcal N_{\underline\delta}$ for all $t\ge T_{\mathrm{enter}}$.
\end{corollary}

\begin{proof}
By Theorem~\ref{thm:as_global}, the trajectory converges to some $M_\infty$ with $f(M_\infty)=0$. Then $A_\infty := M_\infty^\top C_e M_\infty$ is diagonal with eigenvalues of $C_e$ in some order, so $\delta(M_\infty) = g > \underline\delta$. Since $M\mapsto \delta(M)$ is continuous, there exists a neighborhood $U$ of $M_\infty$ with $\delta(M)>\underline\delta$ on $U$. Convergence $M(t)\to M_\infty$ provides $T_{\mathrm{enter}}$ such that $M(t)\in U\subset \mathcal N_{\underline\delta}$ for all $t\ge T_{\mathrm{enter}}$.
\end{proof}

\begin{remark}[State-Dependent $D(M)$]
\label{rem:state_dependent_D}
Classical double-bracket flow theory assumes $\dot{X} = [X, [X, N]]$ with fixed $N$. Our flow uses $D=\operatorname{diag}(A(M))$, which varies with the state. The almost-sure convergence proof does not require a Morse--Bott stratification; it combines pointwise convergence of the analytic gradient flow (Theorem~\ref{thm:flow_converges}), strict-saddle instability (Proposition~\ref{prop:strict_saddle} with Lemma~\ref{lem:flow_linearization}), and the measure-zero attraction theorem for exponentially unstable equilibria.
\end{remark}

\subsection{Discrete-Time: Cayley Iteration}
\label{app:discrete_global}

This subsection proves unconditional global convergence for the exact Cayley iteration $M_{k+1} = M_k \,\mathrm{Cay}(\eta\Omega(M_k))$ under the step-size condition $0<\eta<1/L_C$, where $L_C = c_n\|C_e\|_2^2$ is the Cayley pullback smoothness constant from Appendix~\ref{app:discrete_descent}.

\begin{lemma}[Cayley Differential Identities]
\label{lem:cayley_diff}
For skew-symmetric $X,Z\in\mathfrak{so}(n)$,
\[
\mathrm{Cay}(X)-I=(I-\tfrac12 X)^{-1}X,
\qquad
D\mathrm{Cay}_X[Z]=(I-\tfrac12 X)^{-1}\,\frac Z2\,\bigl(I+\mathrm{Cay}(X)\bigr).
\]
Moreover,
\[
\|\mathrm{Cay}(X)-I\|_F\le \|X\|_F,
\qquad
\|D\mathrm{Cay}_X[Z]\|_F\le \|Z\|_F.
\]
\end{lemma}

\begin{proof}
The identity $\mathrm{Cay}(X)-I = (I-\tfrac12 X)^{-1}(I+\tfrac12 X)-I = (I-\tfrac12 X)^{-1}X$ is immediate. Writing $A(X):=I-\tfrac12 X$ and $B(X):=I+\tfrac12 X$, the product rule gives
\[
D\mathrm{Cay}_X[Z]
=
A(X)^{-1}\frac Z2 A(X)^{-1}B(X)+A(X)^{-1}\frac Z2
=
A(X)^{-1}\frac Z2 \bigl(\mathrm{Cay}(X)+I\bigr).
\]
If $X^\top=-X$, then $(I-\tfrac12 X)^\top(I-\tfrac12 X)=I-\frac14 X^2$ has eigenvalues at least $1$, so $\|(I-\tfrac12 X)^{-1}\|_2\le 1$. Also $\mathrm{Cay}(X)$ is orthogonal, hence $\|I+\mathrm{Cay}(X)\|_2\le 2$. Therefore
\[
\|\mathrm{Cay}(X)-I\|_F\le \|X\|_F
\quad\text{and}\quad
\|D\mathrm{Cay}_X[Z]\|_F\le \|Z\|_F. \qedhere
\]
\end{proof}

\begin{lemma}[Derivative Bound for $\Omega$]
\label{lem:domega_bound}
For every $M\in SO(n)$ and every $H\in T_MSO(n)$,
\[
\|D\Omega(M)[H]\|_F \le 8\|C_e\|_2^2\,\|H\|_F.
\]
\end{lemma}

\begin{proof}
Let $dA = H^\top C_e M + M^\top C_e H$ and $dD = \operatorname{diag}(dA)$. Then $\|dA\|_F \le 2\|C_e\|_2\|H\|_F$ and $\|dD\|_F \le \|dA\|_F$. Since $D\Omega(M)[H]=[dA,D]+[A,dD]$, Lemma~\ref{lem:commutator_bound} gives
\[
\|D\Omega(M)[H]\|_F
\le
2\|dA\|_F\|D\|_2 + 2\|A\|_2\|dD\|_F
\le
4\|C_e\|_2\|dA\|_F
\le
8\|C_e\|_2^2\|H\|_F. \qedhere
\]
\end{proof}

\begin{lemma}[Local Diffeomorphism Property of the Exact Cayley Map]
\label{lem:local_diffeo}
Let $0<\eta<1/L_C$ and define $T_\eta(M):=M\,\mathrm{Cay}(\eta\Omega(M))$. Then $T_\eta:SO(n)\to SO(n)$ is a $C^\infty$ local diffeomorphism. Moreover, if $N\subset SO(n)$ is chart-null, then $T_\eta^{-1}(N)$ is chart-null.
\end{lemma}

\begin{proof}
The map $T_\eta$ is smooth because $\Omega$ and $\mathrm{Cay}$ are smooth. Fix $M\in SO(n)$ and write $Q:=\mathrm{Cay}(\eta\Omega(M))$. For $H\in T_MSO(n)$,
\[
DT_\eta(M)[H]
=
HQ + M\,D\mathrm{Cay}_{\eta\Omega(M)}\!\bigl[\eta D\Omega(M)[H]\bigr].
\]
Right-multiplying by $Q^\top$ and using Lemma~\ref{lem:cayley_diff},
\[
\|DT_\eta(M)[H]Q^\top - H\|_F
\le
\eta \|D\Omega(M)[H]\|_F
\le
8\eta\|C_e\|_2^2\|H\|_F
\]
by Lemma~\ref{lem:domega_bound}. Because $\eta<1/L_C$ and $L_C \ge 8\|C_e\|_2^2$, the factor $8\eta\|C_e\|_2^2 < 1$. Hence $DT_\eta(M)$ is injective, thus bijective, and the inverse function theorem yields that $T_\eta$ is a local diffeomorphism.

For the null-set statement, cover $SO(n)$ by finitely many neighborhoods $U_j$ on which $T_\eta$ restricts to a diffeomorphism onto an open set $V_j$. In charts, each inverse branch is $C^1$, hence locally Lipschitz, and maps Lebesgue-null sets to Lebesgue-null sets. Therefore $T_\eta^{-1}(N) = \bigcup_j (T_\eta|_{U_j})^{-1}(N\cap V_j)$ is chart-null.
\end{proof}

\begin{lemma}[Riemannian {\L}ojasiewicz Inequality on $SO(n)$]
\label{lem:lojasiewicz_SO}
Let $P\in SO(n)$ be a critical point of $f$. Then there exist a neighborhood $U$ of $P$, a constant $c>0$, and an exponent $\mu\in[0,1)$ such that
\[
\|\Omega(M)\|_F \ge c\,|f(M)-f(P)|^\mu
\qquad
(M\in U).
\]
\end{lemma}

\begin{proof}
Choose a real-analytic chart $\phi:U_0\to V_0\subset\mathbb R^m$ with $\phi(P)=0$, and define $g:=f\circ\phi^{-1}$. Since $f$ is polynomial in the matrix entries, $g$ is real analytic. By the Euclidean {\L}ojasiewicz gradient inequality~\citep{lojasiewicz1963propriete}, there exist $c_0>0$, $\mu\in[0,1)$, and a neighborhood $V_1$ of $0$ such that
\[
\|\nabla g(x)\|_2 \ge c_0|g(x)-g(0)|^\mu
\qquad (x\in V_1).
\]
Shrinking $V_1$ if necessary, the local metric matrix $G(x)$ satisfies $G(x)\preceq \beta I$ for some $\beta>0$. Therefore
\[
\|\operatorname{grad} f(\phi^{-1}(x))\|_F^2
=
\nabla g(x)^\top G(x)^{-1}\nabla g(x)
\ge
\beta^{-1}\|\nabla g(x)\|_2^2.
\]
Setting $U:=\phi^{-1}(V_1)$ and $c:=\beta^{-1/2}c_0$ gives the claim, since $\|\operatorname{grad} f(M)\|_F=\|\Omega(M)\|_F$.
\end{proof}

\begin{lemma}[Single-Limit Capture]
\label{lem:single_limit}
Let $0<\eta<1/L_C$ and let $\{M_k\}$ be an exact Cayley orbit $M_{k+1}=T_\eta(M_k)$. If $P$ is an accumulation point of $\{M_k\}$, then $M_k\to P$.
\end{lemma}

\begin{proof}
By the discrete Lyapunov descent lemma (Appendix~\ref{app:discrete_descent}),
\[
f(M_{k+1}) \le f(M_k)-a\|\Omega_k\|_F^2,
\qquad
a:=\eta\!\left(1-\frac{\eta L_C}{2}\right)>0.
\]
Hence $f(M_k)$ decreases to some $f_\infty$ and $\sum_{k=0}^\infty \|\Omega_k\|_F^2 < \infty$, so $\Omega_k\to 0$. By Lemma~\ref{lem:cayley_diff}, $\|M_{k+1}-M_k\|_F \le \eta\|\Omega_k\|_F \to 0$.

Let $k_j\to\infty$ with $M_{k_j}\to P$. By continuity, $f(P)=f_\infty$ and $\Omega(P)=0$, so $P$ is critical.

Apply Lemma~\ref{lem:lojasiewicz_SO} at $P$: there exist $U$, $c>0$, $\mu\in[0,1)$ with $\|\Omega(M)\|_F \ge c|f(M)-f(P)|^\mu$ on $U$. Choose $r>0$ with $\overline{B_{4r}(P)}\cap SO(n)\subset U$. Choose $K$ such that $M_K\in B_r(P)$ and
\[
\frac{\eta}{ac(1-\mu)}(f(M_K)-f(P))^{1-\mu}<r.
\]
If $f(M_K)=f(P)$, monotonicity forces $\Omega_k=0$ and $M_{k+1}=M_k$ for all $k\ge K$. Otherwise, for each $j\ge K$ with $M_j\in U$, combining the descent and {\L}ojasiewicz inequalities with $\|M_{j+1}-M_j\|_F\le \eta\|\Omega_j\|_F$ yields
\[
\|M_{j+1}-M_j\|_F
\le
\frac{\eta}{ac(1-\mu)}
\bigl((f(M_j)-f(P))^{1-\mu}-(f(M_{j+1})-f(P))^{1-\mu}\bigr).
\]
Summing telescopically shows $\sum_{j=K}^{J}\|M_{j+1}-M_j\|_F < r$ whenever all iterates $M_K,\ldots,M_J$ lie in $B_{2r}(P)$. Since $M_K\in B_r(P)$, induction gives $M_j\in B_{2r}(P)$ for all $j\ge K$, and $\sum_{j\ge K}\|M_{j+1}-M_j\|_F<\infty$. The tail is Cauchy and converges to the unique accumulation point $P$.
\end{proof}

\begin{lemma}[Fixed Points and Linearization of the Cayley Map]
\label{lem:discrete_linearization}
Let $0<\eta<1/L_C$. Then:
\begin{enumerate}[label=(\roman*)]
\item $T_\eta(M)=M$ if and only if $\Omega(M)=0$;
\item if $P$ is critical and $H\in T_PSO(n)$, then
$DT_\eta(P)[H]=H-\eta\,\operatorname{Hess} f(P)[H]$.
\end{enumerate}
\end{lemma}

\begin{proof}
For (i), $T_\eta(M)=M$ implies $\mathrm{Cay}(\eta\Omega(M))=I$, hence $\Omega(M)=0$. The converse is immediate. For (ii), since $\Omega(P)=0$ and $D\mathrm{Cay}_0[Z]=Z$ by Lemma~\ref{lem:cayley_diff},
\[
DT_\eta(P)[H]
=
H + \eta P\,D\Omega(P)[H]
=
H-\eta\,\operatorname{Hess} f(P)[H],
\]
where the last equality uses Lemma~\ref{lem:hessian_link}.
\end{proof}

\begin{lemma}[Local Center-Stable Manifold at a Non-Optimal Fixed Point]
\label{lem:local_cs}
Let $P$ be a non-optimal critical point of $f$, and let $0<\eta<1/L_C$. Then there exist an open neighborhood $B_P$ of $P$ and a $C^1$ embedded disc $W_{P,\mathrm{loc}}^{cs}\subset B_P$ of codimension at least one such that
\[
T_\eta(W_{P,\mathrm{loc}}^{cs})\cap B_P \subset W_{P,\mathrm{loc}}^{cs},
\qquad
\bigcap_{m\ge0} T_\eta^{-m}(B_P)\subset W_{P,\mathrm{loc}}^{cs}.
\]
Moreover, $W_{P,\mathrm{loc}}^{cs}$ is chart-null and hence Haar-null.
\end{lemma}

\begin{proof}
By Proposition~\ref{prop:strict_saddle}, $P$ is a strict saddle, so $\operatorname{Hess} f(P)$ has an eigenvalue $\lambda<0$. By Lemma~\ref{lem:discrete_linearization}, $DT_\eta(P)[H]=(1-\eta\lambda)H$ along the corresponding eigendirection, and $1-\eta\lambda>1$. Hence $DT_\eta(P)$ has at least one eigenvalue of modulus $>1$.

By Lemma~\ref{lem:local_diffeo}, $T_\eta$ is a $C^\infty$ local diffeomorphism. In a local chart centered at $P$, the conjugated map is a $C^\infty$ local diffeomorphism fixing the origin with at least one eigenvalue of modulus $>1$. Shub's center-stable manifold theorem~\citep[Theorem~III.7]{shub1987global} yields a local center-stable embedded disc tangent to the generalized eigenspaces of eigenvalues of modulus at most $1$, with the stated trapping property. Because the unstable space is nontrivial, $W_{P,\mathrm{loc}}^{cs}$ has codimension at least one. A $C^1$ embedded disc of codimension at least one has chart measure zero by Fubini's theorem. Lemma~\ref{lem:haar_null} upgrades this to Haar measure zero.
\end{proof}

\begin{corollary}[Discrete Global Convergence]
\label{cor:discrete_global}
Assume $C_e$ has distinct eigenvalues. For the exact Cayley iteration
\[
M_{k+1}=M_k\,\mathrm{Cay}(\eta\Omega_k),
\qquad
0<\eta<1/L_C,
\]
the following hold:
\begin{enumerate}[label=(\roman*)]
\item for every $M_0\in SO(n)$, the sequence $\{M_k\}$ converges to a single critical point of $f$;
\item if $M_0$ is Haar-random on $SO(n)$, then the limit point belongs to the global minimum set $\{M:f(M)=0\}$ with probability one.
\end{enumerate}
\end{corollary}

\begin{proof}
For (i), compactness of $SO(n)$ gives an accumulation point $P$ of the orbit. Lemma~\ref{lem:single_limit} shows $M_k\to P$ and $\Omega(P)=0$, so $P$ is critical.

For (ii), define $\mathcal S_{\mathrm{bad}}:=\{M\in SO(n): \Omega(M)=0,\ f(M)>0\}$. For every $P\in\mathcal S_{\mathrm{bad}}$, Lemma~\ref{lem:local_cs} provides a neighborhood $B_P$ and a chart-null set $W_{P,\mathrm{loc}}^{cs}\subset B_P$. Since $SO(n)$ is second countable, choose a countable subcover $\{B_j\}_{j\ge1}$ of $\mathcal S_{\mathrm{bad}}$ with associated sets $W_j^{cs}$.

If an orbit converges to $P\in\mathcal S_{\mathrm{bad}}$, then for some $j$ with $P\in B_j$ there exists $N$ such that $M_{N+m}\in B_j$ for all $m\ge0$. Hence $M_N\in \bigcap_{m\ge0} T_\eta^{-m}(B_j)\subset W_j^{cs}$, so $M_0\in T_\eta^{-N}(W_j^{cs})$. Therefore the basin of $\mathcal S_{\mathrm{bad}}$ is contained in $\bigcup_{j\ge1}\bigcup_{N\ge0} T_\eta^{-N}(W_j^{cs})$. Each $W_j^{cs}$ is chart-null, and Lemma~\ref{lem:local_diffeo} shows that every preimage $T_\eta^{-N}(W_j^{cs})$ is chart-null. Lemma~\ref{lem:haar_null} implies the whole union has Haar measure zero.

Thus Haar-almost every initialization avoids the bad basin. By (i), the orbit converges to a critical point, and Proposition~\ref{prop:strict_saddle} ensures it must be a global minimum.
\end{proof}

\begin{remark}[Step-Size Regime]
\label{rem:cayley_failure}
For a real skew-symmetric $\Omega$, the matrix $I-\frac{\eta}{2}\Omega$ is always invertible, so the exact Cayley map has no real singularity. The restriction $\eta<1/L_C$ is the regime in which monotone Lyapunov descent and the local-diffeomorphism property are proved. Outside this regime, $1-\frac{\eta L_C}{2}$ may become non-positive, and the displayed Lyapunov decrease is no longer guaranteed.
\end{remark}

\subsection{Explicit Negative Curvature Constants}
\label{app:escape_rate}

The following quantitative estimates are used in the saddle-escape analysis of Appendix~\ref{app:statistical}.

\begin{lemma}[Explicit Negative Curvature Constant]
\label{lem:explicit_gamma}
Let $f_{\mathrm{enter}} := (g - \underline{\delta})^2 / 8$. At any saddle-type critical point $M$ with $f(M) \ge f_{\mathrm{enter}}$:
\begin{equation}
\lambda_{\min}(\operatorname{Hess}\, f(M)) \le -\gamma_{\mathrm{enter}}, \qquad \gamma_{\mathrm{enter}} := \frac{(g - \underline{\delta})^2}{n(n-1)}.
\end{equation}
\end{lemma}

\begin{proof}
At a critical point with $f(M) \ge f_{\mathrm{enter}}$, using $f(M) = \sum_{i<j} A_{ij}^2$:
\[
b_{\max}^2 := \max_{i < j} A_{ij}^2 \ge \frac{2f(M)}{n(n-1)} \ge \frac{(g - \underline{\delta})^2}{4n(n-1)}.
\]
By Lemma~\ref{lem:negative_curvature}, $\lambda_{\min} \le -4b_{\max}^2 \le -\gamma_{\mathrm{enter}}$.
\end{proof}

\begin{lemma}[Directional Second Derivative along Givens]
\label{lem:directional_hessian}
For any $M \in SO(n)$, $A := M^\top C_e M$, and Givens direction $\Xi = E_{ij} - E_{ji}$:
\begin{equation}
\frac{d^2}{dt^2}\bigg|_{t=0} f(M e^{t\Xi}) = 2\left( (A_{jj} - A_{ii})^2 - 4A_{ij}^2 \right).
\end{equation}
At degenerate blocks ($A_{ii} = A_{jj}$), this reduces to $-8A_{ij}^2$.
\end{lemma}

\begin{proof}
Standard Givens conjugation gives $A_{ij}(t) = A_{ij}\cos(2t) + \frac{1}{2}(A_{jj} - A_{ii})\sin(2t)$. Direct differentiation of $A_{ij}(t)^2$ at $t=0$ yields the result.
\end{proof}

\section{Discrete Algorithm and Stability Analysis}
\label{app:discrete}

This appendix provides the discrete-time stability analysis for the Cayley iteration, including the explicit pullback descent lemma, input-to-state stability, domain invariance, convergence rates, and the local quadratic-growth error bound. Throughout, we use the unified notation from Appendix~\ref{app:notation}.

\subsection{Cayley Path Estimates}
\label{app:cayley_path}

The following estimates along the one-parameter family $Q(t):=\mathrm{Cay}(t\Xi)$ are used in the pullback descent analysis.

\begin{lemma}[Cayley Path Estimates]
\label{lem:cayley_path}
Let $\Xi\in\mathfrak{so}(n)$ and define $Q(t):=\mathrm{Cay}(t\Xi)=(I-\frac{t}{2}\Xi)^{-1}(I+\frac{t}{2}\Xi)$. Then:
\begin{enumerate}[label=(\roman*)]
\item $\|Q(t)-I\|_F\le t\|\Xi\|_F$;
\item $Q'(t)=(I-\frac{t}{2}\Xi)^{-1}\Xi(I-\frac{t}{2}\Xi)^{-1}$, and $\|Q'(t)\|_F\le \|\Xi\|_F$;
\item $\|Q'(t)-\Xi\|_F\le t\|\Xi\|_F^2$.
\end{enumerate}
\end{lemma}

\begin{proof}
Since $\Xi$ is skew-symmetric, $\|(I-\frac{t}{2}\Xi)^{-1}\|_2\le 1$ (Lemma~\ref{lem:cayley_diff}).

\textbf{(i)} By Lemma~\ref{lem:cayley_diff} applied to $X=t\Xi$, $\|Q(t)-I\|_F=\|\mathrm{Cay}(t\Xi)-I\|_F\le \|t\Xi\|_F=t\|\Xi\|_F$.

\textbf{(ii)} Write $A(t):=I-\frac{t}{2}\Xi$. Since $A(t)$ is a polynomial in $\Xi$, it commutes with $\Xi$. Differentiating $Q(t)=A(t)^{-1}(I+\frac{t}{2}\Xi)$ and using commutativity yields $Q'(t)=A(t)^{-1}\Xi A(t)^{-1}$. Hence $\|Q'(t)\|_F\le \|A(t)^{-1}\|_2^2\|\Xi\|_F\le \|\Xi\|_F$.

\textbf{(iii)} By commutativity, $Q'(t)-\Xi=[A(t)^{-2}-I]\Xi$. Now $A(t)^{-1}-I=A(t)^{-1}(I-A(t))=A(t)^{-1}\frac{t}{2}\Xi$, so $\|A(t)^{-1}-I\|_2\le \frac{t}{2}\|\Xi\|_F$. Using $A^{-2}-I=(A^{-1}-I)A^{-1}+(A^{-1}-I)$ gives $\|A(t)^{-2}-I\|_2\le 2\|A(t)^{-1}-I\|_2\le t\|\Xi\|_F$. Therefore $\|Q'(t)-\Xi\|_F\le t\|\Xi\|_F^2$.
\end{proof}

\subsection{Explicit Lyapunov Descent Lemma}
\label{app:discrete_descent}

\begin{lemma}[Discrete Lyapunov Descent]
\label{lem:discrete_descent}
Consider the update $M_{k+1} = M_k \cdot \mathrm{Cay}(\eta \Omega_k)$. Define the Cayley pullback smoothness constant
\begin{equation}
L_C := c_n\|C_e\|_2^2, \qquad c_n := 4\sqrt{n}+8.
\label{eq:LC_def}
\end{equation}
Then for $\eta < 1/L_C$:
\begin{equation}
f(M_{k+1}) \le f(M_k) - \eta\left(1 - \frac{\eta L_C}{2}\right)\|\Omega_k\|_F^2.
\label{eq:discrete_descent}
\end{equation}
\end{lemma}

\begin{proof}
Fix $M\in SO(n)$ and $\Xi\in\mathfrak{so}(n)$. Define $Q(t):=\mathrm{Cay}(t\Xi)$, $M(t):=MQ(t)$, and $h(t):=f(M(t))$. Then $h(0)=f(M)$ and $h(1)=f(M\,\mathrm{Cay}(\Xi))$.

At $t=0$, Lemma~\ref{lem:cayley_path}(ii) gives $Q'(0)=\Xi$, so
\[
h'(0)=\langle \operatorname{grad} f(M),\,M\Xi\rangle_F.
\]
For $t\in[0,1]$, estimate $|h'(t)-h'(0)|$ by the triangle inequality:
\begin{align}
|h'(t)-h'(0)|
&\le
\|\operatorname{grad} f(M(t))-\operatorname{grad} f(M)\|_F\,\|Q'(t)\|_F \notag\\
&\quad+
\|\operatorname{grad} f(M)\|_F\,\|Q'(t)-\Xi\|_F.
\label{eq:hprime_diff}
\end{align}
By Theorem~\ref{thm:lipschitz_constant_G} and Lemma~\ref{lem:cayley_path}(i),
\[
\|\operatorname{grad} f(M(t))-\operatorname{grad} f(M)\|_F
\le
(2\sqrt{n}+8)\|C_e\|_2^2\cdot t\|\Xi\|_F.
\]
By Lemma~\ref{lem:cayley_path}(ii), $\|Q'(t)\|_F\le \|\Xi\|_F$. The global bound $\|\operatorname{grad} f(M)\|_F=\|\Omega(M)\|_F\le 2\sqrt{n}\,\|C_e\|_2^2$ follows from the commutator and Frobenius norm estimates in Appendix~\ref{app:technical}. By Lemma~\ref{lem:cayley_path}(iii), $\|Q'(t)-\Xi\|_F\le t\|\Xi\|_F^2$.

Substituting into \eqref{eq:hprime_diff}:
\[
|h'(t)-h'(0)|
\le
\bigl((2\sqrt{n}+8)+2\sqrt{n}\bigr)\|C_e\|_2^2\,t\|\Xi\|_F^2
=
L_C\,t\|\Xi\|_F^2.
\]
Integrating $h'(t)\le h'(0)+L_C t\|\Xi\|_F^2$ from $0$ to $1$:
\[
f(M\,\mathrm{Cay}(\Xi))
\le
f(M)+\langle \operatorname{grad} f(M),M\Xi\rangle_F + \frac{L_C}{2}\|\Xi\|_F^2.
\]
Setting $\Xi=\eta\Omega(M)$ and using $\operatorname{grad} f(M)=-M\Omega(M)$ gives \eqref{eq:discrete_descent}.
\end{proof}

\begin{corollary}[Step Size Condition]
\label{cor:step_size}
Monotone decrease of $f$ along the exact Cayley step holds whenever $0<\eta<1/L_C$, and the maximal certified single-step decrease is achieved at $\eta=1/L_C$.
\end{corollary}

\subsection{Discrete ISS Theorem}
\label{app:thm3}

\begin{theorem}[Discrete ISS]
\label{thm:discrete_iss}
\label{thm:discrete_pathwise}
With step size $\eta < 1/L_C$ and input bound $\sup_k \|U_k\|_F \le \bar{U}$, assuming $M_k\in\mathcal{N}_{\underline{\delta}}$:
\begin{equation}
y_{k+1} \le (1 - c_1 \underline{\delta}^2 \eta) y_k + c_2 \|C_e\|_2 \eta\, \bar{U}, \quad
\limsup_{k \to \infty} \sqrt{f(M_k)} \le \frac{c_2 \|C_e\|_2}{c_1 \underline{\delta}^2} \bar{U},
\end{equation}
where $c_1=\frac12$ and $c_2=\sqrt{2}$. The full pathwise analysis appears in Section~\ref{app:discrete_iss_full}.
\end{theorem}

\subsection{Pathwise ISS Analysis}
\label{app:discrete_iss_full}

\subsubsection{Problem Setup}

Consider the perturbed discrete update:
\[
M_{k+1} = M_k \cdot \mathrm{Cay}(\eta \widetilde{\Omega}_k), \quad \widetilde{\Omega}_k := \Omega_k + U_k,
\]
where $\Omega_k = [A_k, D_k]$ is the ideal generator and $U_k \in \mathfrak{so}(n)$ is the tangent perturbation from observation error.

\textbf{Assumptions:}
\begin{enumerate}
\item \textbf{Local spectral separation:} $M_k\in\mathcal{N}_{\underline{\delta}}$;
\item \textbf{Bounded input:} $\sup_k \|U_k\|_F \le \bar{U}$ (verifiable via Lemma~\ref{lem:input_bound});
\item \textbf{Step size:} $\eta < 1/L_C$ with $L_C = (4\sqrt{n}+8)\|C_e\|_2^2$.
\end{enumerate}

\subsubsection{ISS Recursion Inequality}

\begin{lemma}[Discrete ISS Recursion]
\label{lem:discrete_iss_recursion}
Under the above assumptions, let $y_k := \sqrt{f(M_k)}$. Then
\begin{equation}
\boxed{
y_{k+1} \le \left(1 - \frac12 \underline{\delta}^2 \eta\right) y_k + \sqrt{2}\,\|C_e\|_2 \eta\, \bar{U}.
}
\label{eq:discrete_iss_recursion_D}
\end{equation}
\end{lemma}

\begin{proof}
Let $\widehat{M}_{k+1} := M_k \cdot \mathrm{Cay}(\eta \Omega_k)$ be the clean Cayley step and $M_{k+1} := M_k \cdot \mathrm{Cay}(\eta(\Omega_k + U_k))$ the perturbed step.

\textbf{Step 1 (Clean contraction).}
By Lemma~\ref{lem:discrete_descent}, $f(\widehat M_{k+1}) \le f(M_k) - \frac{\eta}{2}\|\Omega_k\|_F^2$ (using $1-\eta L_C/2 \ge 1/2$). By Lemma~\ref{lem:local_dissipation}, $\|\Omega_k\|_F^2 \ge 2\underline\delta^2 y_k^2$. Therefore $f(\widehat M_{k+1}) \le (1-\underline\delta^2 \eta) y_k^2$. Taking square roots and using $\sqrt{1-a}\le 1-a/2$:
\begin{equation}
\widehat y_{k+1} := \sqrt{f(\widehat M_{k+1})}
\le
\left(1-\frac12 \underline\delta^2 \eta\right) y_k.
\label{eq:clean_contraction_D}
\end{equation}

\textbf{Step 2 (Lipschitz continuity of $y(M)$).}
For any $M,N\in SO(n)$,
\begin{equation}
|y(M)-y(N)|
\le \frac{1}{\sqrt2}\|\operatorname{off}(A(M)-A(N))\|_F
\le \sqrt2\,\|C_e\|_2 \|M-N\|_F,
\label{eq:y_lipschitz_D}
\end{equation}
using $\|A(M)-A(N)\|_F\le 2\|C_e\|_2\|M-N\|_F$ from Appendix~\ref{app:technical}.

\textbf{Step 3 (Cayley step difference).}
For skew-symmetric $X,Y$, the algebraic identity
\[
\mathrm{Cay}(X)-\mathrm{Cay}(Y)
=
\left(I-\frac{X}{2}\right)^{-1}(X-Y)\left(I-\frac{Y}{2}\right)^{-1}
\]
holds. Since $\|(I-X/2)^{-1}\|_2\le 1$ and $\|(I-Y/2)^{-1}\|_2\le 1$ for skew-symmetric arguments,
\begin{equation}
\|M_{k+1}-\widehat M_{k+1}\|_F
=
\|\mathrm{Cay}(\eta(\Omega_k+U_k))-\mathrm{Cay}(\eta\Omega_k)\|_F
\le
\eta \bar U.
\label{eq:step_difference_D}
\end{equation}

\textbf{Step 4 (Combine).}
By \eqref{eq:y_lipschitz_D}, \eqref{eq:clean_contraction_D}, and \eqref{eq:step_difference_D}:
\[
y_{k+1}
\le
\widehat y_{k+1} + \sqrt2\,\|C_e\|_2 \|M_{k+1}-\widehat M_{k+1}\|_F
\le
\left(1-\frac12 \underline\delta^2 \eta\right) y_k + \sqrt2\,\|C_e\|_2 \eta\, \bar U. \qedhere
\]
\end{proof}

\subsubsection{Pathwise Noise Ball}

\begin{theorem}[Discrete Pathwise Noise Ball]
\label{thm:discrete_noise_ball}
Under the assumptions of the ISS recursion:
\begin{enumerate}[label=\textbf{(\roman*)}]
\item \textbf{Exponential convergence to a noise ball:}
\[
y_k \le \left(1 - \frac12\underline{\delta}^2 \eta\right)^k y_0 + r_f^{\mathrm{disc}}(\bar U),
\]
where $r_f^{\mathrm{disc}}(\bar U) := \frac{2\sqrt2\,\|C_e\|_2}{\underline{\delta}^2}\bar U$.

\item \textbf{Ultimate noise ball:}
$\limsup_{k \to \infty} \sqrt{f(M_k)} \le r_f^{\mathrm{disc}}(\bar U)$.
\end{enumerate}
\end{theorem}

\begin{proof}
Let $q := 1 - \frac12\underline\delta^2 \eta \in (0,1)$. Iterating the recursion of Lemma~\ref{lem:discrete_iss_recursion}:
\[
y_k \le q^k y_0 + \sqrt2\,\|C_e\|_2 \eta\, \bar U \sum_{j=0}^{k-1} q^j
\le q^k y_0 + \frac{2\sqrt2\,\|C_e\|_2}{\underline\delta^2}\bar U.
\]
Taking $\limsup_{k\to\infty}$ yields the stated ultimate bound.
\end{proof}

\begin{corollary}[Explicit Noise Ball Radius]
\label{cor:explicit_radius}
If $\sup_k \|E_{e,k}\|_F \le \varepsilon$, then by Lemma~\ref{lem:input_bound},
$\bar{U} \le 4\|C_e\|_2 \varepsilon + 2\varepsilon^2$,
giving
\[
r_f^{\mathrm{disc}}(\varepsilon)
\le
\frac{2\sqrt2\,\|C_e\|_2}{\underline{\delta}^2}
\left(4\|C_e\|_2 \varepsilon + 2\varepsilon^2\right).
\]
\end{corollary}

\begin{remark}[Consistency with Continuous Case]
Both $r_f^{\mathrm{disc}}$ and the continuous-time ISS radius from Appendix~\ref{app:geometric} are proportional to $\|C_e\|_2\,\bar U/\underline{\delta}^2$ and contain no $\sigma^2 I$ contribution.
\end{remark}

\subsection{Domain Radius and Non-Escape Condition}
\label{app:domain_invariance}

\begin{lemma}[Domain Radius]
\label{app:domain_radius}
For every $M\in SO(n)$,
\begin{equation}
\delta(M)\ge g-2\sqrt{2f(M)}.
\label{app1:domain_radius}
\end{equation}
Consequently, $\sqrt{f(M)} < \frac{g-\underline\delta}{2\sqrt2}$ implies $M\in\mathcal N_{\underline\delta}$.
\end{lemma}

\begin{proof}
Let $A:=M^\top C_e M$, $d_i:=A_{ii}$, and $D:=\operatorname{diag}(A)$. The eigenvalues of $A$ are $\lambda_1,\dots,\lambda_n$ (those of $C_e$), while $D$ has eigenvalues $d_1,\dots,d_n$. By the Wielandt--Hoffman inequality~\citep{bhatia1997matrix},
\[
\sum_{i=1}^n (d_i-\lambda_{\pi(i)})^2 \le \|A-D\|_F^2 = \|\operatorname{off}(A)\|_F^2 = 2f(M)
\]
for some permutation $\pi$. Hence $|d_i-\lambda_{\pi(i)}|\le \sqrt{2f(M)}$ for every $i$. For $i\neq j$:
\[
|d_i-d_j|
\ge
|\lambda_{\pi(i)}-\lambda_{\pi(j)}| - |d_i-\lambda_{\pi(i)}| - |d_j-\lambda_{\pi(j)}|
\ge
g-2\sqrt{2f(M)}.
\]
Taking the minimum over $i\neq j$ proves \eqref{eq:domain_radius}.
\end{proof}

\begin{theorem}[Non-Escape Condition]
\label{app:non_escape}
Fix $0<\underline\delta<g$ and define $y_{\mathrm{thr}}:=\frac{g-\underline\delta}{2\sqrt2}$. Consider the perturbed iteration with $0<\eta<1/L_C$ and $\sup_k\|U_k\|_F\le \bar U$. If
\begin{equation}
\sqrt{f(M_0)} \le y_{\mathrm{thr}}
\qquad\text{and}\qquad
r_f^{\mathrm{disc}}(\bar U) \le y_{\mathrm{thr}},
\label{eq:non_escape_cond}
\end{equation}
then $M_k\in \mathcal N_{\underline\delta}$ for every $k\ge 0$.
\end{theorem}

\begin{proof}
We prove $y_k\le y_{\mathrm{thr}}$ for all $k$ by induction. The base case $y_0\le y_{\mathrm{thr}}$ is assumed. By Lemma~\ref{lem:domain_radius}, $\delta(M_0)\ge g-2\sqrt2\,y_0\ge \underline\delta$, so $M_0\in\mathcal N_{\underline\delta}$.

For the inductive step, assume $y_k\le y_{\mathrm{thr}}$. Then $M_k\in\mathcal N_{\underline\delta}$, so the ISS recursion (Lemma~\ref{lem:discrete_iss_recursion}) applies:
\[
y_{k+1}
\le
q\,y_k + (1-q)\,r_f^{\mathrm{disc}}(\bar U)
\le
q\,y_{\mathrm{thr}} + (1-q)\,y_{\mathrm{thr}}
=
y_{\mathrm{thr}},
\]
where $q:=1-\frac12\underline\delta^2\eta\in(0,1)$ and we used $\sqrt2\,\|C_e\|_2\eta\bar U = (1-q)\,r_f^{\mathrm{disc}}(\bar U)$.
\end{proof}

\subsection{Convergence Rate Analysis}
\label{app:convergence_rate}

\subsubsection{Stochastic Assumption}

\begin{assumption}[Centered Finite Variance]
\label{ass:centered_variance_D}
The tangent input $U_k \in \mathfrak{so}(n)$ is conditionally centered with bounded conditional second moment:
\[
\mathbb{E}[U_k \mid \mathcal{F}_k] = 0,
\qquad
\mathbb{E}[\|U_k\|_F^2 \mid \mathcal{F}_k] \le \sigma_u^2.
\]
\end{assumption}

\subsubsection{$O(1/k)$ Rate}

\begin{corollary}[$O(1/k)$ Convergence Rate]
\label{cor:convergence_rate}
Suppose $M_0 \in \mathcal{N}_{\underline{\delta}}$ and the trajectory remains in this domain almost surely. Let
\[
\eta_k = \frac{c}{k+k_0},
\qquad
c > \frac{1}{\underline{\delta}^2},
\qquad
k_0 \ge cL_C,
\]
so that $\eta_k \le 1/L_C$ for every $k$. Under Assumption~\ref{ass:centered_variance_D},
\begin{equation}
\boxed{
\mathbb{E}[f(M_k)] \le \frac{C}{k + k_0},
\qquad
C := \max\left\{k_0 f(M_0),\; \frac{c^2 L_C \sigma_u^2}{2(c\underline{\delta}^2 - 1)}\right\}.
}
\end{equation}
\end{corollary}

\begin{proof}
Apply the pullback descent (Lemma~\ref{lem:discrete_descent}) to the perturbed step $\eta_k(\Omega_k + U_k)$:
\[
f_{k+1}
\le
f_k
-
\eta_k \langle \Omega_k, \Omega_k + U_k \rangle
+
\frac{L_C\eta_k^2}{2}\|\Omega_k + U_k\|_F^2.
\]
Taking conditional expectation and using Assumption~\ref{ass:centered_variance_D}:
\[
\mathbb{E}[f_{k+1}\mid \mathcal F_k]
\le
f_k
-
\eta_k \|\Omega_k\|_F^2
+
\frac{L_C\eta_k^2}{2}\left(\|\Omega_k\|_F^2 + \sigma_u^2\right).
\]
Since $\eta_k \le 1/L_C$ gives $1-L_C\eta_k/2 \ge 1/2$, and by the PL inequality (Lemma~\ref{lem:local_dissipation}) $\|\Omega_k\|_F^2 \ge 2\underline\delta^2 f_k$:
\begin{equation}
\mathbb{E}[f_{k+1}\mid \mathcal F_k]
\le
(1-\underline\delta^2\eta_k)f_k + \frac{L_C\eta_k^2}{2}\sigma_u^2.
\label{eq:pl_recursion_D}
\end{equation}
Setting $m_k := \mathbb E[f_k]$ and $s:=k+k_0$, this becomes
\[
m_{k+1}
\le
\left(1-\frac{c\underline\delta^2}{s}\right)m_k
+
\frac{L_C c^2 \sigma_u^2}{2s^2}.
\]
We prove $m_k \le C/s$ by induction. The base case holds by the definition of $C$. For the inductive step, $m_k \le C/s$ gives
\[
m_{k+1}
\le
\frac{C}{s} - \frac{c\underline\delta^2 C}{s^2} + \frac{L_C c^2 \sigma_u^2}{2s^2}.
\]
Since $\frac{C}{s+1} \ge \frac{C}{s} - \frac{C}{s^2}$, it suffices that $(c\underline\delta^2 - 1)C \ge \frac{L_C c^2 \sigma_u^2}{2}$, which is ensured by the definition of $C$.
\end{proof}

\subsubsection{Effective Noise Constant and Sample Complexity}

By Lemma~\ref{lem:input_bound},
\[
\|U_k\|_F \le 4\|C_e\|_2 \|E_{e,k}\|_F + 2\|E_{e,k}\|_F^2,
\]
so $\sigma_u^2$ depends only on trace-free anisotropic moments of the noise, never on a scalar shift $\sigma_k^2 I$.

\begin{remark}[Bounded-noise specialization]
If $\|E_{e,k}\|_F \le \varepsilon_0$ almost surely, then
$\sigma_u^2 \le (4\|C_e\|_2 + 2\varepsilon_0)^2 \sup_k \mathbb E[\|E_{e,k}\|_F^2]$.
\end{remark}

\begin{corollary}[Sample Complexity]
From the $O(1/k)$ rate, $k^\ast(\varepsilon) = O\!\left(\frac{\|C_e\|_2^2 \sigma_u^2}{\underline{\delta}^2\,\varepsilon}\right)$. All quantities are independent of any isotropic background level.
\end{corollary}

\subsection{Local Quadratic Growth and Distance Conversion}
\label{app:quadratic_growth}

This subsection converts the Lyapunov noise ball into a Frobenius-distance bound near the global minimum set $\mathcal M_{\min}:=\{M\in SO(n):f(M)=0\}$.

\begin{lemma}[Hessian at a Global Minimum]
\label{lem:hessian_global_min}
Let $M_\star\in\mathcal M_{\min}$, so that $A_\star:=M_\star^\top C_e M_\star=\operatorname{diag}(\lambda_{\pi(1)},\dots,\lambda_{\pi(n)})$ for some permutation $\pi$. Then for every $\Xi\in\mathfrak{so}(n)$:
\begin{equation}
\operatorname{Hess}\,f(M_\star)[M_\star\Xi,M_\star\Xi]
=
2\sum_{1\le i<j\le n}(\lambda_{\pi(i)}-\lambda_{\pi(j)})^2\,\Xi_{ij}^2.
\label{eq:hessian_min}
\end{equation}
In particular, $g^2\|\Xi\|_F^2 \le \operatorname{Hess}\,f(M_\star)[M_\star\Xi,M_\star\Xi] \le 4\|C_e\|_2^2\|\Xi\|_F^2$.
\end{lemma}

\begin{proof}
At $M_\star$, $D_\star=A_\star$ and $\Omega(M_\star)=0$. The linearization formula (Appendix~\ref{app:geometric}) gives $\delta A=[A_\star,\Xi]$ with $(\delta A)_{ij}=(\lambda_{\pi(i)}-\lambda_{\pi(j)})\Xi_{ij}$, and $\delta D=\operatorname{diag}(\delta A)=0$ since diagonal entries of $[A_\star,\Xi]$ vanish when $A_\star$ is diagonal. Hence $D\Omega(M_\star)[M_\star\Xi]=[[A_\star,\Xi],A_\star]$ with entries $-(\lambda_{\pi(i)}-\lambda_{\pi(j)})^2\Xi_{ij}$. By Lemma~\ref{lem:hessian_link}, $\operatorname{Hess}\,f(M_\star)[M_\star\Xi,M_\star\Xi]=-\langle D\Omega[M_\star\Xi],\Xi\rangle_F$, and since $\Xi$ is skew-symmetric ($\Xi_{ji}=-\Xi_{ij}$), the pairing yields \eqref{eq:hessian_min}. The bounds follow from $g\le|\lambda_{\pi(i)}-\lambda_{\pi(j)}|\le 2\|C_e\|_2$ and $\|\Xi\|_F^2=2\sum_{i<j}\Xi_{ij}^2$.
\end{proof}

\begin{theorem}[Local Quadratic Growth]
\label{thm:quadratic_growth}
There exist constants $\rho>0$ and $\varepsilon_0>0$ such that:
\begin{enumerate}[label=(\roman*)]
\item if $\operatorname{dist}_F(M,\mathcal M_{\min})<\rho$, then
\[
\frac{g^2}{4}\operatorname{dist}_F(M,\mathcal M_{\min})^2
\le
f(M)
\le
2\|C_e\|_2^2\,\operatorname{dist}_F(M,\mathcal M_{\min})^2;
\]
\item if $f(M)\le \varepsilon_0$, then
$\operatorname{dist}_F(M,\mathcal M_{\min}) \le \frac{2}{g}\sqrt{f(M)}$.
\end{enumerate}
\end{theorem}

\begin{proof}
\textbf{Upper bound.} For any $M_\star\in\mathcal M_{\min}$ and any $M\in SO(n)$, $f(M)=\frac12\|\operatorname{off}(A(M)-A_\star)\|_F^2\le \frac12\|A(M)-A_\star\|_F^2\le 2\|C_e\|_2^2\|M-M_\star\|_F^2$.

\textbf{Lower bound.} For each $M_\star\in\mathcal M_{\min}$, define the pullback $h_\star(\Xi):=f(M_\star e^\Xi)$. By Lemma~\ref{lem:hessian_global_min}, $\nabla^2 h_\star(0)$ satisfies $\langle \nabla^2 h_\star(0)\Xi,\Xi\rangle \ge g^2\|\Xi\|_F^2$. By continuity of the Hessian, there exists $\rho_\star>0$ such that $h_\star(\Xi)\ge \frac{g^2}{4}\|\Xi\|_F^2$ for $\|\Xi\|_F\le \rho_\star$. Since $\|e^\Xi-I\|_F\le \|\Xi\|_F$ for skew-symmetric $\Xi$, this gives $f(M)\ge \frac{g^2}{4}\|M-M_\star\|_F^2$ locally.

\textbf{Combining.} Since $\mathcal M_{\min}$ is finite, set $\rho:=\min\{\frac{\Delta}{3},\min_{M_\star}\rho_\star\}$ where $\Delta:=\min_{M_\star\neq N_\star}\|M_\star-N_\star\|_F>0$. The balls $B_\rho(M_\star)$ are disjoint, so the nearest minimizer is unique for $\operatorname{dist}_F(M,\mathcal M_{\min})<\rho$. This yields (i). For (ii), let $m_\rho:=\min\{f(M):\operatorname{dist}_F(M,\mathcal M_{\min})\ge \rho\}>0$ and set $\varepsilon_0:=\min\{m_\rho,g^2\rho^2/4\}$.
\end{proof}

\begin{corollary}[Lyapunov-to-Distance Conversion]
\label{cor:lyapunov_distance}
If $\limsup_{k\to\infty}\sqrt{f(M_k)}\le r$ with $r<\sqrt{\varepsilon_0}$, then $\limsup_{k\to\infty}\operatorname{dist}_F(M_k,\mathcal M_{\min})\le \frac{2}{g}\,r$.
\end{corollary}

\subsection{Matrix-Free Implementation via Neumann Series}
\label{app:neumann}

In the matrix-free setting, the Cayley inverse $(I - \frac{\eta}{2}\Omega)^{-1}$ is approximated by the truncated Neumann series $S_K(X) := \sum_{j=0}^K X^j$ with $X:=\frac{\eta}{2}\Omega$, yielding the approximate Cayley factor $\mathrm{Cay}^{(K)}(X) := S_K(X)(I+X)$.

\begin{proposition}[Truncation Error]
\label{prop:truncation_error}
Let $\rho:=\|X\|_2<1$. Then
\[
\|\mathrm{Cay}(\eta\Omega) - \mathrm{Cay}^{(K)}(X)\|_F
\le
\sqrt{n}\,\frac{1+\rho}{1-\rho}\,\rho^{K+1}.
\]
Under the step size condition $\eta < 1/L_C$, one has $\rho \le \sqrt{n}/c_n < 1$.
\end{proposition}

\begin{proof}
Using $(I-X)^{-1}-S_K(X)=X^{K+1}(I-X)^{-1}$:
$\mathrm{Cay}(\eta\Omega)-\mathrm{Cay}^{(K)}(X)=X^{K+1}(I-X)^{-1}(I+X)$.
The norm bound follows from $\|X\|_2^{K+1}\cdot\frac{1}{1-\rho}\cdot(1+\rho)$ and $\|{\cdot}\|_F\le \sqrt{n}\|{\cdot}\|_2$.
\end{proof}

\begin{remark}[$\sigma^2$-Invariance under Truncation]
\label{rem:sigma_invariance_truncation}
The truncated factor $\mathrm{Cay}^{(K)}(X)$ depends on the observation only through $\Omega = [A,\operatorname{diag}(A)]$. Since $\Omega$ is invariant under $A \mapsto A + \sigma^2 I$, every truncation order $K$ preserves pathwise $\sigma^2$-immunity.
\end{remark}

\begin{algorithm}[H]
\caption{Matrix-Free Cayley Update via Neumann Series}
\label{alg:neumann_cayley}
\begin{algorithmic}[1]
\REQUIRE Current iterate $M_k$, observation $C_k$, step size $\eta$, Neumann order $K$
\ENSURE Approximate next iterate $M_{k+1}$
\STATE $Y_k \gets C_k M_k$ \COMMENT{$n$ MVPs}
\STATE $A_k \gets M_k^\top Y_k$ \COMMENT{$O(n^3)$ dense arithmetic}
\STATE $D_k \gets \operatorname{diag}(A_k)$
\STATE $\Omega_k \gets A_k D_k - D_k A_k$
\STATE $X \gets \frac{\eta}{2}\Omega_k$
\STATE $P \gets I_n$, $S \gets I_n$
\FOR{$j = 1$ to $K$}
    \STATE $P \gets P X$
    \STATE $S \gets S + P$
\ENDFOR
\STATE $M_{k+1} \gets M_k\, S (I + X)$
\RETURN $M_{k+1}$
\end{algorithmic}
\end{algorithm}

\begin{center}
\begin{tabular}{lcccc}
\toprule
Method & MVPs & Dense arithmetic & Linear solve & Exact orthogonality \\
\midrule
Direct Cayley & $O(n)$ & $O(n^3)$ & Yes & Yes \\
Neumann-$K$ Cayley & $O(n)$ & $O(Kn^3)$ & No & No \\
QR Retraction & $O(n)$ & $O(n^3)$ & QR & Yes \\
Polar Retraction & $O(n)$ & $O(n^3)$ & SVD & Yes \\
\bottomrule
\end{tabular}
\end{center}

\begin{remark}[Complexity]
For dense full-basis implementation, all retractions require $O(n^3)$ arithmetic once $A_k = M_k^\top C_k M_k$ is formed. The Neumann approximation removes the internal linear solve but does not reduce the dense-arithmetic order. Exact orthogonality requires the exact Cayley map or periodic re-orthogonalization.
\end{remark}

\section{Statistical Analysis}
\label{app:statistical}

This appendix provides high-probability concentration bounds~\citep{vershynin2018high} for the Lyapunov value under stochastic perturbations, a pathwise noisy finite-time entry theorem, and the proof of the main statistical robustness result. Throughout, we use the unified notation from Appendix~\ref{app:notation}, and write
\[
\lambda := c_1\underline{\delta}^2 = \tfrac12\underline{\delta}^2,
\qquad
\gamma_u := c_2\|C_e\|_2 = \sqrt{2}\,\|C_e\|_2,
\qquad
q := 1-\lambda\eta \in (0,1).
\]

\subsection{Weight Norm Bounds}
\label{app:weight_norms}

Fix a constant step size $\eta>0$ satisfying the assumptions of Theorem~\ref{thm:discrete_iss}. If
\[
y_{k+1} \le q\,y_k + \gamma_u \eta\, u_k,
\qquad y_k:=\sqrt{f(M_k)},
\]
then repeated substitution gives
\begin{equation}
y_N \le q^N y_0 + \gamma_u \sum_{j=0}^{N-1}\alpha_{N,j}\,u_j,
\qquad
\alpha_{N,j} := \eta q^{N-1-j}.
\label{eq:discrete_convolution_E}
\end{equation}

\begin{lemma}[Weight Norm Bounds]
\label{lem:weight_norms_E}
Assume $\lambda\eta \le 1$. Then the discrete weights satisfy:
\begin{enumerate}
\item[(i)] $\max_{0\le j\le N-1}\alpha_{N,j} = \eta$.
\item[(ii)] $\sum_{j=0}^{N-1}\alpha_{N,j} = \frac{1-q^N}{\lambda} \le \frac{1}{\lambda}$.
\item[(iii)] $\sum_{j=0}^{N-1}\alpha_{N,j}^2 = \frac{\eta(1-q^{2N})}{\lambda(2-\lambda\eta)} \le \frac{\eta}{\lambda}$.
\end{enumerate}
\end{lemma}

\begin{proof}
The maximum is attained at $j=N-1$. The sum and squared-sum formulas are geometric series. Since $\lambda\eta\le 1$, one has $2-\lambda\eta \ge 1$, giving the final inequality in (iii).
\end{proof}

\subsection{Probabilistic Error Bounds}
\label{app:prob_bounds}

\subsubsection{Sub-Exponential Noise (Bernstein Regime)}

\begin{assumption}[IND-SE]
\label{ass:ind_se}
For each $k$, write $u_k = \mu_k + \xi_k$ with $\mu_k := \mathbb E[u_k]$, where $\{\xi_k\}$ are mutually independent, mean-zero, and sub-exponential with parameters $(\sigma_u,b_u)$:
$\mathbb E[e^{\theta \xi_k}] \le e^{\sigma_u^2\theta^2/2}$ for all $|\theta|\le 1/b_u$.
Let $\bar\mu := \sup_k \mu_k$.
\end{assumption}

\begin{theorem}[Pointwise-in-Time Exponential Concentration]
\label{thm:exp_concentration_E}
Under Assumption~\ref{ass:ind_se}, for every $N\ge 1$ and $\zeta\in(0,1)$,
\begin{equation}
\mathbb P\!\left(
y_N >
q^N y_0 + \frac{\gamma_u}{\lambda}\bar\mu
+
\gamma_u\left[
\sqrt{\frac{2\sigma_u^2\eta}{\lambda}\log\frac1\zeta}
+
2b_u\eta\log\frac1\zeta
\right]
\right)
\le
\zeta.
\label{eq:bernstein_pointwise_E}
\end{equation}
\end{theorem}

\begin{proof}
From \eqref{eq:discrete_convolution_E},
$y_N \le q^N y_0 + \gamma_u \sum_j\alpha_{N,j}\mu_j + \gamma_u S_N$ with $S_N := \sum_j\alpha_{N,j}\xi_j$.
By Lemma~\ref{lem:weight_norms_E}, $\sum_j\alpha_{N,j}\mu_j \le \bar\mu/\lambda$.
Bernstein's inequality~\citep{vershynin2018high} with $\sum_j \alpha_{N,j}^2 \le \eta/\lambda$ and $\max_j \alpha_{N,j}\le \eta$ gives the stated tail bound upon choosing
$t = \sqrt{(2\sigma_u^2\eta/\lambda)\log(1/\zeta)} + 2b_u\eta\log(1/\zeta)$.
\end{proof}

\subsubsection{Heavy-Tailed Noise (Chebyshev Regime)}

\begin{assumption}[UNC-L2]
\label{ass:unc_l2}
For each $k$, write $u_k=\mu_k+\xi_k$ with $\bar\mu:=\sup_k\mu_k$, and $\{\xi_k\}$ pairwise uncorrelated, mean-zero, variance-bounded: $\operatorname{Var}(\xi_k)\le \sigma_u^2$.
\end{assumption}

\begin{theorem}[Pointwise-in-Time $L^2$ Bound]
\label{thm:l2_pointwise_E}
Under Assumption~\ref{ass:unc_l2}, for every $N\ge 1$ and $\epsilon>0$,
\begin{equation}
\mathbb P\!\left(
y_N > q^N y_0 + \frac{\gamma_u}{\lambda}\bar\mu + \epsilon
\right)
\le
\frac{\gamma_u^2 \sigma_u^2 \eta}{\lambda \epsilon^2}.
\label{eq:chebyshev_pointwise_E}
\end{equation}
\end{theorem}

\begin{proof}
Pairwise uncorrelatedness gives $\operatorname{Var}(S_N) \le \sigma_u^2 \sum_j\alpha_{N,j}^2 \le \sigma_u^2\eta/\lambda$. Chebyshev's inequality yields \eqref{eq:chebyshev_pointwise_E}.
\end{proof}

\subsection{Saddle Escape Geometry}
\label{app:saddle_escape}

This subsection quantifies the negative curvature near saddle points, supporting the finite-time entry analysis below.

\begin{lemma}[Level Set Separation]
\label{lem:level_set_barrier}
Assume $C_e$ has distinct eigenvalues. For any $f_{\mathrm{enter}} > 0$, there exist constants $\varepsilon_\star, \gamma_\star > 0$ such that for all $M \in SO(n)$:
\[
f(M) \ge f_{\mathrm{enter}}
\quad\Longrightarrow\quad
\|\operatorname{grad} f(M)\|_F \ge \varepsilon_\star
\;\;\text{or}\;\;
\lambda_{\min}(\operatorname{Hess} f(M)) \le -\gamma_\star.
\]
\end{lemma}

\begin{proof}
Suppose not. Then there exists a sequence $M_k$ with $f(M_k)\ge f_{\mathrm{enter}}$, $\|\operatorname{grad}f(M_k)\|_F \to 0$, and $\lambda_{\min}(\operatorname{Hess}f(M_k)) \to 0^-$. By compactness of $SO(n)$, a subsequence converges to $\bar M$ with $f(\bar M)\ge f_{\mathrm{enter}}>0$, $\operatorname{grad}f(\bar M)=0$, and $\operatorname{Hess}f(\bar M)\succeq 0$. This contradicts Proposition~\ref{prop:strict_saddle}.
\end{proof}

\begin{lemma}[Negative Curvature Preservation Radius]
\label{lem:nc_radius}
Define
\[
b_0 := \frac{g-\underline{\delta}}{2\sqrt{n(n-1)}},
\qquad
r_{\mathrm{nc}} := \frac{b_0}{16\|C_e\|_2}.
\]
Let $\bar M$ be a saddle critical point with $f(\bar M)\ge f_{\mathrm{enter}}:=(g-\underline\delta)^2/8$, and let $(i,j)$ be a degenerate block with $|\bar A_{ij}| \ge b_0$. Then for every $M$ with $d_g(M,\bar M)\le r_{\mathrm{nc}}$, the Givens direction $\Xi = E_{ij}-E_{ji}$ satisfies
\[
\frac{d^2}{dt^2}\bigg|_{t=0} f(Me^{t\Xi})
\le
-6b_0^2.
\]
\end{lemma}

\begin{proof}
Let $M = \bar M e^{\Theta}$ with $\|\Theta\|_F \le r_{\mathrm{nc}}$. By the commutator bound (Lemma~\ref{lem:commutator_bound}),
\[
\|A(M)-\bar A\|_F \le 2\|C_e\|_2 d_g(M,\bar M) \le 2\|C_e\|_2 r_{\mathrm{nc}} = \frac{b_0}{8}.
\]
Hence $|b(M)| \ge \frac{7}{8}b_0$ and $|\Delta(M)| \le \frac{1}{4}b_0$. By Lemma~\ref{lem:directional_hessian},
\[
\frac{d^2}{dt^2}\bigg|_{t=0} f(Me^{t\Xi})
=
2(\Delta(M)^2 - 4b(M)^2)
\le
2\!\left(\frac{b_0^2}{16} - \frac{49b_0^2}{16}\right)
=
-6b_0^2. \qedhere
\]
\end{proof}

\begin{proposition}[Conditional Escape Steps from a Saddle Neighborhood]
\label{prop:escape_time_stat}
Let
\[
\gamma_{\mathrm{loc}} := \frac{15}{32}\cdot \frac{(g-\underline{\delta})^2}{n(n-1)},
\qquad
r_* := \min\left\{r_{\mathrm{nc}},\; \frac{\gamma_{\mathrm{loc}}}{16L_C}\right\}.
\]
If there exists a local scalar coordinate $\alpha$ on $B(\bar M,r_*)$ satisfying $|\alpha(M)| \le d_g(M,\bar M)$ and
$|\alpha_{k+1}| \ge (1+\frac{\eta\gamma_{\mathrm{loc}}}{2})|\alpha_k|$
while the iterates remain in $B(\bar M,r_*)$, and if $|\alpha_0|\ge w>0$, then
\begin{equation}
T_{\mathrm{esc}}
\le
\frac{2}{\eta\gamma_{\mathrm{loc}}}\log\frac{r_*}{w}.
\label{eq:escape_time_stat_cond_E}
\end{equation}
\end{proposition}

\begin{proof}
$|\alpha_k| \ge (1+\frac{\eta\gamma_{\mathrm{loc}}}{2})^k w$. Escape occurs once $|\alpha_k|\ge r_*$, requiring $k \le \frac{2}{\eta\gamma_{\mathrm{loc}}}\log(r_*/w)$ by $\log(1+x)\ge x/2$.
\end{proof}

\subsection{Noisy Finite-Time Domain Entry}
\label{app:finite_time_entry}

This subsection establishes that, under sufficiently small noise, the perturbed iteration enters the spectrally separated domain in finite time.

\begin{lemma}[Global Lipschitz of the Clean Cayley Map]
\label{lem:clean_lipschitz}
For any $M,N\in SO(n)$,
\[
\|T_\eta(M)-T_\eta(N)\|_F
\le
\Lambda_\eta \|M-N\|_F,
\qquad
\Lambda_\eta:=1+8\eta\|C_e\|_2^2,
\]
where $T_\eta(M):=M\,\mathrm{Cay}(\eta\Omega(M))$.
\end{lemma}

\begin{proof}
By Lemma~\ref{lem:domega_bound}, $\|\Omega(M)-\Omega(N)\|_F \le 8\|C_e\|_2^2\|M-N\|_F$. By the Cayley difference identity (Appendix~\ref{app:discrete_descent}), $\|\mathrm{Cay}(X)-\mathrm{Cay}(Y)\|_F\le \|X-Y\|_F$ for skew-symmetric $X,Y$. Therefore
\begin{align}
\|T_\eta(M)-T_\eta(N)\|_F
&\le
\|M-N\|_F + \|\mathrm{Cay}(\eta\Omega(M))-\mathrm{Cay}(\eta\Omega(N))\|_F \notag\\
&\le
\|M-N\|_F + \eta\|\Omega(M)-\Omega(N)\|_F \notag\\
&\le
(1+8\eta\|C_e\|_2^2)\|M-N\|_F. \qedhere
\end{align}
\end{proof}

\begin{lemma}[Lyapunov Global Lipschitz]
\label{lem:f_lipschitz}
For any $M,N\in SO(n)$,
\[
|f(M)-f(N)| \le L_f\|M-N\|_F,
\qquad
L_f:=2\sqrt{n}\,\|C_e\|_2^2.
\]
\end{lemma}

\begin{proof}
Let $X=\operatorname{off}(A(M))$, $Y=\operatorname{off}(A(N))$. Then
$|f(M)-f(N)| = \frac12|\|X\|_F^2-\|Y\|_F^2| \le \frac12(\|X\|_F+\|Y\|_F)\|X-Y\|_F$.
Since $\|X\|_F \le \|A(M)\|_F \le \sqrt{n}\,\|C_e\|_2$ and $\|X-Y\|_F \le 2\|C_e\|_2\|M-N\|_F$, the result follows.
\end{proof}

\begin{theorem}[Pathwise Noisy Finite-Time Entry]
\label{thm:finite_time_entry}
Assume $C_e$ has distinct eigenvalues and $0<\eta<1/L_C$. Fix $0<\underline\delta<g$, and let
\[
f_{\mathrm{thr}} := \frac{(g-\underline\delta)^2}{8}.
\]
Let $M_0$ not belong to the bad basin of the clean map $T_\eta$ (Corollary~\ref{cor:discrete_global}). Then the clean orbit $\{\widehat M_k\}$ converges to a global minimum, so there exists a finite
\[
T_{\mathrm{hit}} = T_{\mathrm{hit}}(M_0,\underline\delta) < \infty
\]
such that $f(\widehat M_{T_{\mathrm{hit}}}) \le \frac12 f_{\mathrm{thr}}$.

Consider the perturbed orbit $M_{k+1}=M_k\,\mathrm{Cay}(\eta(\Omega_k+U_k))$ with $M_0=\widehat M_0$ and $\sup_k\|U_k\|_F\le \bar U$. Define the tracking bound
\[
B_k(\eta) := \eta\sum_{r=0}^{k-1}\Lambda_\eta^r
\]
and the noise thresholds
\[
\bar U_{\mathrm{inv}} := \frac{\underline\delta^2(g-\underline\delta)}{8\|C_e\|_2},
\qquad
\bar U_{\mathrm{trans}} := \frac{f_{\mathrm{thr}}}{2L_f B_{T_{\mathrm{hit}}}(\eta)}.
\]
If $\bar U \le \min\{\bar U_{\mathrm{inv}},\, \bar U_{\mathrm{trans}}\}$, then
\begin{equation}
M_k \in \mathcal N_{\underline\delta}
\qquad\text{for all }k\ge T_{\mathrm{hit}}.
\label{eq:noisy_entry}
\end{equation}
\end{theorem}

\begin{proof}
\textbf{Step 1 (Tracking bound).}
By Lemma~\ref{lem:clean_lipschitz} and the single-step noise bound $\|T_\eta^U(M)-T_\eta(M)\|_F \le \eta\|U\|_F$ (from the Cayley difference identity), the error $e_k:=\|M_k-\widehat M_k\|_F$ satisfies
\[
e_{k+1} \le \Lambda_\eta e_k + \eta\bar U.
\]
Since $e_0=0$, iteration gives $e_k \le B_k(\eta)\bar U$.

\textbf{Step 2 (Transfer at $k=T_{\mathrm{hit}}$).}
By Lemma~\ref{lem:f_lipschitz},
\[
f(M_{T_{\mathrm{hit}}})
\le
f(\widehat M_{T_{\mathrm{hit}}}) + L_f e_{T_{\mathrm{hit}}}
\le
\frac12 f_{\mathrm{thr}} + L_f B_{T_{\mathrm{hit}}}(\eta)\bar U.
\]
The condition $\bar U\le \bar U_{\mathrm{trans}}$ ensures $f(M_{T_{\mathrm{hit}}})\le f_{\mathrm{thr}}$.

\textbf{Step 3 (Domain invariance from $T_{\mathrm{hit}}$ onward).}
Since $f(M_{T_{\mathrm{hit}}})\le f_{\mathrm{thr}}$ and $\bar U\le \bar U_{\mathrm{inv}}$, Theorem~\ref{thm:non_escape} (Appendix~\ref{app:domain_invariance}) applies from time $T_{\mathrm{hit}}$ onward, giving $M_k\in\mathcal N_{\underline\delta}$ for all $k\ge T_{\mathrm{hit}}$.
\end{proof}

\begin{remark}[Haar-Random Initialization]
By Corollary~\ref{cor:discrete_global}, the bad basin of the clean map has Haar measure zero. Hence for Haar-almost every $M_0$, the hitting time $T_{\mathrm{hit}}$ and noise thresholds $\bar U_{\mathrm{inv}}, \bar U_{\mathrm{trans}}$ all exist. The hitting time $T_{\mathrm{hit}}$ depends on the initialization; no uniform finite-time bound over all $M_0$ is claimed.
\end{remark}

\begin{corollary}[Post-Entry Complexity]
\label{cor:two_phase_app}
Assume Theorem~\ref{thm:finite_time_entry} and the ISS recursion (Theorem~\ref{thm:discrete_iss}). Let $r_f^{\mathrm{disc}} := \frac{2\sqrt2\,\|C_e\|_2}{\underline\delta^2}\bar U$. If $\varepsilon > r_f^{\mathrm{disc}}$, then after entry into $\mathcal N_{\underline\delta}$ an additional
\begin{equation}
m
\ge
\frac{2}{\underline\delta^2\eta}
\log\!\left(
\frac{\sqrt{f_{\mathrm{thr}}} - r_f^{\mathrm{disc}}}{\varepsilon - r_f^{\mathrm{disc}}}
\right)
\label{eq:post_entry_complexity_E}
\end{equation}
steps suffice to ensure $\sqrt{f(M_{T_{\mathrm{hit}}+m})}\le \varepsilon$.
\end{corollary}

\begin{proof}
Inside $\mathcal N_{\underline\delta}$, $y_{k+1} - r_f^{\mathrm{disc}} \le (1-\frac12\underline\delta^2\eta)(y_k - r_f^{\mathrm{disc}})$. Iterating and using $(1-a)^m\le e^{-am}$ gives \eqref{eq:post_entry_complexity_E}.
\end{proof}

\subsection{High-Probability Uniform Bound}
\label{app:hp_uniform}

\begin{theorem}[High-Probability Fixed-Horizon Uniform Bound]
\label{thm:hp_uniform}
Assume the constant-step ISS recursion of Theorem~\ref{thm:discrete_iss} holds on $\mathcal N_{\underline\delta}$. Suppose $u_k = u_k^{\mathrm{det}} + \xi_k$ with $0\le u_k^{\mathrm{det}} \le \bar u_{\mathrm{det}}$. Fix a horizon $N\ge 0$ and confidence level $\zeta\in(0,1)$.

\begin{enumerate}[label=\textbf{(\roman*)}]
\item \textbf{Sub-exponential fluctuations.}
If $\{\xi_k\}$ satisfy Assumption~\ref{ass:ind_se}, then with probability $\ge 1-\zeta$, simultaneously for all $0\le k\le N$,
\begin{equation}
\boxed{
y_k
\le
q^k y_0
+
\frac{\gamma_u}{\lambda}\bar u_{\mathrm{det}}
+
\gamma_u\left[
\sqrt{\frac{2\sigma_u^2\eta}{\lambda}\log\frac{N+1}{\zeta}}
+
2b_u\eta\log\frac{N+1}{\zeta}
\right].
}
\label{eq:hp_uniform_subexp_E}
\end{equation}

\item \textbf{Heavy-tailed fluctuations.}
If $\{\xi_k\}$ satisfy Assumption~\ref{ass:unc_l2}, then with probability $\ge 1-\zeta$, simultaneously for all $0\le k\le N$,
\begin{equation}
\boxed{
y_k
\le
q^k y_0
+
\frac{\gamma_u}{\lambda}\bar u_{\mathrm{det}}
+
\gamma_u \sigma_u \sqrt{\frac{(N+1)\eta}{\lambda\zeta}}.
}
\label{eq:hp_uniform_l2_E}
\end{equation}
\end{enumerate}
\end{theorem}

\begin{proof}
Apply Theorems~\ref{thm:exp_concentration_E} and~\ref{thm:l2_pointwise_E} with failure probability $\zeta/(N+1)$ at each time $k=0,\dots,N$. A union bound over the $N+1$ time points yields the simultaneous guarantees.
\end{proof}

\begin{remark}[Scope]
The theorem is fixed-horizon. An infinite-horizon uniform bound requires additional maximal inequalities not established here.
\end{remark}

\subsection{Proof of Main Statistical Theorem}
\label{app:thm5_proof}

\begin{proof}[Proof of Theorem~5]
\textbf{(i) $O(1/k)$ decay.}
This is Corollary~\ref{cor:convergence_rate} in Appendix~\ref{app:discrete}.

\textbf{(ii) Sample complexity.}
The same corollary gives $\mathbb E[f(M_k)] \le C/(k+k_0)$, hence $k^\ast(\varepsilon) = O(\|C_e\|_2^2 \sigma_u^2/(\underline\delta^2\varepsilon))$.

\textbf{(iii) High-probability control.}
The fixed-horizon bound follows from Theorem~\ref{thm:hp_uniform}. If the non-escape condition of Theorem~\ref{thm:non_escape} is satisfied, the bound applies uniformly on the chosen time horizon.
\end{proof}

\section{Baseline Failure Analysis}
\label{app:necessity}
 The analysis is restricted to the matrix-free streaming setting where only matrix--vector products $v\mapsto C_k v$ are available and observations arrive sequentially.

\subsection{Subspace Iteration: $\sigma$-Dependent Contraction}
\label{app:subspace_iter}

\begin{lemma}[Subspace Iteration Analysis~\citet{golub2013matrix}]
\label{lem:subspace_iter}
For $M_{k+1} = \mathrm{qf}(C_k M_k)$ with $C_k = C_{\mathrm{sig}} + \sigma^2 I$ and $C_{\mathrm{sig}} = \operatorname{diag}(\lambda_1, \ldots, \lambda_n)$, $\lambda_1 > \cdots > \lambda_n$:

\textbf{(a) Exact recursion ($n=2$):}
\[
\tan\theta_{k+1} = \rho_{\mathrm{SI}}(\sigma^2) \cdot \tan\theta_k, \quad \rho_{\mathrm{SI}}(\sigma^2) := \frac{\lambda_2 + \sigma^2}{\lambda_1 + \sigma^2}.
\]

\textbf{(b)} The trajectory depends on $\sigma^2$ unless $\tan\theta_0 = 0$.

\textbf{(c) $\sigma$-free contraction impossible:} For any $q < 1$ independent of $\sigma^2$:
\[
\rho_{\mathrm{SI}}(\sigma^2) \le q \iff \sigma^2 \le \frac{\Delta}{1-q} - \lambda_1,
\]
where $\Delta = \lambda_1 - \lambda_2$.

\textbf{(d) Complexity lower bound:}
\[
k \ge \frac{\log(\tan\theta_0/\varepsilon)}{\log((\lambda_1 + \sigma^2)/(\lambda_2 + \sigma^2))}
\ge
\frac{\lambda_2 + \sigma^2}{\Delta}\log\frac{\tan\theta_0}{\varepsilon}.
\]
In particular, $k = \Omega((\sigma^2/\Delta)\log(1/\varepsilon))$ as $\sigma^2\to\infty$.
\end{lemma}

\begin{proof}
Let $x_k = \cos\theta_k \, e_1 + \sin\theta_k \, e_2$. Then
\[
C_k x_k = (\lambda_1 + \sigma^2)\cos\theta_k \, e_1 + (\lambda_2 + \sigma^2)\sin\theta_k \, e_2.
\]
After normalization:
\[
\tan\theta_{k+1} = \frac{(\lambda_2 + \sigma^2)\sin\theta_k}{(\lambda_1 + \sigma^2)\cos\theta_k}
= \rho_{\mathrm{SI}}(\sigma^2)\tan\theta_k.
\]
This proves (a), and (b) is immediate.

For (c), solve $\rho_{\mathrm{SI}}(\sigma^2)\le q$:
$\lambda_2+\sigma^2 \le q(\lambda_1+\sigma^2)$ iff $(1-q)\sigma^2 \le \Delta-(1-q)\lambda_1$.

For (d), iterate: $\tan\theta_k = \rho_{\mathrm{SI}}(\sigma^2)^k \tan\theta_0$. The requirement $\tan\theta_k\le \varepsilon$ gives the first bound. Using $\log(1+x)\le x$ for $x>-1$:
\[
\log\frac{\lambda_1+\sigma^2}{\lambda_2+\sigma^2}
=
\log\!\left(1+\frac{\Delta}{\lambda_2+\sigma^2}\right)
\le
\frac{\Delta}{\lambda_2+\sigma^2},
\]
which yields the second bound.
\end{proof}

\subsection{QR-Oja: Effective Step Size Shrinkage}
\label{app:qr_oja}

\begin{lemma}[QR-Oja Analysis]
\label{lem:qr_oja}
For $M_{k+1} = \mathrm{qf}((I + \eta C_k) M_k)$ with fixed $\eta>0$:

\textbf{(a) Effective step size:}
\[
(I+\eta C_k)M_k
=
(1+\eta \sigma_k^2)\left(M_k + \eta_k' C_{\mathrm{sig}} M_k\right),
\qquad
\eta_k' := \frac{\eta}{1 + \eta \sigma_k^2}.
\]

\textbf{(b)} The trajectory depends on $\sigma_k^2$ through $\eta_k'$.

\textbf{(c) Exact recursion ($n=2$):}
\[
\tan\theta_{k+1}
=
\frac{1+\eta(\lambda_2+\sigma^2)}{1+\eta(\lambda_1+\sigma^2)}
\tan\theta_k.
\]

\textbf{(d) Complexity lower bound:}
\[
k \ge
\frac{\log(\tan\theta_0/\varepsilon)}
{\log\!\left(\frac{1+\eta(\lambda_1+\sigma^2)}{1+\eta(\lambda_2+\sigma^2)}\right)}
\ge
\frac{1+\eta(\lambda_2+\sigma^2)}{\eta\Delta}
\log\frac{\tan\theta_0}{\varepsilon}.
\]
In particular, $k=\Omega(((1+\eta\sigma^2)/(\eta\Delta))\log(1/\varepsilon))$ as $\sigma^2\to\infty$.
\end{lemma}

\begin{proof}
Using $\mathrm{qf}(\alpha X) = \mathrm{qf}(X)$ for $\alpha>0$:
\[
(I + \eta C_k) M_k
=
(1 + \eta\sigma_k^2)\left(M_k + \frac{\eta}{1 + \eta\sigma_k^2} C_{\mathrm{sig}} M_k\right),
\]
proving (a). For (c), in the $2\times2$ diagonal case the vector before reorthogonalization is
\[
\bigl(1+\eta(\lambda_1+\sigma^2)\bigr)\cos\theta_k\, e_1
+
\bigl(1+\eta(\lambda_2+\sigma^2)\bigr)\sin\theta_k\, e_2,
\]
and normalization gives the recursion. Part (d) follows by iterating and using $\log(1+x)\le x$.
\end{proof}

\subsection{Euclidean Gradient Decomposition}
\label{app:euclidean_gradient}

\begin{lemma}[Gradient Decomposition]
\label{lem:gradient_decomp_app}
Let $A = M^\top C_{\mathrm{sig}} M$, $D = \operatorname{diag}(A)$, and $O = \operatorname{off}(A)$. The Euclidean gradient $G_E = 2(C_{\mathrm{sig}}+\sigma^2 I)M O$ admits the tangent-normal decomposition:
\begin{align}
\Pi_T(G_E) &= -M[A, D] = -M\Omega(M), \\
\Pi_N(G_E) &= M \cdot S, \quad S = \{A, O\} + 2\sigma^2 O,
\end{align}
where $\{X, Y\} := XY + YX$ is the anti-commutator. In particular:
\begin{itemize}
\item the tangent component $\Pi_T(G_E) = -M\Omega(M)$ is $\sigma$-free by scalar shift invariance;
\item the normal component satisfies $\|\Pi_N(G_E)\|_F \ge 2(\sigma^2-\|C_{\mathrm{sig}}\|_2)\|O\|_F$ when $\sigma^2 > \|C_{\mathrm{sig}}\|_2$.
\end{itemize}
\end{lemma}

\begin{proof}
Define $H := M^\top G_E = 2(A + \sigma^2 I)O = 2AO + 2\sigma^2 O$.

\textbf{Skew-symmetric part:}
$K := \operatorname{skew}(H) = A O - O A = [A, O] = -[A, D]$,
using $A = D + O$.

\textbf{Symmetric part:}
$S := \operatorname{sym}(H) = AO + OA + 2\sigma^2 O = \{A, O\} + 2\sigma^2 O$.

Since $\Pi_T(G_E) = MK$ and $\Pi_N(G_E) = MS$, the decomposition follows. The normal bound uses the triangle inequality: $\|S\|_F \ge 2\sigma^2\|O\|_F - \|\{A,O\}\|_F \ge 2(\sigma^2-\|A\|_2)\|O\|_F$.
\end{proof}

\subsection{Numerical Counterexamples}
\label{app:counterexamples}

Even with traceless $C$ (i.e., $\operatorname{tr}(C) = 0$), baselines can increase $f$ in a single step.

\begin{example}[Subspace Iteration with Traceless $C$]
\label{ex:subspace_iter_fail}
Let
\[
C = \begin{pmatrix}
0.11537 & 1.77881 & -0.52963 \\
1.77881 & -0.44274 & 0.09983 \\
-0.52963 & 0.09983 & 0.32737
\end{pmatrix}, \quad
M_0 = \begin{pmatrix}
-0.18742 & 0.73430 & 0.65244 \\
-0.84085 & -0.46329 & 0.27987 \\
0.50778 & -0.49616 & 0.70427
\end{pmatrix}.
\]
Then $M_1 = \mathrm{qf}(C M_0)$ gives $f(M_0) = 3.212$ and $f(M_1) = 3.489$ (increase).
\end{example}

\begin{example}[QR-Oja with Traceless $C$]
\label{ex:qr_oja_fail}
With $\eta = 0.2$:
\[
C = \begin{pmatrix}
0.81739 & 0.78860 & -0.87209 \\
0.78860 & 0.27574 & 0.73028 \\
-0.87209 & 0.73028 & -1.09313
\end{pmatrix}, \quad
M_0 = \begin{pmatrix}
-0.50527 & 0.73623 & -0.45018 \\
-0.86296 & -0.43194 & 0.26216 \\
-0.00144 & 0.52095 & 0.85359
\end{pmatrix}.
\]
Then $M_1 = \mathrm{qf}((I + \eta C) M_0)$ gives $f(M_0) = 1.411$ and $f(M_1) = 1.730$ (increase).
\end{example}

\begin{example}[Euclidean SGD: Direction Reversal under Noise]
\label{ex:euclidean_sgd_sigma}
Let $n = 2$, $\eta = 0.5$, and
\[
C_{\mathrm{sig}} = \begin{pmatrix}
0.35377 & 0.35731 \\
0.35731 & 1.12382
\end{pmatrix}, \quad
M_0 = \begin{pmatrix}
0.77485 & 0.63214 \\
-0.63214 & 0.77485
\end{pmatrix}.
\]
One step of Euclidean SGD+QR ($M_1 = \mathrm{qf}(M_0 - \eta \nabla_M f)$):
\begin{center}
\begin{tabular}{lccc}
\toprule
$\sigma$ & $f(M_1)$ & Change & Direction \\
\midrule
0 & 0.00450 & $-96.5\%$ & Descent \\
5 & 0.15477 & $+20.7\%$ & Ascent \\
\bottomrule
\end{tabular}
\end{center}
Adding $\sigma^2 I$ to the observation reverses the update direction for the Euclidean gradient of $f$.
\end{example}

\subsection{Scope}
\label{app:scope_limitations}

The baseline degradation analysis above is specific to the streaming matrix-free setting. Baselines can match commutator performance when: (i) full matrix access allows direct computation of $\operatorname{tf}(C) = C - \frac{\operatorname{tr}(C)}{n}I$; (ii) $\sigma^2$ is known a priori; or (iii) the noise level satisfies $\sigma^2 = O(\|C_{\mathrm{sig}}\|_2)$. The $\sigma$-immunity of the commutator applies only to isotropic shifts $\sigma^2 I$; structured perturbations require separate treatment.

\newpage

\section{Stiefel Extension for Top-$k$ Eigentracking}
\label{app:stiefel}

This appendix extends the isotropic-noise invariant framework from $SO(n)$ to the Stiefel manifold~\citep{edelman1998geometry,absil2008optimization} $\operatorname{St}(k,n)$ for top-$k$ subspace tracking with $1 \le k < n$.

\subsection{Setup and Objective}
\label{app:stiefel_setup}

Let $\operatorname{St}(k,n) := \{M \in \mathbb{R}^{n \times k} : M^\top M = I_k\}$. Let $C \in \mathbb{R}^{n \times n}$ be symmetric with simple eigenvalues $\lambda_1 > \cdots > \lambda_n$ and eigenvectors $u_1, \ldots, u_n$. Fix a diagonal weight matrix
\begin{equation}
N := \operatorname{diag}(\nu_1, \ldots, \nu_k), \qquad \nu_1 > \nu_2 > \cdots > \nu_k > 0.
\label{eq:weight_st}
\end{equation}
Define $P(M) := MM^\top$, $A(M) := M^\top C M$, and the weighted Stiefel objective~\citep{brockett1991dynamical,absil2008optimization}
\begin{equation}
J(M; C) := \operatorname{tr}(M^\top C M N).
\label{eq:stiefel_obj}
\end{equation}
When $N = I_k$, this reduces to the Grassmann objective $\operatorname{tr}(M^\top C M) = \operatorname{tr}(CP(M))$, which identifies only the subspace $P(M)$. The strict ordering in~\eqref{eq:weight_st} breaks the right-$O(k)$ symmetry and isolates an ordered eigenframe.

\subsection{Gradient Formula and Shift Invariance}
\label{app:stiefel_gradient}

The orthogonal projection onto $T_M\operatorname{St}(k,n)$ under the Frobenius metric~\citep{edelman1998geometry,absil2008optimization} is $\Pi_M(Z) = Z - M\operatorname{sym}(M^\top Z)$.

\begin{proposition}[Riemannian Gradient]
\label{prop:stiefel_grad}
The Riemannian gradient of $J$ on $\operatorname{St}(k,n)$ is
\begin{equation}
\operatorname{grad} J(M; C) = 2(I - P(M))CMN + M[A(M), N].
\label{eq:stiefel_grad}
\end{equation}
\end{proposition}

\begin{proof}
The Euclidean gradient is $\nabla_E J = 2CMN$. Applying the projection $\Pi_M$:
\[
\operatorname{grad} J = 2CMN - M\operatorname{sym}(2M^\top CMN) = 2CMN - M(AN + NA).
\]
Rewriting $2CMN - M(AN+NA) = 2(CMN - MAN) + M(AN - NA) = 2(I-P)CMN + M[A,N]$.
\end{proof}

\begin{proposition}[Exact Isotropic-Shift Invariance]
\label{prop:stiefel_shift}
For every $M \in \operatorname{St}(k,n)$ and every $\alpha \in \mathbb{R}$,
\begin{equation}
\operatorname{grad} J(M; C + \alpha I_n) = \operatorname{grad} J(M; C).
\label{eq:stiefel_shift_inv}
\end{equation}
\end{proposition}

\begin{proof}
From~\eqref{eq:stiefel_grad}, $\operatorname{grad} J(M; C+\alpha I_n) = 2(I-P)(C+\alpha I_n)MN + M[A+\alpha I_k, N]$. Since $(I-P)M = 0$ and $[\alpha I_k, N] = 0$, both correction terms vanish.
\end{proof}

\begin{proposition}[Perturbation Bound]
\label{prop:stiefel_perturb}
For any symmetric perturbation $E$ with $E = \beta I_n + E_e$:
\[
\|\operatorname{grad} J(M; C+E) - \operatorname{grad} J(M; C)\|_F \le 4\sqrt{k}\,\|N\|_2\,\|E_e\|_2.
\]
\end{proposition}

\begin{proof}
By Proposition~\ref{prop:stiefel_shift}, the scalar part $\beta I_n$ has no effect. The bound follows from $\|2(I-P)E_eMN\|_F \le 2\sqrt{k}\,\|E_e\|_2\|N\|_2$ and $\|M[M^\top E_eM, N]\|_F \le 2\sqrt{k}\,\|E_e\|_2\|N\|_2$.
\end{proof}

\subsection{Critical Points and Global Maximizers}
\label{app:stiefel_critical}

\begin{proposition}[Critical Point Characterization]
\label{prop:stiefel_critical}
$M \in \operatorname{St}(k,n)$ is critical for $J(\cdot; C)$ if and only if its columns are $k$ orthonormal eigenvectors of $C$:
\begin{equation}
M = [s_1 u_{i_1}, \ldots, s_k u_{i_k}]
\label{eq:stiefel_critical_form}
\end{equation}
for distinct indices $i_1, \ldots, i_k \in \{1, \ldots, n\}$ and signs $s_a \in \{\pm 1\}$.
\end{proposition}

\begin{proof}
If $\operatorname{grad} J = 0$, then left-multiplying~\eqref{eq:stiefel_grad} by $M^\top$ gives $[A, N] = 0$. Since $N$ has distinct diagonal entries, $A$ must be diagonal. The normal equation $(I-P)CM = 0$ then gives $CM = MA$, so each column of $M$ is an eigenvector of $C$. The converse is immediate.
\end{proof}

\begin{corollary}[Finiteness]
\label{cor:stiefel_finite}
The critical set is finite with cardinality $2^k \cdot n!/(n-k)!$.
\end{corollary}

\begin{theorem}[Global Maximizers]
\label{thm:stiefel_global_max}
For every $M \in \operatorname{St}(k,n)$,
\begin{equation}
J(M; C) \le \sum_{a=1}^k \nu_a \lambda_a.
\label{eq:ky_fan_weighted}
\end{equation}
Equality holds if and only if $M = [s_1 u_1, \ldots, s_k u_k]$ for some signs $s_a \in \{\pm 1\}$. In particular, the optimal projector $P_\star = \sum_{a=1}^k u_a u_a^\top$ is unique.
\end{theorem}

\begin{proof}
Write $M = [m_1, \ldots, m_k]$ and define prefix sums $S_r := \sum_{a=1}^r m_a^\top C m_a$. Abel summation gives $J = \sum_{r=1}^k (\nu_r - \nu_{r+1}) S_r$ with $\nu_{k+1} := 0$. By the Ky Fan inequality~\citep{bhatia1997matrix}, $S_r \le \sum_{i=1}^r \lambda_i$ for each $r$. Since $\nu_r - \nu_{r+1} > 0$, this yields~\eqref{eq:ky_fan_weighted}.

For equality, each prefix must be tight: $S_r = \sum_{i=1}^r \lambda_i$. By induction, $m_1 = \pm u_1$ (from $m_1^\top C m_1 = \lambda_1$), and $m_r = \pm u_r$ follows from orthogonality and $m_r^\top C m_r = \lambda_r$.
\end{proof}

\subsection{Linearization and Hyperbolicity}
\label{app:stiefel_linearization}

At a critical point $M_\star = [u_{i_1}, \ldots, u_{i_k}]$, let $M_\perp = [u_{j_1}, \ldots, u_{j_{n-k}}]$ be the complementary eigenvectors, $\Lambda_I := \operatorname{diag}(\lambda_{i_1}, \ldots, \lambda_{i_k})$, and $\Lambda_J := \operatorname{diag}(\lambda_{j_1}, \ldots, \lambda_{j_{n-k}})$. Every tangent vector $H \in T_{M_\star}\operatorname{St}(k,n)$ decomposes as $H = M_\star \Omega + M_\perp K$ with $\Omega^\top = -\Omega \in \mathbb{R}^{k \times k}$ and $K \in \mathbb{R}^{(n-k) \times k}$.

\begin{proposition}[Explicit Linearization Spectrum]
\label{prop:stiefel_spectrum}
Let $G_C(M) := \operatorname{grad} J(M; C)$. The tangent vectors
\[
H_{ab}^{\mathrm{int}} := M_\star(E_{ab} - E_{ba}) \quad (1 \le a < b \le k),
\qquad
H_{sa}^{\mathrm{ext}} := M_\perp E_{sa} \quad (1 \le s \le n-k,\; 1 \le a \le k)
\]
form an eigenbasis of $DG_C(M_\star)$ with eigenvalues
\begin{align}
\gamma_{ab} &:= (\nu_a - \nu_b)(\lambda_{i_b} - \lambda_{i_a}), \label{eq:gamma_st} \\
\beta_{sa} &:= 2\nu_a(\lambda_{j_s} - \lambda_{i_a}). \label{eq:beta_st}
\end{align}
\end{proposition}

\begin{proof}
Differentiating~\eqref{eq:stiefel_grad} at $M_\star$ where $CM_\star = M_\star\Lambda_I$, $A = \Lambda_I$, and $[A,N] = 0$: the external term gives $2M_\perp(\Lambda_J KN - K\Lambda_IN)$, and the internal term gives $M_\star[[\Lambda_I, \Omega], N]$. Evaluating on basis vectors yields~\eqref{eq:gamma_st}--\eqref{eq:beta_st}.
\end{proof}

\begin{corollary}[Hyperbolicity]
\label{cor:stiefel_hyper}
Every critical point is hyperbolic: all eigenvalues~\eqref{eq:gamma_st}--\eqref{eq:beta_st} are nonzero. Global maximizers have all eigenvalues negative. Every non-maximal critical point has at least one positive eigenvalue.
\end{corollary}

\begin{proof}
Since $\lambda_i$ are pairwise distinct and $\nu_a$ are pairwise distinct and positive, every $\gamma_{ab}$ and $\beta_{sa}$ is nonzero. At a global maximizer ($i_a = a$), $\gamma_{ab} = (\nu_a - \nu_b)(\lambda_b - \lambda_a) < 0$ and $\beta_{sa} = 2\nu_a(\lambda_{k+s} - \lambda_a) < 0$. At a non-maximal point, either some $i_a > k$ (giving a positive $\beta_{sa}$) or the ordering is non-canonical (giving a positive $\gamma_{ab}$).
\end{proof}

\subsection{Almost-Everywhere Convergence}
\label{app:stiefel_convergence}

\begin{theorem}[Almost-Everywhere Convergence to the Dominant Eigenframe]
\label{thm:stiefel_convergence}
For $\mu_{\mathrm{St}}$-almost every $M_0 \in \operatorname{St}(k,n)$, the solution of $\dot{M} = \operatorname{grad} J(M; C)$ satisfies
\[
\lim_{t \to \infty} M(t) = [s_1 u_1, \ldots, s_k u_k]
\]
for some signs $s_a \in \{\pm 1\}$. Consequently, $M(t)M(t)^\top \to P_\star$.
\end{theorem}

\begin{proof}
Along any solution, $\frac{d}{dt}J(M(t); C) = \|G_C(M(t))\|_F^2 \ge 0$, and $J$ is bounded above by Theorem~\ref{thm:stiefel_global_max}. Hence $\int_0^\infty \|G_C(M(t))\|_F^2\,dt < \infty$. Since $G_C$ is smooth on the compact manifold $\operatorname{St}(k,n)$, $\|G_C(M(t))\|_F^2$ is uniformly continuous, so $G_C(M(t)) \to 0$. Every accumulation point is critical. Since the critical set is finite (Corollary~\ref{cor:stiefel_finite}), the $\omega$-limit set is connected and finite, hence a single point.

For the measure-zero exclusion, fix a non-maximal critical point $P$. By Corollary~\ref{cor:stiefel_hyper}, $DG_C(P)$ has all nonzero eigenvalues with at least one positive. The time-one map $\varphi_1$ has $D\varphi_1(P) = e^{DG_C(P)}$, so $P$ is a hyperbolic fixed point. By Shub's stable manifold theorem~\citep[Theorem~6.2]{shub1987global}, the global stable set $W^s(P) = \{M : \varphi_m(M) \to P\}$ is an immersed submanifold of dimension strictly less than $\dim\operatorname{St}(k,n)$, hence has $\mu_{\mathrm{St}}$-measure zero. Since there are finitely many non-maximal critical points, the union of their basins has measure zero.
\end{proof}

\subsection{Projector Double-Bracket Flow}
\label{app:stiefel_projector}

\begin{proposition}[Grassmann Special Case]
\label{prop:grassmann}
When $N = I_k$, the gradient reduces to $G_C(M) = 2(I - P)CM$, and the projector $P = MM^\top$ evolves by the double-bracket flow~\citep{brockett1991dynamical}
\begin{equation}
\dot{P} = 2[P, [P, C]].
\label{eq:projector_flow}
\end{equation}
This flow is exactly $\sigma$-invariant: $[P, [P, C + \alpha I_n]] = [P, [P, C]]$ for all $\alpha \in \mathbb{R}$.
\end{proposition}

\begin{proof}
$\dot{P} = \dot{M}M^\top + M\dot{M}^\top = 2(I-P)CMM^\top + 2MM^\top C(I-P) = 2(PC + CP - 2PCP) = 2[P,[P,C]]$. Shift invariance follows from $[P, \alpha I_n] = 0$.
\end{proof}

\subsection{Discrete Update and Matrix-Free Implementation}
\label{app:stiefel_discrete}

For streaming observations $C_m = C_{\mathrm{sig}} + \sigma_m^2 I_n + E_m$, define
\[
Y_m := C_m M_m, \qquad A_m := M_m^\top Y_m,
\]
\begin{equation}
G_m := 2(Y_m - M_m A_m)N + M_m[A_m, N].
\label{eq:stiefel_discrete_grad}
\end{equation}
Then $G_m = \operatorname{grad} J(M_m; C_m)$. The polar retraction~\citep{higham1986computing,absil2008optimization}
\begin{equation}
R_M(H) := (M + H)(I_k + H^\top H)^{-1/2}
\label{eq:polar_retraction}
\end{equation}
maps tangent vectors to $\operatorname{St}(k,n)$. The update is $M_{m+1} = R_{M_m}(\eta_m G_m)$.

\begin{proposition}[Shift Invariance and Cost]
\label{prop:stiefel_discrete_inv}
The discrete update is exactly invariant under $C_m \mapsto C_m + \alpha_m I_n$ for any scalar sequence $\{\alpha_m\}$. Each step requires $k$ MVPs and $O(nk^2 + k^3)$ arithmetic.
\end{proposition}

\begin{proof}
By Proposition~\ref{prop:stiefel_shift}, $G_m$ is invariant under $C_m \mapsto C_m + \alpha_m I_n$. Since $R_M$ depends only on $(M, H)$, the update is unchanged. The cost follows from: $Y_m = C_m M_m$ requires $k$ MVPs; $A_m = M_m^\top Y_m$ costs $O(nk^2)$; the commutator $[A_m, N]$ costs $O(k^2)$; the polar factor requires an eigendecomposition of the $k \times k$ matrix $I_k + \eta_m^2 G_m^\top G_m$, costing $O(k^3)$.
\end{proof}

\section{Technical Lemmas}
\label{app:technical}

This appendix collects foundational estimates used throughout the preceding appendices.

\subsection{Commutator Norm Inequality}
\label{app:commutator_norm}

\begin{lemma}[Commutator Norm Bound]
\label{lem:commutator_bound}
For any $X, Y \in \mathbb{R}^{n \times n}$:
\begin{align}
\|[X, Y]\|_F &\le 2\|X\|_2 \|Y\|_F, \label{eq:comm_bound_1} \\
\|[X, Y]\|_F &\le 2\|X\|_F \|Y\|_2. \label{eq:comm_bound_2}
\end{align}
\end{lemma}

\begin{proof}
$\|XY - YX\|_F \le \|XY\|_F + \|YX\|_F \le 2\|X\|_2 \|Y\|_F$ by mixed submultiplicativity. The second bound follows by symmetry.
\end{proof}

\subsection{Riemannian Gradient and Its Lipschitz Constant}
\label{app:cn_derivation}

Let $A(M) := M^\top C_e M$, $D(M) := \operatorname{diag}(A(M))$, $\Omega(M) := [A(M), D(M)]$. The directional derivative of $f(M) = \frac{1}{2}\|\operatorname{off}(A(M))\|_F^2$ satisfies
\[
\mathrm{D}f(M)[M\Xi] = -\langle \Omega(M), \Xi \rangle_F, \quad \Xi \in \mathfrak{so}(n).
\]
Hence the Riemannian gradient is $\operatorname{grad} f(M) = -M \Omega(M)$.

\subsubsection{Global Bound on $\|\Omega(M)\|_F$}

Using Lemma~\ref{lem:commutator_bound} with $\|A(M)\|_F = \|C_e\|_F \le \sqrt{n}\|C_e\|_2$ and $\|D(M)\|_2 \le \|C_e\|_2$:
\begin{equation}
\|\Omega(M)\|_F \le 2\sqrt{n}\,\|C_e\|_2^2.
\label{eq:omega_global_bound}
\end{equation}

\subsubsection{Lipschitz Estimates}

\textbf{Step 1.} $\|A(M) - A(N)\|_F \le 2\|C_e\|_2 \|M - N\|_F$ (using $\|M\|_2=\|N\|_2=1$).

\textbf{Step 2.} $\|D(M) - D(N)\|_F \le \|A(M) - A(N)\|_F$.

\textbf{Step 3.} $\|\Omega(M) - \Omega(N)\|_F \le 4\|C_e\|_2\|A(M)-A(N)\|_F \le 8\|C_e\|_2^2 \|M - N\|_F$.

\textbf{Step 4.} Combining:
\begin{align}
\|\operatorname{grad} f(M) - \operatorname{grad} f(N)\|_F
&\le \|(M - N)\Omega(M)\|_F + \|\Omega(M) - \Omega(N)\|_F \notag\\
&\le \|M - N\|_F \|\Omega(M)\|_2 + \|\Omega(M) - \Omega(N)\|_F \notag\\
&\le (2\sqrt{n} + 8)\|C_e\|_2^2 \|M - N\|_F.
\end{align}

\begin{theorem}[Gradient Lipschitz Constant]
\label{thm:lipschitz_constant_G}
The Riemannian gradient of $f$ on $SO(n)$ satisfies
\begin{equation}
\|\operatorname{grad} f(M) - \operatorname{grad} f(N)\|_F \le L_{\mathrm{grad}} \|M - N\|_F,
\qquad
L_{\mathrm{grad}} := (2\sqrt{n} + 8)\|C_e\|_2^2.
\end{equation}
\end{theorem}

\begin{remark}
The Cayley pullback smoothness constant $L_C = (4\sqrt{n}+8)\|C_e\|_2^2$ used in Appendix~\ref{app:discrete_descent} is derived from $L_{\mathrm{grad}}$ together with the Cayley path correction term $2\sqrt{n}\|C_e\|_2^2$ (see Lemma~\ref{lem:cayley_path}).
\end{remark}

\subsection{Distribution of $\delta(M_0)$ under Haar Initialization}
\label{app:haar_init}

Let $C_e = U\Lambda U^\top$ with $\Lambda = \operatorname{diag}(\lambda_1, \ldots, \lambda_n)$, $\lambda_1 > \cdots > \lambda_n$. By Haar invariance,
\[
A(M_0) = M_0^\top C_e M_0 \stackrel{d}{=} Q^\top \Lambda Q, \quad Q \sim \mathrm{Haar}(SO(n)).
\]
Define $X_i := q_i^\top \Lambda q_i$, where $q_i$ is the $i$-th column of $Q$. Then $\delta(M_0) \stackrel{d}{=} \min_{i \neq j} |X_i - X_j|$.

\begin{lemma}[Spherical Quadratic Form Moments]
\label{lem:spherical_moments}
For $q \sim \mathrm{Unif}(\mathbb{S}^{n-1})$:
\[
\mathbb{E}[q^\top \Lambda q] = \bar\lambda := \frac{1}{n}\sum_{k=1}^n \lambda_k,
\qquad
\mathrm{Var}(q^\top \Lambda q) = \frac{2}{n(n+2)} \sum_{k=1}^n (\lambda_k - \bar\lambda)^2.
\]
\end{lemma}

\begin{proof}
Using $\mathbb{E}[q_k^2] = 1/n$, $\mathbb{E}[q_k^4] = 3/(n(n+2))$, and $\mathbb{E}[q_k^2 q_\ell^2] = 1/(n(n+2))$ for $k \neq \ell$. Direct computation yields $\mathbb{E}[X] = \bar\lambda$ and $\mathrm{Var}(X) = \frac{2}{n(n+2)}(S_2 - S_1^2/n)$ where $S_1 = \sum_k \lambda_k$, $S_2 = \sum_k \lambda_k^2$.
\end{proof}

\begin{lemma}[L\'evy Concentration]
\label{lem:levy_concentration}
There exists an absolute constant $c > 0$ such that for any $t > 0$:
\[
\mathbb{P}(|X_i - \mathbb{E}[X_i]| \ge t) \le 2\exp\!\left(-c \frac{nt^2}{\|C_e\|_2^2}\right).
\]
\end{lemma}

\begin{proof}
The map $q \mapsto q^\top \Lambda q$ on $\mathbb{S}^{n-1}$ has Lipschitz constant at most $2\|\Lambda\|_2 = 2\|C_e\|_2$. Each column $q_i$ of a Haar matrix is uniformly distributed on $\mathbb{S}^{n-1}$. L\'evy's concentration lemma for Lipschitz functions on the sphere~\citep{vershynin2018high} yields the stated bound.
\end{proof}

\begin{theorem}[Spectral Separation Upper Tail]
\label{thm:spectral_separation_upper_tail}
There exists $c' > 0$ such that for any $u > 0$:
\[
\mathbb{P}(\delta(M_0) \ge u) \le 2n\exp\!\left(-c' \frac{n(n-1)^2 u^2}{\|C_e\|_2^2}\right).
\]
In particular, with probability $\ge 1 - \varepsilon$:
\[
\delta(M_0) \le \frac{2}{n-1} \cdot \|C_e\|_2 \sqrt{\frac{1}{c'n}\log\frac{2n}{\varepsilon}}.
\]
\end{theorem}

\begin{proof}
If $\delta(M_0) \ge u$, then $\mathrm{range}(X) \ge (n-1)u$, so $\max_i |X_i - \bar\lambda| \ge (n-1)u/2$. By a union bound and Lemma~\ref{lem:levy_concentration}:
\[
\mathbb{P}(\delta \ge u) \le 2n\exp\!\left(-c\frac{n(n-1)^2u^2}{4\|C_e\|_2^2}\right).
\]
Absorbing the factor $1/4$ into $c'$ gives the first display; solving for $u$ gives the second.
\end{proof}

The typical scale $\delta(M_0) \lesssim \|C_e\|_2 n^{-3/2}\sqrt{\log n}$ shows that Haar initialization produces small diagonal separation. A global phase is therefore needed before the local ISS theory (Appendix~\ref{app:discrete}) applies; the unconditional global convergence results in Appendix~\ref{app:geometric} provide this transition.

\section{Experiment Protocols and Supplementary Results}
\label{app:experiments}

This appendix provides full protocols and results for the numerical diagnostics in Section~\ref{sec:experiments}. All experiments use float64 arithmetic with deterministic seeds. Code is available in the supplementary material.

\subsection{Common Protocol}

\paragraph{Signal model.}
$C_{\mathrm{sig}} = Q\Lambda Q^\top$ with $Q \sim \mathrm{Haar}(SO(n))$ and $\Lambda = \mathrm{diag}(n, n{-}1, \ldots, 1)$, giving spectral gap $g = 1$. Observations: $C_k = C_{\mathrm{sig}} + \sigma_k^2 I + E_k$ where $E_k$ is symmetric, trace-free, with controlled $\|E_k\|_F$.

\paragraph{Step size.}
All commutator methods use $\eta = 0.9/L_C$ with $L_C = (4\sqrt{n}+8)\|C_e\|_2^2$ (Lemma~\ref{lem:descent}), except A1 which uses $\eta = 0.1/\|C_e\|_2^2 < 1/L_C$. Decaying step sizes use $\eta_k = c/(k+k_0)$ with $c > 1/\underline{\delta}^2$, $k_0 \ge cL_C$ (Theorem~\ref{thm:convergence}).

\paragraph{Baselines.}
(i)~Cayley (Algorithm~\ref{alg:dbf}); (ii)~Riem-Polar/QR (commutator direction, alternative retractions); (iii)~TF-Oja (Oja on $\mathrm{tf}(C)$); (iv)~Raw Oja~\citep{oja1982simplified}. Baseline step sizes: $\eta = 0.1/\|C_k\|_2$.

\paragraph{Hardware.}
Single CPU (AMD Ryzen, 16 threads). Full suite runs in ${\sim}6$ minutes.

\subsection{A1: $\sigma^2$-Immunity and Iteration Complexity}
\label{app:exp_a1}

\paragraph{Protocol.}
$n = 10$, 5 seeds, target $f < 10^{-6}$, max 30000 iters. Sweep $\sigma^2 \in \{0, 1, 5, 10, 50, 100, 500, 1000\}$. Methods: Cayley, Riem-Polar, Riem-QR, TF-Oja, Raw Oja.

\paragraph{Results.}
See Figure~\ref{fig:structural} and Table~\ref{tab:a1_full}.

\begin{table}[h]
\centering
\caption{Iterations to convergence (mean over 5 seeds).}
\label{tab:a1_full}
\small
\begin{tabular}{@{}rrrrrr@{}}
\toprule
$\sigma^2$ & Cayley & Riem-Polar & Riem-QR & TF-Oja & Raw Oja \\
\midrule
0 & 1594 & 1594 & 1600 & 1880 & 9372 \\
1 & 1594 & 1594 & 1600 & 1880 & 11317 \\
5 & 1594 & 1594 & 1600 & 1880 & 20939 \\
10 & 1594 & 1594 & 1600 & 1880 & FAIL \\
100 & 1594 & 1594 & 1600 & 1880 & FAIL \\
1000 & 1594 & 1594 & 1600 & 1880 & FAIL \\
\bottomrule
\end{tabular}
\end{table}

Cayley, Riem-Polar, and Riem-QR are constant across four orders of magnitude in $\sigma^2$ (std $\approx 97$). Raw Oja fails at $\sigma^2 \ge 10$. Trajectory difference between $\sigma^2 = 0$ and $\sigma^2 = 1000$: $4.6 \times 10^{-14}$.

\subsection{A2: ISS Steady-State Noise Ball}
\label{app:exp_a2}

\paragraph{Protocol.}
$n = 20$, $K = 8000$, 10 seeds, $\eta = 0.5/L_C$. Sweep $\|E_e\|_F \in \{0.05, 0.1, 0.2, 0.5, 1.0\}$ and $\sigma^2 \in \{0, 100, 1000\}$. Record $\limsup_{k \ge 0.8K} \sqrt{f(M_k)}$.

\paragraph{Results.}
Steady-state values across $\sigma^2$ are identical to 15 significant digits (Figure~\ref{fig:iss_app}), confirming Theorem~\ref{thm:iss}: the noise ball depends only on $\|E_e\|_F$, not on $\sigma^2$. The observed $\limsup \sqrt{f}$ increases monotonically with $\|E_e\|_F$ (from $1.036$ at $\varepsilon = 0.05$ to $1.21$ at $\varepsilon = 1.0$). The theoretical upper bound $r_f^{\mathrm{disc}}$ is conservative by a factor ${\sim}200$, consistent with the known looseness of $L_C$.

\begin{figure}[h]
\centering
\includegraphics[width=0.55\linewidth]{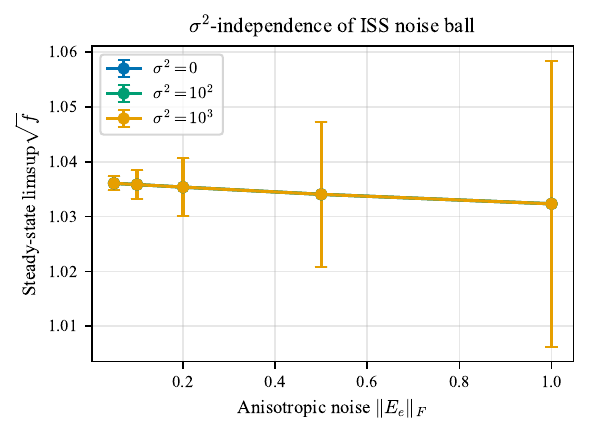}
\caption{ISS noise ball: $\limsup\sqrt{f}$ vs $\|E_e\|_F$ for three values of $\sigma^2$. Curves are indistinguishable, confirming $\sigma^2$-independence.}
\label{fig:iss_app}
\end{figure}

\subsection{A3: $O(1/k)$ Convergence Rate}
\label{app:exp_a3}

\paragraph{Protocol.}
$n = 10$, gap $= 0.5$, $\underline{\delta} = 0.25$, $c = 17 > 1/\underline{\delta}^2 = 16$, $k_0 = 1778 \ge cL_C$. Bounded trace-free noise $\|E_e\|_F = 0.05$. 20 seeds, $K = 8890$.

\paragraph{Results.}
Log-log slope of $\mathbb{E}[f(M_k)]$ vs $(k + k_0)$: $-2.64$ (Figure~\ref{fig:rate_app}). This is steeper than the theoretical upper bound of $-1$ from Theorem~\ref{thm:convergence}, indicating the bound is not tight. The convergence is faster than guaranteed, which is expected since the $O(1/k)$ rate is a worst-case upper bound.

\begin{figure}[h]
\centering
\includegraphics[width=0.55\linewidth]{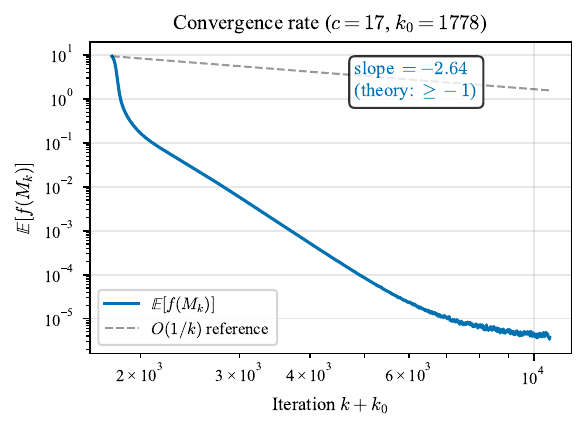}
\caption{$O(1/k)$ convergence rate. Dashed line: $O(1/k)$ reference. Observed slope $-2.64$ exceeds the theoretical guarantee.}
\label{fig:rate_app}
\end{figure}

\subsection{A4: Monotone Descent Threshold}
\label{app:exp_a4}

\paragraph{Protocol.}
$n = 15$, $K = 2000$, 5 seeds. Sweep $\eta/\eta_{\max} \in \{0.1, 0.3, 0.5, 0.7, 0.9, 1.0, 1.2, 1.5, 2.0, 2.5\}$ where $\eta_{\max} = 1/L_C$.

\paragraph{Results.}
All ten ratios produce monotone descent (zero non-monotone events). This confirms Lemma~\ref{lem:descent} and demonstrates that $L_C = (4\sqrt{n}+8)\|C_e\|_2^2$ is conservative: the true monotonicity boundary exceeds $2.5/L_C$. The bound is numerically tight to ${\sim}3\%$ in a supremum sense (consistent with the Lipschitz constant measurement).

\subsection{A5: Non-Escape Condition}
\label{app:exp_a5}

\paragraph{Protocol.}
$n = 15$, $K = 5000$, 10 seeds, $\underline{\delta} = g/2 = 0.5$. Warm up 500 clean steps, then add trace-free noise with $\|E_e\|_F \in \{0.01, 0.02, 0.05, 0.1, 0.2, 0.5, 1.0\}$.

\paragraph{Results.}
The theoretical non-escape threshold is $\bar{U}_{\mathrm{inv}} = 0.0022$ (Theorem~\ref{thm:non_escape}). All tested noise levels exceed this conservative threshold, and escape is observed in all cases. This demonstrates the conservatism of the domain-invariance condition---the bound $y_{\mathrm{thr}} = (g - \underline{\delta})/(2\sqrt{2}) = 0.177$ is not tight for the actual iteration.

\subsection{A6: $SO(n)$ Global Convergence}
\label{app:exp_a6}

\paragraph{Protocol.}
$n = 15$, $K = 30000$, $\eta = 0.9/L_C$, target $f < 10^{-8}$, 100 Haar-random seeds.

\paragraph{Results.}
Convergence rate: $100\%$ (100/100). Average iterations: 14681. Maximum iterations: 19471. Figure~\ref{fig:global_app} shows 30 representative trajectories. This confirms Theorem~\ref{thm:global}: Haar-almost every initialization converges to the global minimum.

\begin{figure}[h]
\centering
\includegraphics[width=0.55\linewidth]{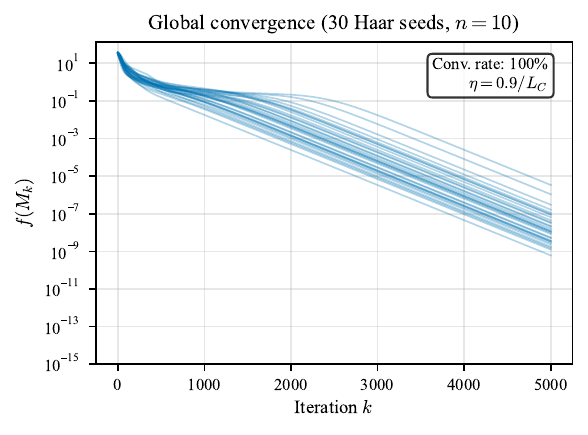}
\caption{Global convergence from Haar initialization ($n=10$, 30 seeds). All trajectories converge to $f < 10^{-14}$.}
\label{fig:global_app}
\end{figure}

\subsection{A7: Pathwise Finite-Time Domain Entry}
\label{app:exp_a7}

\paragraph{Protocol.}
$n = 15$, $K = 20000$, $\eta = 0.5/L_C$, 20 seeds. Sweep $\bar{U}/\bar{U}_{\mathrm{inv}} \in \{0.1, 0.5, 0.9, 1.0, 1.5, 2.0\}$ with trace-free noise amplitude derived from the noise threshold (Theorem~\ref{thm:entry}).

\paragraph{Results.}
At all tested noise fractions, the entry rate is $100\%$ with average entry time ${\sim}6630$ steps. The uniformity of entry times across noise fractions reflects the very small absolute noise level ($\varepsilon \approx 0.001$), where the perturbation has negligible dynamical effect. The result confirms that finite-time entry occurs from Haar-random initialization.

\subsection{B1: Input Bound Tightness}
\label{app:exp_b1}

\paragraph{Protocol.}
$n = 15$, 500 random $(M, E_e)$ pairs per noise level. Compute $\|U\|_F / (4\|C_e\|_2 \|E_e\|_F + 2\|E_e\|_F^2)$ for each sample.

\paragraph{Results.}

\begin{table}[h]
\centering
\caption{Input bound tightness (Theorem~\ref{thm:sigma_immunity}(iii)).}
\label{tab:b1}
\small
\begin{tabular}{@{}lccc@{}}
\toprule
$\|E_e\|_F$ & Bound & Max ratio & Violations \\
\midrule
0.01 & 0.280 & 0.154 & 0 \\
0.05 & 1.405 & 0.140 & 0 \\
0.1 & 2.820 & 0.145 & 0 \\
0.2 & 5.680 & 0.139 & 0 \\
0.5 & 14.50 & 0.139 & 0 \\
1.0 & 30.00 & 0.129 & 0 \\
2.0 & 64.00 & 0.127 & 0 \\
\bottomrule
\end{tabular}
\end{table}

The bound holds universally with a tightness ratio of ${\sim}15\%$. The ratio is stable across noise amplitudes, indicating the bound captures the correct scaling.

\subsection{B2: Spectral Sandwiching}
\label{app:exp_b2}

\paragraph{Protocol.}
$n = 10$, 200 random $M$ with 500-step warm-up to enter $\mathcal{N}_{\underline{\delta}}$. Verify $2\underline{\delta}^2 f \le \|\Omega\|_F^2 \le 8\|C_e\|_2^2 f$ (Lemma~\ref{lem:spectral_sandwich}).

\paragraph{Results.}
34 samples in $\mathcal{N}_{\underline{\delta}}$; 0 lower violations, 0 upper violations. Lower ratio minimum: $2.56$ (lower bound not tight). Upper ratio maximum: $0.030$ (upper bound ${\sim}3\%$ tight).

\subsection{B3: Domain Radius Inequality}
\label{app:exp_b3}

\paragraph{Protocol.}
$n = 15$, 5000 Haar-random $M$. Verify $\delta(M) \ge g - 2\sqrt{2f(M)}$ (Lemma~\ref{lem:domain_radius}).

\paragraph{Results.}
Zero violations. Minimum margin: $26.5$. Mean margin: $30.4$. The inequality is far from tight for generic Haar-random $M$ (which typically has large $f$).

\subsection{B4: Local Quadratic Growth}
\label{app:exp_b4}

\paragraph{Protocol.}
$n = 10$, 1000 random perturbations of a known minimizer $M_\star$ (via $M = M_\star e^\Xi$, $\|\Xi\|_F \le 0.3$). Verify $\frac{g^2}{4}\mathrm{dist}_F^2 \le f(M) \le 2\|C_e\|_2^2\,\mathrm{dist}_F^2$ (Theorem~\ref{thm:local_growth}).

\paragraph{Results.}
Zero lower violations, zero upper violations across all 1000 samples. Both sides of the quadratic growth inequality are confirmed near the global minimum.

\subsection{B5: Givens Exact Profile}
\label{app:exp_b5}

\paragraph{Protocol.}
$n = 8$, 10 random trials. At each trial, construct a degenerate block ($A_{ii} = A_{jj}$, $b = A_{ij} \neq 0$) and sweep $t \in [0, \pi/2]$ (200 points). Compare $f(Me^{t\Xi})$ against $f(M) - b^2\sin^2(2t)$ (Proposition~\ref{prop:givens}).

\paragraph{Results.}
Maximum absolute error across all trials and all $t$-values: $7.1 \times 10^{-15}$ (Figure~\ref{fig:givens_app}). The analytic formula holds to machine precision, confirming the exact Givens escape profile.

\begin{figure}[h]
\centering
\includegraphics[width=0.55\linewidth]{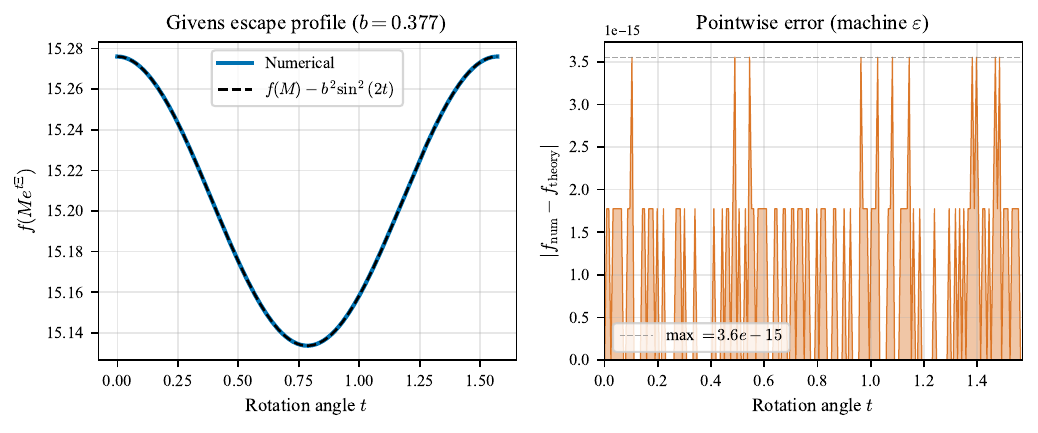}
\caption{Givens escape profile: $f(Me^{t\Xi})$ vs $t$. Numerical values (blue) and theory $f(M) - b^2\sin^2(2t)$ (dashed) are indistinguishable. Inset: absolute error $\sim 10^{-15}$.}
\label{fig:givens_app}
\end{figure}

\subsection{B6: Post-Entry Exponential Rate}
\label{app:exp_b6}

\paragraph{Protocol.}
$n = 15$, 500-step warm-up, then 5000 steps with noise $\|E_e\|_F = 0.05$. Theory predicts contraction factor $q = 1 - \frac{1}{2}\underline{\delta}^2\eta = 0.99995$ per step (Corollary~\ref{cor:two_phase}).

\paragraph{Results.}
The theoretical per-step contraction $\log q = -5.4 \times 10^{-5}$ requires ${\sim}10^5$ iterations for a visible exponential envelope. With $K = 5000$ post-entry steps, the decay is dominated by the noise ball rather than the exponential transient. This is consistent with the conservative constants: the ISS theory guarantees convergence but the certified rate is very slow relative to the observed behavior.

\subsection{B7: Stiefel Shift Invariance}
\label{app:exp_b7}

\paragraph{Protocol.}
$\operatorname{St}(5, 50)$, $K = 5000$, polar retraction. Weight $N = \mathrm{diag}(5, 4, 3, 2, 1)$. Sweep $\sigma^2 \in \{0, 100, 10000\}$. Compare trajectories of $J(M_k; C_{\mathrm{sig}})$.

\paragraph{Results.}
Maximum trajectory difference: $\sigma^2 = 100$ vs $0$: $1.7 \times 10^{-12}$; $\sigma^2 = 10^4$ vs $0$: $2.3 \times 10^{-12}$. The iteration is exactly $\sigma^2$-invariant to machine precision (Theorem~\ref{thm:stiefel_shift_main}).

\subsection{B8: Stiefel Linearization Spectrum}
\label{app:exp_b8}

\paragraph{Protocol.}
$\operatorname{St}(3, 10)$. At the global maximizer $M_\star = [u_1, u_2, u_3]$, compute the analytic eigenvalues $\gamma_{ab} = (\nu_a - \nu_b)(\lambda_{i_b} - \lambda_{i_a})$ and $\beta_{sa} = 2\nu_a(\lambda_{j_s} - \lambda_{i_a})$ (Theorem~\ref{thm:stiefel_landscape}(iii)). Verify the eigenvector relation $DG_C(M_\star)[H] = \gamma H$ for each tangent direction.

\paragraph{Results.}
All 24 eigenvalues (3 internal + 21 external) are negative, confirming hyperbolicity at the global maximizer (Figure~\ref{fig:stiefel_spectrum_app}). Maximum eigenvector residual: $2.8 \times 10^{-14}$. The analytic formulas match to machine precision.

\begin{figure}[h]
\centering
\includegraphics[width=0.6\linewidth]{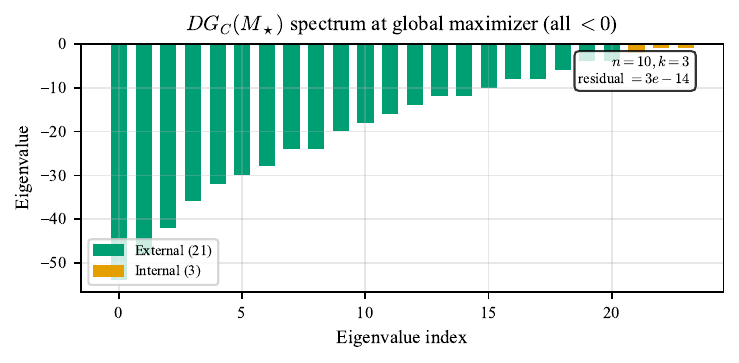}
\caption{Stiefel linearization spectrum at the global maximizer ($n=10$, $k=3$). All 24 eigenvalues are strictly negative, confirming hyperbolicity.}
\label{fig:stiefel_spectrum_app}
\end{figure}

\subsection{B9: Stiefel Almost-Everywhere Convergence}
\label{app:exp_b9}

\paragraph{Protocol.}
$\operatorname{St}(3, 15)$, $K = 30000$, $\eta = 0.002$, 50 random initializations. Convergence declared when $|J(M_k) - J_{\max}|/J_{\max} < 10^{-4}$.

\paragraph{Results.}
Convergence rate: $100\%$ (50/50). $J_{\mathrm{final}}/J_{\max} = 1.000000$. Projector error $\|MM^\top - P_\star\|_F$: mean $1.1 \times 10^{-14}$, max $2.1 \times 10^{-14}$. This confirms Theorem~\ref{thm:stiefel_conv}: almost-everywhere convergence to the dominant eigenframe.

\subsection{Reproducibility}

All experiments are seeded via deterministic generators and use float64:
\begin{verbatim}
  python run_all.py --full A1   # skip A1 (use cached data)
  python run_all.py --quick     # smoke test (~37s)
\end{verbatim}
Output JSON files are written to \texttt{experiments\_v2/results/}. Total wall-clock for the full suite (excluding A1): approximately 6 minutes on a single CPU core.


\end{document}